\documentclass[twoside]{article}

\usepackage[accepted]{arxiv_aistats2025}
\usepackage[round]{natbib}

\let\oldsection\section
\RenewDocumentCommand{\section}{s o m}{%
  \IfBooleanTF{#1}
    {\oldsection*{\MakeUppercase{#3}}}
    {\IfValueTF{#2}
       {\oldsection[\MakeUppercase{#2}]{\MakeUppercase{#3}}}
       {\oldsection{\MakeUppercase{#3}}}
    }%
}

\usepackage{booktabs} 

\usepackage[colorlinks=true, citecolor=blue]{hyperref}

\usepackage{etoolbox}

\usepackage{savesym,multirow,url,xspace,graphicx,enumitem,color}
\usepackage{subcaption}
\usepackage[toc]{appendix}
\usepackage{dsfont}
\usepackage{multicol}
\usepackage{comment}

\usepackage{tikz}
\usetikzlibrary{automata, positioning, arrows}
\tikzset{
	->, 
	every state/.style={thick, fill=gray!10}, 
	initial text=$ $, 
}
\usetikzlibrary{arrows.meta}
\usepackage{pgfplots}
\pgfplotsset{width=10cm,compat=1.9}
\usepackage{diagbox}
\usepackage{caption}

\definecolor{gold}{HTML}{D2691E}
\newcommand{\rilind}[1]{{\textcolor{gold}{{\bf RS:} #1}}}

\newcommand{\todo}[1]{{\textcolor{red}{{\bf TODO:} #1}}}

\newcommand{\NN}{\ensuremath{\mathbb{N}}}

\newcommand{\RR}{\ensuremath{\mathbb{R}}}
\newcommand{\PP}{\ensuremath{\mathbb{P}}}

\newcommand{\E}{\mathbb{E}}

\makeatletter
\newif\if@restonecol
\makeatother

\usepackage{amsmath,amssymb,xfrac,amsthm}
\usepackage{mathtools}
\setitemize{noitemsep,topsep=1pt,parsep=1pt,partopsep=1pt, leftmargin=12pt}
\newtheorem{lemma}{Lemma}
\newtheorem{theorem}{Theorem}
\newtheorem*{theorem*}{Theorem}

\newtheorem{definition}{Definition}
\newtheorem{assumption}{Assumption}

\makeatletter

\newcommand{\Rmnum}[1]{\expandafter\@slowromancap\romannumeral #1@}
\makeatother

\usepackage{enumitem}
\usepackage{wasysym}

\DeclareMathOperator*{\argmin}{arg\,min}
\DeclareMathOperator*{\argmax}{arg\,max}

\usepackage{bm}
\usepackage{bbm}
\let\mathbbm\mathds

\usepackage{algorithm}
\usepackage{algpseudocode}

\usepackage{cleveref}
\renewcommand{\algorithmicrequire}{\textbf{Input:}}
\renewcommand{\algorithmicensure}{\textbf{Output:}}

\newcommand{\Acal}{\mathcal{A}}

\newcommand{\perfc}{\omega}

\begin{document}

\newtoggle{longversion}
\settoggle{longversion}{true}

%

%
\runningauthor{Rilind Sahitaj, Paulius Sasnauskas,  Yiğit Yalın, Debmalya Mandal, Goran Radanović}

\twocolumn[

\aistatstitle{Independent Learning in Performative Markov Potential Games}

\renewcommand{\thefootnote}{\fnsymbol{footnote}}




\aistatsauthor{ Rilind Sahitaj\footnotemark[1] \And Paulius Sasnauskas\footnotemark[1] \And  Yi\u{g}it Yal\i{}n}
\aistatsaddress{ RWTH Aachen \& MPI-SWS  \And  MPI-SWS \And MPI-SWS }

\aistatsauthor{Debmalya Mandal \And Goran Radanovi\'{c}}
\aistatsaddress{ University of Warwick \And MPI-SWS}

]


\addtocontents{toc}{\protect\setcounter{tocdepth}{-1}} 

\begin{abstract}
Performative Reinforcement Learning (PRL) refers to a scenario in which the deployed policy changes the reward and transition dynamics of the underlying environment. 
In this work, we study multi-agent PRL by incorporating performative effects into Markov Potential Games (MPGs). We introduce the notion of a performatively stable equilibrium (PSE) and show that it always exists under a reasonable sensitivity assumption. We then provide convergence results for state-of-the-art algorithms used to solve MPGs. Specifically, we show that independent policy gradient ascent (IPGA) and independent natural policy gradient (INPG) converge to an approximate PSE in the \textit{best-iterate} sense, with an additional term that accounts for the performative effects. Furthermore, we show that INPG asymptotically converges to a PSE in the \textit{last-iterate} sense. As the performative effects vanish, we recover the convergence rates from prior work. For a special case of our game, we provide \textit{finite-time last-iterate} convergence results for a repeated retraining approach, in which agents independently optimize a surrogate objective. We conduct extensive experiments to validate our theoretical findings.

\end{abstract}

\section{Introduction}

\begin{table*}[t]
    \centering
    \caption{
    Iteration complexity of the independent policy gradient methods with gradient oracles. IPGA stands for independent policy gradient ascent and INPG abbreviates independent natural policy gradient. 
    $\delta_{r, p} \coloneqq \frac{1}{1-\gamma}  \left(\perfc_r +\frac{ \gamma \cdot \perfc_p \sqrt{S}}{1-\gamma} \right)$, where $\perfc_r$ and $\perfc_p$ are the sensitivity parameters modelling the \textit{performative} effects, $\mathcal{W}_{r, p} \coloneqq (n+1)\cdot n^2 \cdot S \cdot \mathcal{\delta}_{r, p}$, $S$ is the number of states, $\gamma$ is the discount factor.
    $\kappa_\rho \coloneqq \sup_t \sup_{\pi \in \Pi} \lVert d^\pi_{\rho,t}/\rho \rVert_\infty$ and its minimax-version $\widetilde{\kappa}_\rho \coloneqq \inf_{\rho \in \Delta(\mathcal{S})} \kappa_\rho$, where $A$ is the total number of actions over all agents, $A_{\max}$ is the largest action set over agents, $n$ is the number of agents, $C_\Phi$ is the potential function difference upper bound, $T$ is the number of rounds, $M_\rho \coloneqq \sup_{\pi} \sup_{\pi'} \max_s \frac{1}{d_{\rho, \pi}^{\pi'}(s)} < \infty$, $c$ is the lower bound for the probability of playing an optimal action.}
    
    \begin{tabular}{lcc}
        \toprule
        \textbf{Algorithm} & \textbf{Theorem} & \textbf{Iteration Complexity} \\
        \midrule
        IPGA & Theorem \ref{Theorem: Convergence Result - Infinite Sample Case} & $\mathcal{O} \left( \frac{\min\{\kappa_\rho, S\}^4 \cdot n \cdot A \cdot C_\Phi}{(1-\gamma)^{10}} \cdot \frac{\mathcal{W}_{r, p}}{\epsilon^2} \right)$ \\
        \midrule
        INPG & Theorem \ref{thm:inpg_conv} & $ \mathcal{O} \left( \left( \frac{ \Tilde{\kappa}_\rho \left(\sqrt{n} + M \cdot \delta_{r, p}\right)}{\epsilon^2 \cdot c \cdot (1 - \gamma)^{4}}  \right) \right) $  \\
        \midrule
        INPG + log-barrier reg. & Theorem \ref{thm:inpg-reg-conv} & $\mathcal{O} \left(\frac{n \cdot A_{\max} \cdot M^2}{\epsilon^2 \cdot (1 - \gamma)^4} \cdot \max\left\{1, S \cdot  \delta_{r, p} \right\} \right)$ \\
        \bottomrule
    \end{tabular}
    \label{tab:main-results}
\end{table*}

Multi-Agent Reinforcement Learning (MARL) is a framework for learning equilibrium policies in complex strategic environments. Although recent success stories in game playing~\citep{AlphaGO, Brown2019SuperhumanAF, StarCraft} showcase the practical importance of scalable MARL algorithms, deploying them in real world settings often requires robustness properties that are not needed in games with well-defined rules.
Specifically, it is important to account for higher-order effects that agents may have on the dynamics of the environment.

\renewcommand{\thefootnote}{\fnsymbol{footnote}}
\footnotetext[1]{This work was done as a part of 
an internship project at the Max Planck Institute for Software Systems.}
\renewcommand{\thefootnote}{\arabic{footnote}}

In MARL, we typically identify two types of feedback loops: a) one which is due to the RL nature of the problem setting, i.e., an agent's policy affecting future states, and b) one which is due to the multi-agent nature of the problem setting, i.e., an agent's transitions and rewards are affected by the other agents. However, the environment itself, i.e., reward function and transition probabilities, can be influenced by the policies that the agents deploy in the environment. We refer to this type of feedback loop as \textit{performative} effects. Take, for example, AI-assistants, which interact with both users and other agents. As an AI-assistant interacts with its users, it may change their preferences, which in turn influence future interactions. This can be exploited by competing agents that may adapt their strategies accordingly.

Recently, \textit{performative} effects of an agent on its environment have been studied in single-agent reinforcement learning~\citep[\textit{performative RL}]{PerformativeReinforcementLearning}, as well as in supervised learning~\citep[\textit{performative prediction}]{DBLP:conf/icml/PerdomoZMH20}.
However, \textit{performative} effects are unexplored in the context of MARL.
Unlike prior work on performative prediction and performative RL, \textit{performative} MARL has to simultaneously account for the three type of feedback loops from the previous paragraph. 

Motivated by the practical importance of this setting, our goal is to understand the impact of \textit{performative} effects on the convergence of well-known MARL algorithms.
We base our formal framework on a class of Markov games, called Markov Potential Games (MPGs), and focus on independent policy gradient algorithms, which have been shown to converge in MPGs. Our contributions are as follows: 
\begin{itemize}
    \item {\bf Framework:} We extend the MPG setting to incorporate \textit{performative} effects. We introduce the notions of a performatively stable equilibrium (PSE) and performative regret (PReg).
    \item {\bf Solution Concepts:} We show that a PSE exists under reasonable sensitivity assumptions.
    Furthermore, we prove that every PSE corresponds to an approximate Nash equilibrium. 
    \item {\bf Performative Regret Guarantees:} We study IPGA and two versions of INPG, with and without log-barrier regularization. We show that the algorithms achieve  
    bounded performative regret, hence converging to a PSE in the best-iterate sense. Table \ref{tab:main-results} provides an overview of the convergence results.  
    \item {\bf Last-Iterate Convergence Guarantees:} 
    We show that unregularized INPG achieves last-iterate convergence to an exact PSE asymptotically. However, compared to the regularized version, its guarantee depends on the probability of playing an optimal action. Furthermore, we show that the regularized version of INPG also achieves best-iterate convergence, but compared to the unregularized version, its guarantee does not require a lower bound on the probability of playing an optimal action. 
    Furthermore, we generalize the repeated optimization approach from prior work on performative RL to multi-agent RL for a special class of \textit{performative} MPGs. We show a {\em finite-time last-iterate} convergence guarantee for this method, which we prove by adapting the proof technique from \citet{PerformativeReinforcementLearning} to our problem setting. 
    \item {\bf Experiments:} We evaluate the gradient-based algorithms on the safe-distancing game \citep{GlobalConvergence_Leonardos, IndepPolicyGrad_Ding} and stochastic congestion games \citep{NPG_converges}, showing that \textit{performative} effects significantly impact convergence, affecting algorithms differently. According to our empirical results, the natural policy gradient methods are more robust against \textit{performative} effects. 
     
\end{itemize}

\section{Related Work}


\paragraph{Markov Potential Games.}
Our work is related to Markov Games \citep{stochasticGames} and to a subclass of games known as Markov Potential Games. \citet{GlobalConvergence_Leonardos} show that in the infinite-horizon case, an independent policy gradient ascent algorithm converges to an $\epsilon$-approximate Nash equilibrium after $\mathcal{O}(1/\epsilon^6)$ iterations. 
\citet{IndepPolicyGrad_Ding} extend this to function approximation and improve to an iteration complexity of $\mathcal{O}(1/\epsilon^5)$ using a slightly different update rule. Given gradient-oracles, both methods achieve an $\mathcal{O}(1/\epsilon^2)$ iteration complexity.
In the setting without gradient oracles, \citet{DBLP:conf/icml/Mao0ZB22} improve on the previous bounds with a $\mathcal{O}(1/\epsilon^{4.5})$ iteration complexity, by reducing the variance of the policy gradient algorithm.
Furthermore, it is well known that independent natural policy gradient (INPG) converges in MPGs \citep{NPG_converges, NPG_converges2, DBLP:conf/nips/ZhangMDS022, NPG_3} with sample complexity $\mathcal{O}(1/\epsilon^2)$. Under additional assumptions, \citet{NPG_converges2} improve this to an iteration complexity of $O(1/\epsilon)$, while \citet{NPG_3} reduce the linear dependence on the number of agents to a sublinear dependence, maintaining the $O(1/\epsilon^2)$ complexity.\\
In our setting, our guarantees additionally depend on an additive term that is dependent on the strength of the \textit{performative} effects. We recover the guarantees from prior work for the considered algorithms as the \textit{performative} effects become negligible.
\paragraph{Performative Prediction.}
Since the seminal work on performative prediction \citep{DBLP:conf/icml/PerdomoZMH20}, various adaptations have been studied. We refer to the survey by \citet{PP_PF}.
A recent line of work studies performative prediction in a multi-agent setting \citep{li2022multi, narang2023multiplayer, piliouras2023multi}.
More similar to reinforcement learning and related to our work, \citet{PP_StatefulWorld} study a version of performative prediction in which the target population exhibits a state that captures historical information to account for gradual changes in the distribution. See also work by \citet{DBLP:conf/aistats/0017W22, DBLP:conf/aaai/RayRDF22, DBLP:conf/aistats/Izzo0Y22}. 
For more details on the growing literature of performative prediction, we refer to Appendix~\ref{Additional Related work}.

\paragraph{Performative Reinforcement Learning.}

\citet{PerformativeReinforcementLearning} introduce the setting of performative reinforcement learning (PRL), where the rewards and transition function adapt to the policy.
They propose a repeated retraining approach over the occupancy measure space that converges to an approximate performatively stable point.
We generalize this technique to MPGs with performative effects and agent independent transitions. 
\citet{rank2024performative} study an extension, where the environment additionally changes based on the past dynamics.
Recently, \citet{PRL_LinearMDPs} extend PRL to the linear MDP setting, \citet{PDynamicalSystem} extend \textit{performative} effects in linear dynamical systems, while \citet{warwick189192} study corruption-robustness in PRL.

\section{Formal Setting}\label{sec: Formal Setting}

First, we provide some basic notation that we will use throughout. We let $\Delta(X)$ denote the probability simplex over the finite set $X$. Further, given two vectors $a, b \in \mathbb{R}^k$ for a natural number  $k \in \mathbb{N}$, we define $a/b$ as the vector obtained by component-wise division. 

\subsection{Performative Markov Game}
We define a $n$-player Game (MG) with \textit{performative} effects induced by the adopted joint policy $\Bar{\pi}$ as a tuple
\begin{align*}
    \mathcal{G}(\Bar{\pi}) = ( \mathcal{N}, \mathcal{S},  \{ \mathcal{A}_i \}^n_{i=1}, \{ r_{i, \Bar{\pi}} \}^n_{i=1}, P_{\Bar{\pi}}, \gamma, \rho), 
\end{align*}
where $\mathcal N = \{1, \dots, n \}$ is the set of agents, $\mathcal S$ is a (finite) state space, $\mathcal{A}_i$ is a finite action set of agent $i \in \mathcal{N}$ with the joint action space denoted as $\mathcal A \vcentcolon = \prod_{i \in \mathcal{N}} \mathcal A_i$ for $i \in \mathcal{N}$. Under the adopted policy $\Bar{\pi}$, $r_{i, \Bar{\pi}}: \mathcal{S} \times \mathcal A \rightarrow [0,1]$ is the $i$-th agent's reward function and $P_{\Bar{\pi}}$ is the transition probability measure, characterized by a distribution $P_{\Bar{\pi}}( \cdot \mid s, a)$ over $\mathcal S$, given an action $a \in \mathcal{A}$ and state $s \in \mathcal{S}$. Furthermore, $\gamma \in [0,1)$ is the discount factor and $\rho \in \Delta(S)$ corresponds to the initial state distribution. We abbreviate the cardinalities of the action space and state space as $| \mathcal{A} | = A, | \mathcal{A}_i | = A_i$ and $|\mathcal{S}| = S$, respectively.

We assume that the joint policy $\pi = (\pi_i)_{i \in \mathcal{N}} \in \Pi \coloneqq \Delta(\mathcal{A}_1)^{\mathcal{S}} \times \dots \times \Delta(\mathcal{A}_n)^{\mathcal{S}}$ is a product of individual, stochastic policies.
Specifically, we define agent's $i$ policy $\pi_i : \mathcal{S} \rightarrow \Delta(\mathcal{A}_i)$ as a probability distribution $\pi_i( \cdot | s)$ over $\mathcal{A}_i$, given a state $s \in \mathcal{S}$. For brevity, we denote $\Pi_i = \Delta(\mathcal{A}_i)$. 
We emphasize that the probability transition function and the agent specific reward functions adapt to the underlying joint policy $\Bar{\pi}$ specified in the game $\mathcal{G}(\Bar{\pi})$.

We use superscript $t \in [T]$ to indicate dependence on the iterations of the methods we study. For example, the agents' joint policy at iteration $t$ is denoted by $\pi^t$. When a function $f_{\pi^{t'}}^{\pi^t}$ depends on $\pi^t$ and $\mathcal{G}(\pi^{t'})$, we use the shorthand notation $f_{t'}^t$.

We use superscript $h$ to denote dependence on a single decision-making time step. More specifically, at a time step $h$, we are given a state $s^{h-1} \in \mathcal{S}$, an action profile $a^h = (a^h_{1},\dots,a^h_{n}) \in \mathcal{A}$ and observe rewards $r_{1, \Bar{\pi}}(s^{h-1}, a^h), \dots, r_{n, \Bar{\pi}}(s^{h-1}, a^h)$ and transition to a state $s^h$. 



The probability of a trajectory $\tau = (s^h, a^h, r^h)_{h = 0}^\infty$ is given by: $\PP_{\Bar{\pi}}^{\pi}(\tau) = \rho(s^0) \cdot \Pi_{h \geq 0} P_{\Bar{\pi}}(s^{h+1} | s^h, \pi(s^h))$, i.e., the trajectory $\tau$ is sampled by following policy $\pi$ under the transition function $P_{\Bar{\pi}}$. 
Furthermore, we define the \textit{performative} value function of agent $i \in \mathcal{N}$ in $\mathcal G(\Bar{\pi})$ under joint policy $\pi$ as:
\begin{equation}\label{eq: Value function}
    V_{i,\Bar{\pi}}^{\pi}(\rho) 
    = 
    \mathbb{E}_{ \tau \sim \PP_{\Bar{\pi}}^\pi} \Big[ \sum_{h=0}^\infty \gamma^h r_{i, \Bar{\pi}}^h \big| s^0 \sim \rho \Big].
\end{equation}


We denote by $Q_{i, \Bar{\pi}}^{\pi}: \mathcal{S} \times \mathcal{A} \rightarrow \RR$ the action-value function of an agent $i \in \mathcal{N}$ in $\mathcal G(\Bar{\pi})$ under joint policy $\pi$, that is
\begin{align*}
    Q_{i, \Bar{\pi}}^\pi(s,a) 
    &\coloneqq
    \mathbb{E}_{ \tau \sim \PP_{\Bar{\pi}}^\pi} \Big[ \sum_{h=0}^\infty \gamma^h r_{i, \Bar{\pi}}^h \big| s^0 = s, a^0 = a \Big].
\end{align*}



Given initial state distribution $\rho$ and policy $\pi$, the discounted state visitation distribution with respect to the underlying $\mathcal{G}(\Bar{\pi})$ is given by:
\begin{align*}
    d^\pi_{\rho, \Bar{\pi}}(s) = (1-\gamma)
\mathbb{E}_{\tau \sim \PP_{\Bar{\pi}}^\pi} \left[ \sum_{h=0}^\infty \gamma^h \mathbbm{1}_{\{s^h = s, a^h = a\}} \Big| s^0 \sim \rho \right].
\end{align*}
We adapt the definition of the distribution mismatch coefficient (see e.g., \citet{IndepPolicyGrad_Ding}) to our setting.
We use the shorthand notation $d_{\rho, t}^\pi$ for $d^\pi_{\rho,\pi^t}$. Furthermore, we define the distribution mismatch coefficient $\kappa_\rho \coloneqq \sup_t \sup_{\pi \in \Pi} \lVert d^\pi_{\rho,t}/\rho \rVert_\infty$ and the minimax-version $\Tilde{\kappa}_\rho \coloneqq \inf_{\nu \in \Delta(\mathcal{S})} \sup_t \sup_{\pi \in \Pi} \lVert d^\pi_{\rho,t} / \nu \rVert_\infty$.

\subsection{Performative Markov Potential Games}
We are extending the Markov Potential Games (MPGs), which are special classes of MGs.
MPGs assume the existence of a potential function, which, in our setting, would be policy-dependent. 
More formally, a MG with \textit{performative} effects $\mathcal G$ is potential if for every $\Bar{\pi}$, the induced game $\mathcal G(\Bar \pi)$ is a Markov Potential Game.
This implies that for all $\Bar{\pi}$ there exists a potential function $\Phi_{\Bar{\pi}}$ such that
\begin{equation}\label{Def: Markov Potential Game}
    \Phi_{\Bar{\pi}}^\pi(s) - \Phi_{\Bar{\pi}}^{\pi'_i, \pi_{-i}}(s) =  V_{i,\Bar{\pi}}^{\pi_i, \pi_{-i}}(s) - V_{i,\Bar{\pi}}^{\pi'_i, \pi_{-i}}(s)
\end{equation}
for all policies $\pi_i, \pi'_i \in \Pi_i, \pi_{-i} \in \Pi_{-i}$, states $s$ and agents $i \in \mathcal{N}$.
We set $\Phi(\nu) = \E_{s \sim \nu}[\Phi(s)]$.
In the case of a static environment, i.e. $r_{i,\pi'} = r_{i, \pi''}$ for any $i \in \mathcal{N}$ and $P_{\pi'} = P_{\pi''}$ for any $\pi', \pi'' \in \Pi$, then the \textit{performative} MPG simplifies to the standard MPG framework, introduced in \citet{GlobalConvergence_Leonardos}.

\subsection{Solution Concepts}
We generalize the solution concepts of a performatively optimal policy and a performatively stable policy from single-agent PRL \citep{PerformativeReinforcementLearning}.
The following notion of a Nash equilibrium (NE) generalizes the concept of a performatively optimal policy.
\begin{definition}[$\epsilon$-NE]\label{Def: Nash policy}
A policy profile $\pi = (\pi_i)_{i \in \mathcal{N}} \in \Pi$ is an $\epsilon$-NE if it satisfies
\begin{align*}  V_{i,(\pi_i, \pi_{-i})}^{(\pi_i, \pi_{-i})}(\rho) \geq V_{i,(\pi'_i, \pi_{-i})}^{(\pi'_i, \pi_{-i})}(\rho) - \epsilon,
\end{align*}
for all $i \in \mathcal N$, all $\pi'_i \in \Delta(\mathcal{A}_i)^\mathcal{S}$.
\end{definition}
 If $\epsilon$ is implicit, we alternatively say that $\pi$ is an approximate NE. If the definition holds for $\epsilon=0$, then $\pi$ is called an exact NE.  
We extend the notion of a performatively stable policy to the multi-agent setting through the notion of a performatively stable equilibrium (PSE), defined as follows. 
\begin{definition}[$\epsilon$-PSE]\label{Def: PSE}
    A policy profile $\pi = (\pi_i)_{i \in \mathcal{N}} \in \Pi$ is an $\epsilon$-Performatively Stable Equilibrium if it satisfies
    \begin{align*}
V_{i,\pi}^{(\pi_i, \pi_{-i})}(\rho) \geq V_{i,\pi}^{(\pi'_i,\pi_{-i})}(\rho) - \epsilon,
    \end{align*}
    for any $i \in \mathcal N$ and for all $\pi'_i \in \Delta(\mathcal{A}_i)^\mathcal{S}$.
\end{definition}
If the definition holds for $\epsilon=0$, we say that $\pi$ is an exact PSE. In the latter solution concept, i.e., PSE, we essentially uncouple the deviation in policy $\pi$ from the given game $\mathcal G (\pi)$. Hence, we emphasize that NE and PNE are two distinct equilibrium concepts in general \textit{performative} Markov Potential Games. 

\subsection{Learning Objective}


At each iteration $t$ of the learning protocol, the agents deploy a joint policy $\pi^t$, which induces the game $\mathcal G({\pi^{t}})$. Then, each agent updates its policy $\pi^t_i$ to $\pi^{t+1}_i$ based on its performance (e.g., the return) in $\mathcal G({\pi^{t}})$. Quantities relevant for updating $\pi^t_i$ can be estimated from the data obtained when deploying $\pi^t$.

Given a sequence of policies $\pi^1,\dots,\pi^T$, we introduce a new regret notion called \textit{performative} regret (PReg). Formally, 
\begin{align*}
    &\mathrm{PReg}(T) \coloneqq \frac{1}{T} \sum_{t=1}^T \max_{i \in \mathcal{N}} \max_{\pi'_i} \left( V_{i,\pi^t}^{\pi'_i, \pi^{t}_{-i}}(\rho) - V_{i,\pi^t}^{\pi^{t}}(\rho) \right).
\end{align*} 
Intuitively, this notion captures the worst-case suboptimality gap across agents, averaged over the time horizon $T$, while accounting for the environment's response.  
If we bound $\mathrm{PReg}(T) \leq \epsilon$, 
then there is at least one iteration $t \in [T]$ such that $\pi^t$ is an $\epsilon$-PSE.
Notably, the average policy $\pi_{avg} = \frac{1}{T} \sum_{t=1}^T \pi_t$ does not necessarily correspond to an $\epsilon$-PSE in general, as this depends on the structure of the game $\mathcal{G}(\pi_{avg})$.

\section{Characterization of Solution Concepts} \label{sec.main_results}

In this chapter, we argue that a PSE exists in every MPG with \textit{performative} effects. Furthermore, a PSE corresponds to an approximate NE. 
We require a sensitivity assumption that bounds the magnitude of the \textit{performative} effects, which is a standard assumption in the \textit{performative} prediction/reinforcement learning literature.
Unlike prior work in PRL \citep{PerformativeReinforcementLearning, rank2024performative}, our assumption is made with respect to the policy space. 

\begin{assumption}[Sensitivity]\label{Assump: Sensitivity}
For any two policies $\pi$ and $\pi'$, we have that
    \begin{align*}
        \lVert r_{i,\pi}(\cdot, \cdot) - r_{i,\pi'}(\cdot, \cdot) \rVert_2 \leq \perfc_r \cdot \lVert \pi-\pi' \rVert_2,
        \\ 
        \lVert P_\pi(\cdot \mid \cdot, \cdot) - P_{\pi'}(\cdot \mid \cdot, \cdot ) \rVert_2 \leq \perfc_p \cdot \lVert \pi-\pi' \rVert_2,  
    \end{align*}
    for all $i \in \mathcal{N}$.
\end{assumption}

At a high-level idea, we prove the existence of a PSE by using the fact that $\Phi_{\pi}^{\pi'}$ is continuous in $\pi'$ for a fixed underlying game induced by $\pi$ and that $\Phi_{\pi}^{\pi'}$ is continuous in $\pi$ for a fixed policy $\pi'$. More specifically, we show that the function $\widehat{\Phi} = \argmax_{\pi' \in \Pi} \Phi_{\pi}^{\pi'}(\rho)$ is upper hemicontinuous, which allows us to use the Kakutani fixed point theorem \citep{glicksberg1952} to show the existence of a fixed point that corresponds to a PSE. 

\begin{lemma}\label{lemma: potential fixed point}
    For any state distribution $\rho \in \Delta(S)$, there exists a policy $\pi^* \in \Pi$ 
    such that 
    \begin{equation}\label{Eq.: Potential decoupled maximum}
         \Phi_{\pi^*}^{\pi^*}(\rho) - \Phi_{\pi^*}^{\pi_i, \pi^*_{-i}}(\rho)
        \geq 0,
    \end{equation}
    for all $i \in \mathcal{N}$, $\pi_i \in \Pi_i$.  
\end{lemma}

The existence result follows by a simple argument over the potential functions, see Appendix~\ref{App: Existence PSE}.  


\begin{theorem}\label{Theorem: Existence stable policy}
    Under Assumption~$\ref{Assump: Sensitivity}$, every \textit{performative} MPG admits a PSE. 
\end{theorem}

It is important to understand the behavior at a PSE. We know that under Assumption~\ref{Assump: Sensitivity}, the difference between two value functions under two different $\mathcal{G}(\pi')$ and $\mathcal{G}(\pi'')$ is bounded, as formalized in the following lemma. Its proof is provided in Appendix~\ref{App: Bounding costs due to performativity}.

\begin{lemma}\label{lemma: Bound noise terms}
Under Assumption~\ref{Assump: Sensitivity}, in every \textit{performative} MG, we have 
    \begin{align*}
      \left| V_{i,\pi'}^{\pi}(s) - V_{i,\pi''}^{\pi}(s) \right|
      \leq \delta_{r, p} \cdot \lVert \pi' - \pi'' \rVert_2,
\end{align*}
for all $\pi', \pi'' \in \Pi$, where $\delta_{r, p} \coloneqq \frac{1}{1-\gamma}  \left(\perfc_r +\frac{ \gamma \cdot \perfc_p \sqrt{S}}{1-\gamma} \right)$.
\end{lemma}

This result allows us to show that even though PSE and NE may differ, every PSE corresponds to an approximate NE, with an approximation factor that vanishes as the \textit{performative} effects approach zero. Hence, the two solution concepts coincide as $\perfc_r, \perfc_p \rightarrow 0$. We provide the proof in Appendix~\ref{App: PSE is approximate NE}.

\begin{theorem}\label{Thm: PSE is approximate NE}
    Given Assumption~$\ref{Assump: Sensitivity}$ in a \textit{performative} MPG, every PSE is a $\mathcal{\delta}_{r, p}$-NE.
\end{theorem}

The next sections focus on algorithms for computing a PSE. 

\section{Independent Gradient methods}\label{sec: IGM}

In this section, we study gradient-based methods and provide convergence guarantees. First, in Section~\ref{sec. IGM GO}, we assume that we have access to a gradient oracle. In Section~\ref{Subsec: Independent Learning without Oracle}, we relax the latter assumption.


We show that Lemma~\ref{lemma: Bound noise terms} is critical to generalize the theoretical guarantees for independent policy gradient ascent (IPGA), independent natural policy gradient ascent (INPG), and regularized INPG (INPG reg.) to \textit{performative} MPGs. Specifically, we leverage Lemma~\ref{lemma: Bound noise terms} to generalize the potential improvement argument from prior work \citep{IndepPolicyGrad_Ding, DBLP:conf/nips/ZhangMDS022, NPG_3} to account for \textit{performative} effects. 
We refer to Appendix~\ref{app: proofs and derivations} for the complete proofs of this section.

\subsection{Independent Gradient Methods with Oracle} \label{sec. IGM GO}

Next, we present the independent gradient methods and show the corresponding convergence guarantees. 

\subsubsection{Independent Policy Gradient Ascent}
First, we consider the IPGA introduced by \citet{IndepPolicyGrad_Ding}. Each agent $i$ updates its policy at 
iteration $t$ according to the following update rule
\begin{equation}\label{eq:PGA ALGO}
\begin{split}
    \pi^{t+1}_i(\cdot | s) &= \argmax_{\pi_i(\cdot | s ) \in \Delta(\mathcal{A}_i)} \Big\{ \big\langle \pi_i(\cdot | s), \Bar{Q}_{i, t}^t(s, \cdot) \big\rangle_{\mathcal{A}_i} \\ 
    &\quad - \frac{1}{2\eta} \cdot \lVert \pi_i(\cdot | s) - \pi_i^t(\cdot|s) \rVert_2^2\Big\},
\end{split}
\end{equation}
where $\eta$ is the learning rate and $\Bar{Q}_{i,t}^{t}(s, a_i) = \sum_{a_{-i}} \pi^{t}_{-i}(a_{-i} \mid s) \cdot Q_{i, t}^{t}(s, a_i, a_{-i})$ is the averaged $Q_i$-value at step $t$. Let $C_{\Phi} \in \mathbb{R}$ with $C_{\Phi} \geq | \Phi^\pi_{\Bar{\pi}}(\mu) - \Phi_{\Bar{\pi}}^{\pi'}(\mu) |$ for any $\pi, \pi', \Bar{\pi} \in \Pi, \mu \in \Delta(\mathcal{S})$. Such a constant $C_{\Phi}$ always exists and is trivially upper-bounded by $\frac{n}{1 - \gamma}$ \citep[Lemma 18]{IndepPolicyGrad_Ding}.

\begin{theorem}\label{Theorem: Convergence Result - Infinite Sample Case}
    For $\eta=\frac{(1-\gamma)^4}{8 \min\{\kappa_\rho, S\}^3 n A}$, running the IPGA algorithm \eqref{eq:PGA ALGO} satisfies
    \begin{align*}
    &\mathrm{PReg}(T) \lesssim  
    \frac{\min\{ \kappa_\rho, S \}^2 \sqrt{A n C_\Phi}}{ \sqrt{T} \cdot (1-\gamma)^3} \\
    &+ \frac{\sqrt{8 \min\{\kappa_\rho, S\}^3 n A }}{(1-\gamma)^5} \cdot \sqrt{\frac{\mathcal{W}_{r, p}}{T}},
\end{align*}
    where $\mathcal{W}_{r, p} \coloneqq T\cdot (n+1)\cdot n^2 \cdot S \cdot \mathcal{\delta}_{r, p}.$ 
\end{theorem}
Thus, the IPGA algorithm converges to an approximate PSE in the best-iterate sense.

\subsubsection{Independent Natural Policy Gradients}
Here, we consider total potential functions of the form
\begin{align*}
    \Phi_{\pi'}^{\pi}(\rho) = \mathbb{E}_{s^0 \sim \rho} \left[ \sum_{h=0}^\infty \gamma^h \phi_{\pi'}^\pi(s^h, a^h) \Big| \pi \right],
\end{align*}
induced by functions $\phi_{\pi'}^\pi: \mathcal{S} \times \mathcal{A} \rightarrow [0,1]$ for $\pi, \pi' \in \Pi$.
%
For every iteration $t$, 
INPG dynamics have the following multiplicative update rule under the softmax parameterization:
\begin{equation} \label{eq:inpg}
    \pi_i^{t + 1}(a_i | s) = \pi_i^{t}(a_i | s) \frac{\exp\left(\frac{\eta}{1 - \gamma} \Bar{A}_{i, t}^{t}(s, a_i) \right)}{Z_i^t(s)},
\end{equation}
for every $\forall i \in \mathcal{N}, s \in \mathcal{S}, a_i \in \mathcal{A}_i$, 
where $\Bar{A}_{i, t}^{t}(s, a_i) = \sum_{a_{-i}} \pi^t_{-i}(a_{-i} \lvert s) \left(Q_{i, t}^{t}(s,a) - V_{i, t}^{t}(s) \right)$ is the marginalized advantage function and $Z_i^t(s)$ is the normalization term given as 
\begin{equation}
    Z_i^t(s) = \sum_{a_i} \pi_i^{t}(a_i | s) \exp\left(\frac{\eta}{1 - \gamma} \Bar{A}_{i, t}^{t}(s, a_i) \right) \ge 1.
\end{equation}

We now introduce the following standard assumption in the analysis of INPG dynamics.
\begin{assumption} \label{assumption:positive_visit}
    For any initial state distribution $\rho$, $\inf_{\pi} \inf_{\pi'} \min_{s} d_{\rho, \pi}^{\pi'}(s) > 0$.
\end{assumption}
Based on this assumption, for any initial distribution $\rho$, we also define
\begin{equation*}
    M_\rho \coloneqq \sup_{\pi} \sup_{\pi'} \max_s \frac{1}{d_{\rho, \pi}^{\pi'}(s)} < \infty,
\end{equation*}
and we drop the dependence on $\rho$ when it is clear from the context. Moreover, we denote the lower bound for the probability of playing the optimal action by
\begin{align*}
    c
    \coloneqq
    \min_i \min_t \min_s \quad \smashoperator{\sum_{a_i^* \in \argmax_{a_i \in A_i} \Bar{Q}_{i, t}^{\pi^t}(s, a_i)}} \quad \pi_i^t(a_i^* | s) > 0.
\end{align*}

\paragraph{Unregularized INPG Dynamics}
In contrast to IPGA, the regret introduced by the \textit{performative} effects vanish as $T \to \infty$ under INPG dynamics.
 \begin{theorem} \label{thm:inpg_conv}
    Suppose that Assumption~\ref{Assump: Sensitivity} and Assumption~\ref{assumption:positive_visit} hold.
    For $T \ge 1$ and $\eta \le \frac{(1 - \gamma)^2}{\sqrt{n}} + \frac{\sqrt{2} (1 - \gamma)}{M \delta_{r, p}}$, the INPG dynamics \eqref{eq:inpg} satisfy
    \begin{align*}
        &\mathrm{PReg}(T) \le
        \sqrt{\frac{1}{T} \frac{3 \Tilde{\kappa} \left(\frac{\sqrt{n}}{1 - \gamma} + \sqrt{2} M \delta_{r, p} \right)}{c (1 - \gamma)^3}}.
    \end{align*}
\end{theorem}
  Hence, INPG dynamics converge to an approximate-PSE in the best-iterate sense. In order to show asymptotic last-iterate convergence of INPG dynamics, we require the following standard assumption for the analysis of the INPG dynamics.
\begin{assumption} \label{assumption:isolated_points}
    The stationary points of Eq.~\eqref{eq:inpg} are isolated.
\end{assumption}
 Given Assumption~\ref{assumption:isolated_points}, we obtain that the INPG dynamics converges to an exact PSE in the last-iterate sense as $T \to \infty$. This holds because Lemma~\ref{NPG Policy Improv} and the boundedness of $\phi_{\pi'}^\pi$ implies $\lim_{t \to \infty} Z_i^t(s) = 1 \implies \lim_{t \to \infty} \pi^t(a_i | s) \Bar{A}_{i, t}^{t}(s, a_i) = 0$, in which case the gradient norm of the potential function with respect to the policy approaches 0 in the limit. Along with Assumption~\ref{assumption:isolated_points}, this observation implies that the sequence of policies $\pi^t$ converges to some stationary policy $\pi^\infty$. The rest of the proof follows the same lines as in \citet[Section 12.0.2]{DBLP:conf/nips/ZhangMDS022}.

\paragraph{Regularized INPG Dynamics.}
Note that $c$ can be arbitrarily small for bad initializations of the policy. This can slow down the convergence of INPG dynamics due to its $\tfrac{1}{c}$ dependence.
We consider the INPG dynamics with log-barrier regularization to overcome this.
Following \citet{DBLP:conf/nips/ZhangMDS022}, we define the regularized objective and potential function as follows:
\begin{align*}
    \Tilde{V}_{i, \pi}^{\pi'} &= V_{i, \pi}^{\pi'}(\rho) + \lambda \sum_{s, a_i} \log \pi'_i(a_i | s), \\
    \Tilde{\Phi}_\pi^{\pi'}(\rho) &= \Phi_\pi^{\pi'}(\rho) + \lambda \sum_i \sum_{s, a_i} \log \pi'_i(a_i | s).
\end{align*}
We refer the reader to the work by \citet{DBLP:conf/nips/ZhangMDS022} for further discussion on the choice of the regularizer.
Under log-barrier regularization the INPG dynamics have the following update rule:
\begin{equation} \label{eq:inpg-reg}
    \pi_i^{t + 1}(a_i | s) = \pi_i^{t}(a_i | s) \cdot  \frac{\eta f_{i}^{t}(a_i | s)}{Z_i^t(s)},
\end{equation}
where $f_i^t(s, a_i)$ and $\Tilde{Z}_i^t(s)$ are defined as
\begin{align*}
    f_i^t(s, a_i) 
    &=
    \frac{1}{1 - \gamma} \Bar{A}_{i, t}^{\pi^t}(s, a_i)
    +
    \frac{\lambda}{d_{\rho, t}^{t}(s) \pi_i^t(a_i | s)}
    -
    \frac{\lambda \left\lvert \mathcal{A}_i \right\rvert}{d_{\rho, t}^{t} (s)}, \\
    \Tilde{Z}_i^t(s)
    &=
    \sum_{a_i} \pi_i^t(a_i | s) \exp \left( \eta f_i^t(s, a_i) \right).
\end{align*}


\begin{theorem}\label{thm:inpg-reg-conv}
    Suppose that Assumption~\ref{Assump: Sensitivity} and Assumption~\ref{assumption:positive_visit} hold. Then, for all $T\ge 1$, it holds that
    \begin{align*}
        &\mathrm{Perform\textit{-}Regret}(T) 
        \le \frac{9\sqrt{2}}{\eta \lambda (1 - \gamma) T} + \lambda M A_{\max}
        \\
        & + \eta n S \delta_{r, p} \left( \frac{1}{(1 - \gamma)^2} + 4 \lambda^2 A_{\max} M^2 + \frac{4 \lambda M}{1 - \gamma} \right),
    \end{align*}
    Moreover, for any $\epsilon > 0$, set $\lambda = \frac{\epsilon}{3 M A_{\max}}$ and
    \begin{align*}
        \eta
        = \min \Biggl\{
        &\left( \frac{16\epsilon M}{3} + \frac{16M}{(1 - \gamma)^2} + \frac{12nM}{(1 - \gamma)^3}\right)^{-1},  \\
        &\left( \frac{1}{(1 - \gamma)^2} + \frac{4 \epsilon^2}{9 A_{\max}} + \frac{4 \epsilon}{3 A_{\max} (1 - \gamma)^2} \right)^{-1} \\
        & \times \left(3  n S \delta_{r, p} \right)^{-1}, \left(\frac{15}{(1 - \gamma)^2} + 5\epsilon \right)^{-1} \Biggr\}.
    \end{align*}
    Then, for $T \ge \mathcal{O} \left(\frac{n A_{\max} M^2}{\epsilon^2 (1 - \gamma)^4} \max\left\{1, S \delta_{r, p} \right\} \right)$, the regularized INPG dynamics \eqref{eq:inpg-reg} satisfy
    \begin{equation*}
        \mathrm{PReg}(T) \le \epsilon.
    \end{equation*}
\end{theorem}
Note that the result no longer depends on the constant $c$, as highlighted in \citet{DBLP:conf/nips/ZhangMDS022}.

\subsection{Independent Learning without Oracle}\label{Subsec: Independent Learning without Oracle}

In this section, we remove the exact gradient requirement for the IPGA algorithm.
In this case, the gradient can be estimated using offline data. 
At iteration $t$, 
we generate $K$ trajectories by rolling out $\pi^t$ in $\mathcal{G}(\pi^t)$ $K$ times.

Each trajectory is a sequence of state-action-reward tuples 
$(\bar s^h, \bar a^h , \bar r^h)_{h=0}^{H-1}$
of length $H$, where $\bar s^0 \sim \rho$ and the $H$ is the horizon. Horizon $H$ is a random variable  $H = \max_{i \in \mathcal N}(h_i+h'_i)$, where $h_i$ and $h_i'$ are sampled from a geometric distribution with parameter $1-\gamma$. Time steps $h_i$ to $h_i + h_i'$ are used for estimating Q-values of agent $i$. We follow the approach from  \cite{IndepPolicyGrad_Ding} to obtain an unbiased estimate of $\Bar{Q}_{i,t}^t$, denoted as $R_{i,t}^k \coloneqq \sum_{h = h_i}^{h_i + h'_i - 1} \overline{r}_i^h$. 
In the policy update phase, we obtain an estimate for $\Bar{Q}_{i,t}^t$ by minimizing the expected regression loss:
\begin{align*}
\widehat{Q}^t_i(\cdot,\cdot) \approx \smashoperator[l]{\argmin_{ \lVert \Bar{Q}_{i,t}^t(\cdot, \cdot) \rVert_2 \leq \frac{\sqrt{|\mathcal{S}| \cdot |\mathcal{A}|}}{1-\gamma}}} \underbrace{\sum_{k=1}^K \Big(R^k_{i,t} - \widehat{Q}_{i,t}^t(\overline{s}_i,\overline{a}_i^k) \Big)^2}_{\eqqcolon L_{i}^t(\widehat{Q}_{i,t}^t)}.
\end{align*}

We assume that $\mathbb{E}\left[L_{i}^t(\widehat{Q}_{i,t}^t)\right] \leq \delta_{\textit{stat}}$, where the expectation is taken over the randomness induced by constructing $\widehat{Q}_{i,t}^t$. As shown by \citet{journals/jmlr/AudibertC11}, the statistical error can be bounded by $\delta_{\textit{stat}} \in \mathcal{O}(\frac{S^2A^2}{(1-\gamma)^4 \cdot K})$. 

We take the estimated $\widehat{Q}$ to update the policy of each agent $i \in \mathcal{N}$:
\begin{equation}\label{Gradient Update: Finite sample}
\begin{split}
    \pi^{t+1}_i(\cdot | s) =\ & \argmax_{\pi(\cdot | s ) \in \Delta_\xi(\mathcal{A}_i)} \Big\{ \big\langle \pi_i(\cdot | s), \widehat{Q}_{i, t}^t(s, \cdot) \big\rangle_{\mathcal{A}_i} \\
    &- \frac{1}{2\eta} \cdot \lVert \pi_i(\cdot | s) - \pi_i^t(\cdot|s) \rVert_2^2\Big\}.
\end{split}
\end{equation}
The projection is taken over $\Delta_\xi(\mathcal{A}_i) \coloneqq \{ (1-\xi) \pi_i(\cdot \mid s) + \frac{\xi}{|\mathcal{A}_i|}, \forall \pi_i(\cdot | s) \}$, forcing every policy to establish a greedy-exploration policy with probability $\xi$. 

We bound changes compared to Theorem~\ref{Theorem: Convergence Result - Infinite Sample Case}, by taking account for estimation errors. 

\begin{theorem}\label{theorem: PGA finite sample case}
Running algorithm IPGA without gradient oracles (\ref{Gradient Update: Finite sample}) and setting $\eta = \frac{(1-\gamma)^4}{16 \kappa_{\rho}^3 n A}$, we obtain,
\begin{align*}
\E{ \left[ \mathrm{PReg}(T) \right] }
\lesssim  
\mathcal{R}(\eta) 
+ \left( \frac{\sqrt{A} \cdot (\kappa^2_\rho \cdot n \cdot \delta_{\mathrm{stat}})^{1/3}}{(1-\gamma)^2} \right)
\end{align*}
where we have that
\begin{align*}
    \mathcal{R}(\eta) &=
\frac{\min\{ \kappa_\rho, S \}^2 \sqrt{A n C_\Phi}}{ \sqrt{T} \cdot (1-\gamma)^3} \\
    &+ \frac{\sqrt{8 \min\{\kappa_\rho, S\}^3 n A \mathcal{W}_{r, p}}}{(1-\gamma)^5\sqrt{T}}
\end{align*}
and $\mathcal{W}_{r, p} \coloneqq (n+1)\cdot n^2 \cdot S \cdot \mathcal{\delta}_{r, p}$.
\end{theorem}
Extending the results of the INPG algorithm is not straightforward, even in standard MPG; we leave that for future work. 

\section{Repeated Occupancy Measure Optimization}

The results in the previous section do not provide {\em finite-time} last-iterate convergence guarantees. Instead, last-iterate convergence is established only for INPG, and this is an asymptotic convergence result.  
In this section, we study {\em finite-time} last-iterate convergence guarantees for a special case of \textit{performative} MPGs. In particular, we focus on MPGs with \textit{performative} effects and agent independent transitions. 

\begin{assumption}\label{Ass: Agent independent transitions}
    For any MPG with \textit{performative} effects $\mathcal{G}(\Bar{\pi})$, it holds that for any $s, s' \in \mathcal{S}$ and any action $a \in \mathcal{A}$,  
    \begin{align*}
        P_{\Bar{\pi}}(s' \mid s,a) = P_{\Bar{\pi}}(s' \mid s).
    \end{align*}
\end{assumption}

As a method of interest, we consider the repeated optimization approach of \citet{PerformativeReinforcementLearning}, developed for single agent \textit{performative} RL, and extend it to MPGs with \textit{performative} effects. There are two critical aspects that we change in this method: i) our approach aims to optimize a surrogate objective instead of the regularized objective in \cite{PerformativeReinforcementLearning}, ii) our approach can be implemented using multi-agent optimization, where each agent independently optimizes a surrogate objective over the state-action occupancy measures. Note that this is in contrast to the previous sections, where agents were directly updating their policies.  
The policy of agent $i$ can be obtained using the following expression:
\begin{equation}\label{eq.: occupany measure to policy}
    \pi_i(a_i \mid s) 
    \begin{cases}
         \frac{\mu^\pi_{i, \Bar{\pi}}(s, a_i)}{\sum_{a_i} \mu^\pi_{i, \Bar{\pi}}(s,a_i)} & \text{if } \sum_a \mu^\pi_{i, \Bar{\pi}}(s,a_i) > 0, \\
        \frac{1}{A} & \text{otherwise},
    \end{cases} 
\end{equation}
where $\mu^\pi_{i, \Bar{\pi}}$ is the state-action occupancy measure of agent $i \in \mathcal{N}$ over $\mathcal{S} \times \mathcal{A}_i$ for the joint policy $\pi$, given the underlying game $\mathcal{G}(\Bar{\pi})$. We denote by $D_i$ the set of feasible occupancy measures over $\mathcal S \times \mathcal A_i$ given game $\mathcal G(\Bar{\pi})$.  Furthermore, $\mu_{\Bar{\pi}}^\pi = (\mu^\pi_{1, \Bar{\pi}}, \dots, \mu^\pi_{n, \Bar{\pi}})$ is the joint occupancy measure.  
When it is clear from context, we simplify the notation by writing $\mu^\pi_\pi = \mu$ and $\mu^{\pi'}_{\pi'} = \mu'$, respectively. Similarly, for each $i \in \mathcal N$, we write $\mu^\pi_{i,\pi} = \mu_i$, $\mu^{\pi'}_{i, \pi'} = \mu'_i$ for all $i \in \mathcal N$. 
More specifically, to update the current policy $\pi^t$, we aim to solve the following optimization problem:
\begin{equation}\label{Eq: Repreated Optimization}
    \argmax_{\mu \in \mathcal{D}} \left \langle  \nabla_{\pi} \Phi_t^t(\rho), \mu \right \rangle - \frac{\lambda}{2} \left\lVert \mu \right\rVert_2^2,
\end{equation}
where $\mathcal{D}$ is the set of valid occupancy measures in $\mathcal{G}(\pi^t)$. Here, the gradient of the potential is evaluated at the current policy and game $\mathcal{G}(\pi^t)$, specifically: $\nabla_{\pi} \Phi_t^t(\rho) = \nabla_{\pi} \Phi^{\pi}_{\pi'}(\rho) \big|_{\pi = \pi_t, \pi' = \pi_t}$.
To obtain agents' policies from the solution to this optimization problem, we use the relation considered in Eq.~\eqref{eq.: occupany measure to policy}. One can show that the optimization problem is feasible without knowledge of $\Phi$ because it holds that $\nabla_{\pi_i} \Phi^{\pi}_{\pi'}(\rho) = \nabla_{\pi_i} V_{i, \pi'}^\pi(\rho)$ for any $i \in \mathcal N$ and any $\pi$ \citep{GlobalConvergence_Leonardos}. However, to establish the convergence of repeatedly optimizing \eqref{Eq: Repreated Optimization}, we require the following sensitivity assumption over state-action occupancy measures and the following smoothness assumption on the gradients.

\begin{assumption}
\label{Assump: Sensitivity Occupancy measures}
For any two policies $\pi$ and $\pi'$, we have that 
    \begin{align*}
    &\text{i.} \ \   \quad \lVert r_{i,\pi}(\cdot, \cdot) - r_{i,\pi'}(\cdot, \cdot) \rVert_2 \leq \zeta_r \cdot \lVert \mu - \mu' \rVert_2, \\
    &\text{ii.} \ \quad \lVert P_\pi(\cdot \mid \cdot, \cdot) - P_{\pi'}(\cdot \mid \cdot, \cdot) \rVert_2 \leq \zeta_p \cdot \lVert \mu - \mu' \rVert_2, \\
    &\text{iii.} \quad \lVert \nabla_\pi \Phi^\pi_{\pi'}(\rho) - \nabla_\pi \Phi^\pi_{\pi'}(\rho) \rVert_2 \leq \beta \cdot \lVert \pi - \pi' \rVert_2,
\end{align*}
    for all $i \in \mathcal{N}$. 
\end{assumption}

For general potential functions, it holds that $\beta \leq \frac{2n\gamma A_{\max}}{(1-\gamma)^3}$,  \citep{GlobalConvergence_Leonardos}.

 Under Assumption~\ref{Assump: Sensitivity Occupancy measures}, repeatedly optimizing~\eqref{Eq: Repreated Optimization} results in an occupancy measure $\mu^\lambda$ associated with an approximate PSE $\pi^\lambda$ for a sufficiently large value of $\lambda$. 
 Now, instead of optimizing for the joint occupancy measure, we can consider the following optimization problems over individual occupancy measures, which allow agents to independently perform the policy update step:
 \begin{equation}\label{Eq: Independent repeated optimization}
    \argmax_{\mu_i \in \mathcal{D}_i} \left\langle \nabla_{\pi_i} \Phi_t^t(\rho),  \mu_i \right\rangle
    - \frac{\lambda}{2} \lVert \mu_i \rVert_2^2.
\end{equation}
To obtain agent $i$'s policy from the solution of this optimization problem, we can use the relation \eqref{eq.: occupany measure to policy}. We can show that the obtained policies are the same as the ones obtained from the solution of \eqref{Eq: Repreated Optimization} for the special class of \textit{performative} MPGs, which we consider in this section. This holds because $\sum_a \mu^{\pi}_{\pi'}(s,a) = \sum_a \mu^{\Bar{\pi}}_{\pi'}(s,a) \eqqcolon \alpha_\pi(s)$ for a fixed game $\mathcal{G}(\pi)$ and any two policies $\pi, \Bar{\pi}$, i.e., the visitation distribution is independent of the played policy. Combining this with the convergence of repeatedly optimizing \eqref{Eq: Repreated Optimization}, we obtain the following theorem.

\begin{theorem}\label{Thm: Subgame Finite-LIC}
    For the choice of $\lambda \geq O \left( (\zeta_p + \zeta_r) \cdot \frac{S^{2} \sqrt{n} \gamma A^{9/4}}{(1-\gamma)^6} + \frac{S^{3/2} \gamma A^{5/4} \beta}{(1-\gamma)^3 \min_s \alpha(s) } \right)$, let $\mu^\lambda$ be the fixed point of the objective in Eq.~\eqref{Eq: Repreated Optimization}. If agents $i \in \mathcal{N}$ repeatedly optimize~\eqref{Eq: Independent repeated optimization} for $T \geq 2(1-\mu)^{-1} \ln(2/\delta(1-\gamma))$ iterations, it holds that $\lVert \mu^T - \mu^\lambda \rVert_2 \leq \delta$ and the \textit{performative} gap is bounded by 
    \begin{align*}
         &\max_{i \in \mathcal{N}} \max_{\pi'_i} \left( V_{i,\pi^T}^{\pi'_i, \pi^{(T)}_{-i}}(\rho) - V_{i,\pi^T}^{\pi^{(T)}}(\rho) \right) 
         \\ &\leq \frac{\kappa_\rho}{\min_s \alpha_\lambda(s)(1-\gamma)} \cdot \left( \sqrt{\max_i A_{i}} \cdot \delta + \frac{\lambda}{2(1-\gamma)}\right).
    \end{align*}
\end{theorem}

\begin{proof}[Proof sketch]$ $\par\nobreak\ignorespaces
\begin{itemize}
    \item Following the proof argument by \citet{PerformativeReinforcementLearning}, we adopt the dual perspective on Eq.~\eqref{Eq: Repreated Optimization}, allowing us to analyze a strongly convex program with a fixed feasible region. 
    However, we additionally need to show that $\nabla_{\Bar{\pi}} \Phi^{\pi}_{\pi'}$ is Lipschitz continuous in the state action occupancy measure $\mu'$ (associated with policy $\pi'$). We obtain the Lipschitz constant that depends on the smoothness of the gradient and the sensitivity parameters. 
    \item Translating back the dual solution, using strong duality, we show that $\mu^T$ converges to a fixed point $\mu^\lambda$.
    \item Finally, we construct policy $\pi^T$ from $\mu^T$ and combine a gradient domination argument with the fact that $\mu^\lambda$ maximizes expression in Eq.~\eqref{Eq: Repreated Optimization}. \qedhere%
\end{itemize}%
\end{proof}
This theorem shows the last-iterate convergence of the novel repeated optimization approach to an approximate PSE for sufficiently small $\zeta_p, \zeta_r$ and $\beta$.

\section{Experiments}\label{sec:experiments}

We study the empirical performance of IPGA and INPG on two MPGs from the literature -- the safe-distancing game \citep{GlobalConvergence_Leonardos, IndepPolicyGrad_Ding}, and stochastic congestion games \citep{NPG_converges}.
The algorithms we evaluate are non-oracle gradient algorithms, meaning the gradient is estimated based on the rollouts of the policies.
We replicate each experiment across 10 seeds for the safe-distancing game and 5 seeds for the stochastic congestion game.
The plots report the mean and one standard deviation over these seeds.
Full details of practical implementations and hyperparameters can be seen in Appendix~\ref{app:algorithms}.

\paragraph{Environments.}
We take the safe-distancing game from \citet{GlobalConvergence_Leonardos} and extend this environment to adapt to \textit{performative} effects: the agents' actions are replaced with probability $\alpha$ with those sampled from a perturbed environment's optimal Q-values, similar to \citet{PerformativeReinforcementLearning}.
More details are presented in Appendix~\ref{app:safe-dist-env}, with an illustration of this environment in Figure~\ref{fig:safe-dist-env}.
Additionally, we take the stochastic congestion game \citep{NPG_converges} and modify it to have a \textit{performative} response. More precisely, the environment transition probabilities and reward function change according to the deployed policy based on the sensitivity Assumption~\ref{Assump: Sensitivity}: upon playing policy $\pi'$ the transition probabilities become $P_{i,\pi'} = P_{i,\pi_0} + \frac{\perfc_p}{(1-\gamma) \sqrt{|S|\,|\mathcal{A}_i|}} (\pi' - \pi_0) \frac{1}{|S|}$, and the rewards $r_{i,\pi'} = r_{i,\pi_0} + \frac{\perfc_r}{(1-\gamma) \sqrt{|S|\,|\mathcal{A}_i|}} (\pi' - \pi_0)$, where $\perfc_p$ and $\perfc_r$ are the sensitivity parameters indicating the strength of the \textit{performative} effects.
In our experiments, we use the same value $\perfc = \perfc_p = \perfc_r$ for both parameters.
More details are presented in Appendix~\ref{app:cong-game}, with an illustration of this environment in Figure~\ref{fig:cong-game}.
\paragraph{Results.}
We start by showing how the \textit{performative} effect influences convergence for IPGA in Figure~\ref{fig:pgas}.
For completeness, we follow \citet{IndepPolicyGrad_Ding} and include both the version of the algorithm we analyze (IPGA-D), and the variant proposed by \citet[IPGA-L]{GlobalConvergence_Leonardos}.
Generally, we notice that stronger \textit{performative} effects, i.e., increasing $\alpha$ in the safe-distancing game or increasing $\perfc$ in the stochastic congestion game, require a larger amount of iterations until convergence is observed.
Additionally, in the safe-distancing game, higher \textit{performative} strengths (e.g., $\alpha=0.15$) lead to inexact convergence -- the algorithms converge to a region of stable equilibria, exhibiting possible oscillations inside that region.

\begin{figure}[!t]
    \centering
    \vspace{.3in}
    \includegraphics[width=\linewidth]{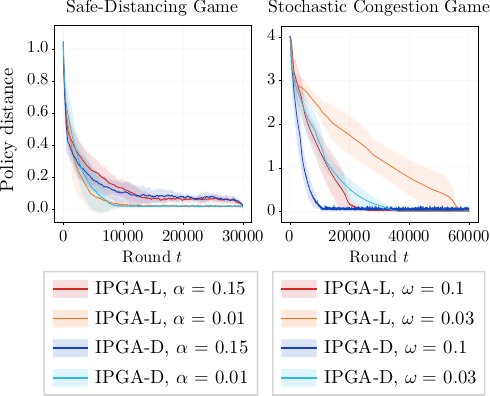}
    \vspace{.3in}
    \caption{Comparison of IPGA-L and IPGA-D, showing the distance from the current policy to the average of the last 10 in that run: $\frac{1}{n} \sum_i^n \left\| \pi_i^t - \pi_i^\text{last} \right\|$, $\gamma = 0.99$. \textbf{Left}: IPGA-L $\eta = 0.00001$, IPGA-D $\eta = 0.0001$.
    \textbf{Right}: $\eta = 0.0003$.
    }
    \label{fig:pgas}
\end{figure}

In the stochastic congestion game, we observe that IPGA-D converges faster than IPGA-L. 
A further study of the choice of the learning rate $\eta$ can be seen in the appendix, Figure~\ref{fig:omegas_lrs}, which shows the performance of both methods.

We now shift our focus to the algorithms of interest that we analyze.
A comparison of IPGA-D and both versions of INPG under \textit{performative} effects can be seen in Figure~\ref{fig:pga_vs_inpg}.
Typically, INPG converges faster than IPGA algorithms, and this is observed in our experiments as well.
This reflects our theoretical guarantees, which indicate a better performance for natural policy gradient methods. 
Additionally, we observe that INPG (reg.)\ takes less rounds than INPG (unreg.) to converge, even under strong \textit{performative} effects.
However, in some cases the log-barrier regularization loses this advantage by having a large \textit{performative} gap (e.g., Figure~\ref{fig:pga_vs_inpg} right).
We present more thorough experiments varying the learning rate $\eta$ and performativity strength $\alpha$, $\perfc$ in Appendix~\ref{app:algorithms}.

\begin{figure}[!t]
    \centering
    \vspace{.3in}
    \includegraphics[width=\linewidth]{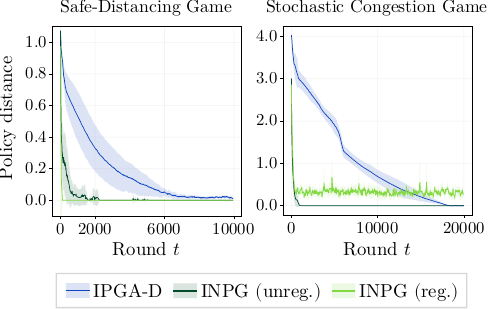}
    \vspace{.3in}
    \caption{Comparison of IPGA-D, INPG (unreg.), INPG (reg.), showing the distance from the current policy to the average of the last 10 in that run: $\frac{1}{N} \sum_i^N \left\| \pi_i^t - \pi_i^\text{last} \right\|$, $\gamma = 0.99$. \textbf{Left}: $\alpha = 0.01$, $\eta = 0.0001$.
    \textbf{Right}: $\perfc = 0.03$, $\eta = 0.0006$.
    }
    \label{fig:pga_vs_inpg}
\end{figure}

\section{Conclusion}

We provided a theoretical treatment of independent learning algorithms in performative Markov Potential Games.
Our results establish best-iterate and asymptotic last-iterate convergence for different independent policy gradient algorithms, which showcases their robustness to performative effects.
We further show that it is possible to obtain finite-time last-iterate convergence results for a special class of MPGs.
It remains an open question whether our finite-time last-iterate convergence can be extended to the general case.
For that, an interesting avenue may be to study weaker equilibria concepts.

\subsubsection*{Acknowledgements}
This research was, in part, funded by the Deutsche Forschungsgemeinschaft (DFG, German Research Foundation) – project number 467367360. Additionally, this work was supported from the DFG under project number 514505843.

\bibliography{main}


\section*{Checklist}

 \begin{enumerate}

 \item For all models and algorithms presented, check if you include:
 \begin{enumerate}
   \item A clear description of the mathematical setting, assumptions, algorithm, and/or model. [Yes]
   \item An analysis of the properties and complexity (time, space, sample size) of any algorithm. [Yes]
   \item (Optional) Anonymized source code, with specification of all dependencies, including external libraries. [Yes]
 \end{enumerate}

 \item For any theoretical claim, check if you include:
 \begin{enumerate}
   \item Statements of the full set of assumptions of all theoretical results. [Yes]
   \item Complete proofs of all theoretical results. [Yes]
   \item Clear explanations of any assumptions. [Yes] 
 \end{enumerate}

 \item For all figures and tables that present empirical results, check if you include:
 \begin{enumerate}
   \item The code, data, and instructions needed to reproduce the main experimental results (either in the supplemental material or as a URL). [Yes] URL: \href{https://github.com/PauliusSasnauskas/performative-mpgs}{github.com/PauliusSasnauskas/performative-mpgs}
   \item All the training details (e.g., data splits, hyperparameters, how they were chosen). [Yes]
   \item A clear definition of the specific measure or statistics and error bars (e.g., with respect to the random seed after running experiments multiple times). [Yes] We replicate the experiments on 5 different seeds for the safe-distancing game and 10 different seeds for the stochastic congestion game, and show the standard deviation across them.
   \item A description of the computing infrastructure used. (e.g., type of GPUs, internal cluster, or cloud provider). [Yes]
 \end{enumerate}

 \item If you are using existing assets (e.g., code, data, models) or curating/releasing new assets, check if you include:
 \begin{enumerate}
   \item Citations of the creator if your work uses existing assets. [Yes]
   \item The license information of the assets, if applicable. [Not Applicable]
   \item New assets either in the supplemental material or as a URL, if applicable. [Yes]
   \item Information about consent from data providers/curators. [Not Applicable]
   \item Discussion of sensible content if applicable, e.g., personally identifiable information or offensive content. [Not Applicable]
 \end{enumerate}

 \item If you used crowdsourcing or conducted research with human subjects, check if you include:
 \begin{enumerate}
   \item The full text of instructions given to participants and screenshots. [Not Applicable]
   \item Descriptions of potential participant risks, with links to Institutional Review Board (IRB) approvals if applicable. [Not Applicable]
   \item The estimated hourly wage paid to participants and the total amount spent on participant compensation. [Not Applicable]
 \end{enumerate}

 \end{enumerate}


\iftoggle{longversion}{%
\clearpage%
\onecolumn%
\appendix%
{\allowdisplaybreaks%
\aistatstitle{Independent Learning in Performative Markov Potential Games \\ Supplementary Materials}%
\vspace{-2in}

\hypersetup{linkcolor=black}  
\renewcommand{\contentsname}{Appendices}
\addtocontents{toc}{\protect\setcounter{tocdepth}{3}}  
\tableofcontents
\clearpage

\hypersetup{linkcolor=red}  

\section{Proofs and Derivations}\label{app: proofs and derivations}

\subsection{Existence of PSE}\label{App: Existence PSE}

\begin{proof}[Proof of Lemma~\ref{lemma: potential fixed point}]
     We show that the function $ \widehat{\Phi}(\pi)= \argmax_{\pi' \in \Pi} \Phi_{\pi}^{\pi'}(\rho)$ has a fixed point. The idea is to show that $\widehat{\Phi}(\cdot)$ is hemicontinuous, so that a Kakutani fixed point theorem is applicable. Observe that $\widehat{\Phi}(\pi)$ is non-empty, as any policy $\pi'$ can be deployed in the MPG $\mathcal{M}(\pi)$, so that the potential $\Phi_\pi^{\pi'}(\mu)$ is well-defined.
     Further, note that $\Phi^{\pi'}_{\pi}(\mu)$ has a global maximum with respect to parameter $\pi'$, given a fixed $\pi$.
     This holds, because the set $\Pi$ is compact and the function is continuous in $\pi'$.
     The latter holds, because for any agent $i \in \mathcal{N}$ and any deviation $\widehat{\pi}_i \in \Pi_i$, it holds that
     \begin{align*}
         \Phi^{\pi'}_{\pi}(\mu) - \Phi^{ \widehat{\pi}_i, \pi'_{-i}}_{\pi}(\mu) 
         = 
         V^{ \pi'}_{i, \pi}(\mu) - V^{ \widehat{\pi}_i, \pi'_{-i}}_{i, \pi}(\mu)
     \end{align*}
     and the value function $V_{i, \pi}^{\pi'}$ is continuous in $\pi'$ for any $i \in \mathcal{N}$. From that, one can infer that $\Phi_\pi^{\pi'}$ is continuous in $\pi'$ because we have for any $\widehat{\pi} \in \Pi$,
     \begin{align*}
         \Phi^{\pi'}_\pi(\mu) - \Phi_\pi^{\widehat{\pi}}(\mu) =
         \left(V_1^{\pi'}(\mu) - V_{1, \pi}^{\widehat{\pi}_1, \pi'_{-1}}(\mu) \right) + 
         \left( V_{2, \pi}^{\widehat{\pi}_1, \pi'_{-1}}  - V_{2, \pi}^{\widehat{\pi}_{\{1,2\}}, \pi'_{-\{1,2\}}}
         \right)
         + \dots +
         \left( V_{n, \pi}^{\pi'_n, \widehat{\pi}_{-n} } - V_{n, \pi}^{\widehat{\pi}}(\mu) \right). 
     \end{align*}
     Further, one observes that 
     $\Phi^{\pi'}_{\pi}$ is continuous in $\pi$ due to $(\perfc_r,\perfc_p)$-sensitivity. So, we have the ingredients to apply Berge's maximum theorem. 
    Finally, we can apply Berge's maximum theorem, this implies that $\widehat{\Phi}(\cdot)$ is upper hemicontinuous. We conclude that $\widehat{\Phi}$ has a fixed point $\pi^*$, by applying the Kakutani fixed point theorem \citep{glicksberg1952}. Hence, it holds
    \begin{align*}
        \pi^* \in \argmax_{\pi} \Phi_{\pi^*}^{\pi}(\mu).
    \end{align*}
    We conclude that, $\Phi_{\pi^*}^{\pi^*}(\mu) \geq \Phi^{\pi_i, \pi^*_{-i}}_{\pi^*}(\mu)$ for any $\pi_i$.
\end{proof}

Using that auxiliary Lemma, we can show the existence of the solution concepts.

\begin{proof}[Proof of Theorem~\ref{Theorem: Existence stable policy}]
    Combining the definition in Eq.~\eqref{Def: Markov Potential Game} and Lemma~\ref{lemma: potential fixed point}, the potential function $\Phi_{*}^*$ satisfies
    \begin{align*}
         0 &\leq \Phi_{\pi^*}^{(\pi^*_i, \pi^*_{-i})}(\rho) - \Phi_{\pi^*}^{(\pi_i, \pi^*_{-i})}(\rho) 
        \\ &=   V_{i,{\pi^*}}^{(\pi^*_i, \pi^*_{-i})}(\rho) - V_{i,{\pi^*}}^{(\pi_i, \pi^*_{-i})}(\rho)
    \end{align*}
    for all policies $\pi_i$ for all $i \in \mathcal{N}$. Therefore, we conclude that $\pi^*$ is a performatively stable policy. 
\end{proof}



\subsection{Every PSE Is an Approximate NE}\label{App: PSE is approximate NE}

\begin{proof}
    Suppose that $\pi^*$ is a PSE. 
    Then, for any agent $i \in \mathcal{N}$ and policy $\pi_i$,
    \begin{align*}
         V_{i,(\pi_i, \pi^*_{-i})}^{(\pi_i, \pi^*_{-i})}(\rho) - V_{i,{\pi^*}}^{\pi^*}(\rho)
         =  V_{i,(\pi_i, \pi^*_{-i})}^{(\pi_i, \pi^*_{-i})}(\rho) - V_{i,{\pi^*}}^{(\pi_i, \pi^*_{-i})}(\rho) + V_{i,\pi^*}^{(\pi_i, \pi^*_{-i})}(\rho) - V_{i,{\pi^*}}^{\pi^*}(\rho)
         \le \frac{1}{1 - \gamma}\left(\perfc_r + \frac{\gamma \perfc_p \sqrt{S}}{1 - \gamma} \right),
    \end{align*}
    where the inequality follows from Definition~\ref{Def: PSE} and Lemma~\ref{lemma: Bound noise terms}.
\end{proof}

\subsection{Bounding the Additional Costs Due to Performative Effects}\label{App: Bounding costs due to performativity}

We bound the sensitivity term, which occurs in our main analysis, Lemma~\ref{Lemma: Policy improvement appendix}.
This boils down to bounding $|V_{i,x}^\pi - V_{i,x'}^\pi|$ for any two policies $x, x' \in \Pi$.

\begin{proof}[Proof of Lemma~\ref{lemma: Bound noise terms}]
First, for any $s \in \mathcal{S}$, it holds that for any two policies $x, x' \in \Pi$
\begin{align*}
    V_{i,x}^{\pi}(s) 
    =
    \sum_a \pi(a | s) \cdot Q_{i,x}^{\pi}(s, a).
\end{align*}
Hence, we can bound
\begin{align*}
    | V_{i,x}^{\pi}(s) - V_{i,x'}^{\pi}(s) |
    \leq 
    \max_{s,a} \Big| Q_{i,x}^{\pi}(a, s) - Q_{i,x'}^{\pi}(s, a) \Big|.
\end{align*}
Rewriting over the Bellman equation, we have that for any $x \in \Pi$
\begin{align*}
    Q_{i,x}^\pi(s, a) = r_{i,x}(s,a) + \gamma \sum_{s',a'} P_x(s' | s,a) \pi(a' | s') Q_{i,x}^\pi(s',a').
\end{align*}
Exploiting that, We derive that for any $s \in \mathcal{S}, a \in \mathcal{A}$
\begin{align*}
    &\Big| Q_{i,x}^{\pi}(s, a) - Q_{i,x'}^{\pi}(s, a) \Big| \\
    &\leq
    \left| r_{i,x}(s,a) - r_{i,x'}(s,a) \right| + \gamma \Big| \sum_{s', a'} \left(P_x(s'|s,a) - P_{x'}(s' | s,a)\right) \pi(a' | s') Q_{i,x}^\pi(s', a') \Big| \\&\quad + \gamma \Big| \sum_{s', a'} P_{x'}(s' | s,a) \pi(a' | s') \left( Q_{i,x}^\pi(s', a') - Q_{i,x'}^\pi(s', a') \right)\Big| 
    \\ &\leq
    \left| r_{i,x}(s,a) - r_{i,x'}(s,a) \right|
    +
    \frac{\gamma}{1-\gamma} \lVert P(\cdot | s,a) - P(\cdot |s,a) \rVert_1 + \gamma \max_{s', a'} \left| Q_x^{\pi}(s',a') - Q_{x'}^{\pi}(s',a') \right|.
\end{align*}
Taking the maximum over $s,a$ on the left-hand side implies that 
\begin{align*} 
     \max_{s,a} \Big| Q_{i,x}^{\pi}(s, a) - Q_{i,x'}^{\pi}(s, a) \Big|
     \leq
     \frac{1}{1-\gamma} \max_{s,a} |r_{i,x}(s,a) - r_{i, x'}(s,a) | 
     + \frac{\gamma}{(1-\gamma)^2} \max_{s,a} \lVert P_x(\cdot | s,a) - P_{x'}(\cdot | s,a) \rVert_1. 
\end{align*}
By Assumption~\ref{Assump: Sensitivity}, we obtain that
\begin{align*}
    \max_{s,a} |r_{i,x}(s,a) - r_{i, x'}(s,a) |
    = \lVert r_{i, x}(\cdot, \cdot) - r_{i, x'}(\cdot, \cdot) \rVert_\infty
    \leq 
    \lVert r_{i, x}(\cdot, \cdot) - r_{i, x'}(\cdot, \cdot) \rVert_2
    \leq \perfc_r \cdot \lVert x - x' \rVert_2 ,
\end{align*}
and that
\begin{align*}
    \max_{s,a} \lVert P_x(\cdot | s,a) - P_{x'}(\cdot | s,a) \rVert_1
    &\leq 
    \sqrt{S} \cdot \max_{s,a} \lVert P_x(\cdot | s,a) - P_{x'}(\cdot | s,a) \rVert_2
    \\ &\leq 
    \sqrt{S} \cdot \lVert P_x(\cdot | \cdot, \cdot) - P_{x'}(\cdot | \cdot,\cdot) \rVert_2
    \\ &\leq
    \sqrt{S} \cdot \perfc_p \cdot \lVert x - x' \rVert_2 .
\end{align*}
In total, we obtain, 
\begin{align*}
    \left| V_{i,x}^{\pi}(s) - V_{i,x'}^{\pi}(s) \right|
    &\leq
     \frac{1}{1-\gamma} \lVert r_{i,x} - r_{i, x'} \rVert_\infty + \frac{\gamma}{(1-\gamma)^2} \max_{s,a} \lVert P_x(\cdot | s,a) - P_{x'}(\cdot | s,a) \rVert_1
     \\ &\leq
     \frac{1}{1-\gamma} \perfc_r \cdot \lVert x - x' \rVert_2 + \frac{\gamma}{(1-\gamma)^2} \perfc_p \cdot \sqrt{S} \cdot \lVert x - x' \rVert_2
     \\ &\leq 
     \frac{1}{1-\gamma} \cdot \Big(\perfc_r +\frac{ \gamma \cdot \perfc_p \sqrt{S}}{1-\gamma} \Big) \cdot \lVert x - x' \rVert_2. \qedhere
\end{align*}
\end{proof}

\subsection{Infinite Sample Case PGA -- Proof of Theorem \ref{Theorem: Convergence Result - Infinite Sample Case}}\label{Proof: PGA Convergence Result -- Infinite Sample Case}

We aim to bound the Nash regret.
By the standard analysis of \citet[Theorem $1$]{IndepPolicyGrad_Ding},  
\begin{align*}
    &\sum_{t=1}^T \max_i \left( \max_{\pi_i'} V_{i, t}^{\pi'_i, \pi_{-i}^t}(\rho) - V_{i,t}^{\pi^t}(\rho) \right)
    \\&\overset{(a)}{=}
    \frac{1}{1-\gamma} \sum_{t=1}^T \max_{\pi'_i} \sum_{s,a_i} d_{\rho,t}^{\pi'_i, \pi_{-i}}(s) \left( \pi'_i(a_i|s) - \pi_i^{(t)}(a_i | s) \right) \Bar{Q}_{i,t}^t(s,a_i)
    \\& \overset{(b)}{\leq}
    \frac{3}{\eta(1-\gamma)} \sum_{t=1}^T \sum_s d_{\rho,t}^{\pi'_i, \pi^{t}_{-i}}(s) \cdot \left\lVert \pi_i^{t+1}(\cdot \mid s)) - \pi_i^t(\cdot \mid s) \right\rVert_2 
    \\& \overset{(c)}{\leq}
    \frac{\sqrt{\sup_{t, \pi} \lVert d^\pi_{\rho,t} / \nu \rVert_\infty}}{\eta(1-\gamma)^{3/2}} \sum_{t=1}^T \sum_s \sqrt{d_{\rho,t}^{\pi'_i, \pi^t_{-i}} \cdot d_{\nu,t}^{\pi_i^{t+1}, \pi^t_{-i}}} \cdot \left\lVert \pi_i^{t+1}(\cdot | s) - \pi_i^t(\cdot | s) \right\rVert_2 
    \\& \overset{(d)}{\leq}
    \frac{\sqrt{\sup_{t, \pi} \lVert d^\pi_{\rho,t} / \nu \rVert_\infty}}{\eta(1-\gamma)^{3/2}}\sqrt{\sum_{t=1}^T \sum_s d_{\rho,t}^{\pi_i^{t+1}, \pi_{-i}^t}(s)} \cdot 
    \sqrt{\sum_{t=1}^T \sum_{i=1}^n \sum_s d_{\nu,t}^{\pi_i^{t+1}, \pi_{-i}^{(t)}}(s) \cdot \lVert \pi^{t+1}_i(\cdot | s) - \pi_i^t(\cdot | s)) \rVert_2^2}'
    \\ &\overset{(e)}{\leq}
\frac{\sqrt{\sup_{t, \pi} \lVert d^\pi_{\rho,t} / \nu \rVert_\infty}}{\eta(1-\gamma)^{3/2}}
    \cdot 
    \sqrt{T}
     \cdot \sqrt{\sum_{t=1}^T \sum_{i=1}^n \sum_s d_{v, t}^{\pi^{t+1}_i, \pi_{-i}^t}(s) \cdot \lVert \pi_i^{t+1}(\cdot | s) - \pi_i^t(\cdot | s) \rVert^2} ,
\end{align*}
where we have $\widetilde{\kappa}_\rho = \sup_t \inf_{v \in \Delta(\mathcal{S})} \sup_{\pi \in \Pi} \lVert d^\pi_{\rho,t} / v  \rVert_{\infty}$. Note, that $(a)$ follows by Lemma~\ref{Lemma: Performative Difference Lemma} and by abusing the notation $i$ to represent $\argmax_{\pi'_i}$. In $(b)$, we use the observation by \citet{IndepPolicyGrad_Ding} that by the optimality of $\pi_i^{t+1}$, it holds that
\begin{align*}
    \left\langle \pi'_i(\cdot|s) - \pi_i^{t+1}(\cdot | s), \eta \Bar{Q}_{i,t}^t(s, \cdot) - \pi_i^{t+1}(\cdot | s) + \pi_i^t(\cdot | s) \right\rangle_{\mathcal{A}_i} \leq 0 \quad \text{for any } \pi'_i \in \Pi_i,
\end{align*}
which implies that,
\begin{align*}
    &\left\langle \pi'_i(\cdot | s) - \pi_i^t(\cdot |s), \Bar{Q}_i^t(s, \cdot) \right\rangle_{\mathcal{A}_i}
    \\ &\leq
    \left\langle \pi'_i(\cdot | s) - \pi_i^{t+1}(\cdot | s), \pi_i^{t+1}(\cdot | s) - \pi_i^t(\cdot | s) \right\rangle_{\mathcal{A}_i} + \left\langle \pi_i^{t+1}(\cdot | s) - \pi_i^t(\cdot | s), \Bar{Q}_i^{t}(s,\cdot) \right\rangle_{\mathcal{A}_i}
    \\ &\leq
    \frac{2}{\eta} \left\lVert \pi_i^{t+1}(\cdot | s) - \pi_i^t(\cdot | s) \right\rVert_2 + \left\lVert \pi_i^{t+1}(\cdot | s) - \pi_i^t(\cdot | s) \right\rVert_2 \left\lVert \Bar{Q}_i^t(s, \cdot ) \right\rVert_2
    \\ &\leq
    \frac{3}{\eta} \left\lVert \pi_i^{t+1}(\cdot | s) - \pi_i^t(\cdot | s) \right\rVert_2, 
\end{align*}
where the line uses that $\lvert \Bar{Q}_{i,t}^t(s, \cdot) \rVert_2 \leq \frac{\sqrt{A}}{1-\gamma}$ and $\eta \leq \frac{1-\gamma}{\sqrt{A}}$.
In $(c)$, we choose an arbitrary $\nu \in \Delta(\mathcal{S})$ and exploit
\begin{align*}
    \frac{d_{\rho,t}^{\pi_i', \pi_{-i}^t}(s)}{d^{\pi^{t+1}_i, \pi^t_{-i}}_{\nu,t}(s)}
    \leq
    \frac{d_{\rho,t}^{\pi_i', \pi_{-i}^t}(s)}{(1-\gamma)\nu(s)}
    \leq
    \frac{\sup_{\pi,t} \lVert d_\rho^\pi / \nu \rVert_\infty}{(1-\gamma)}.
\end{align*}
In $(d)$, we exploit the Cauchy-Schwarz inequality, and we sum over all agents to replace the $i$ from the $\argmax_i$ in $(a)$, and in $(e)$ we proceed with using that the second term is equal to $\sqrt{T}$. 

We apply the first guarantee in Lemma~\ref{Lemma: Policy improvement appendix} to obtain:
\begin{align*}
    &\sum_{t=1}^T \max_i \left( \max_{\pi_i', \pi_{-i}^t} V_{i, t}^{\pi_i', \pi_{-i}^t}(\rho) - V_{i,t}^{\pi^t}(\rho) \right)
    \\ &\leq
\frac{\sqrt{\widetilde{\kappa}_{\rho}}} {\eta(1-\gamma)^{3/2}} \cdot \sqrt{T} \cdot \sqrt{2\eta(1-\gamma) \big( \Phi_{T+1}^{T+1}(v) - \Phi^1_{1}(v) \big) + \frac{8 \eta^3 A^2 n^2}{(1-\gamma)^4}T + \mathcal{W}_{r,p} \cdot 2\eta (1-\gamma)}
\\ &\leq
\sqrt{\frac{\widetilde{\kappa}_{\rho} \cdot T \cdot C_\Phi}{\eta(1-\gamma)^{2}}} 
+
\frac{\sqrt{\widetilde{\kappa}_{\rho}}} {\eta(1-\gamma)^{3/2}} \cdot \sqrt{T} \cdot \sqrt{\frac{4\eta^3  A^2 n^2}{(1-\gamma)^4}T }
    +  \frac{\sqrt{\widetilde{\kappa}_{\rho}}} {\eta(1-\gamma)^{3/2}} \cdot \sqrt{T}
    \cdot \sqrt{2\eta(1-\gamma) \cdot \mathcal{W}_{r,p} }
    \\ &\lesssim
    \sqrt{\frac{\widetilde{\kappa}_\rho T C_\phi}{\eta(1-\gamma)^2}} + \sqrt{\frac{\widetilde{\kappa}_\rho \eta T^2 A^2 n^2}{(1-\gamma)^7}} + \frac{\sqrt{\widetilde{\kappa}_\rho \cdot T \cdot \mathcal{W}_{r,p}}}{\sqrt{\eta}(1-\gamma)^3} 
\end{align*}
where we use $\mathcal{W}_{r,p} = \frac{n(n+1)}{2} \cdot \delta_{r,p} \cdot \lVert \pi^{t+1} - \pi^t \rVert_2$ and that $C_\Phi = \max_{t, \pi, \pi', \mu} |\Phi_t^\pi(\mu) - \Phi_t^{\pi'}(\mu)|$. 
We complete the first guarantee by taking step size $\eta = \frac{(1-\gamma)^{5/2} \sqrt{C_\Phi}}{n A \sqrt{T}}$. 

For completing the guarantee, we set $\eta \leq \frac{(1-\gamma)^4}{8 \min{ \{\kappa_\nu, S \}^3} n A}$ and we apply Lemma~\ref{Lemma: Policy improvement appendix}, for a large enough constant $c$, we obtain 
\begin{align*}
    &\sum_{t=1}^T \max_i \left( \max_{\pi_i', \pi_{-i}^t} V_{i, t}^{\pi_i', \pi_{-i}^t}(\rho) - V_{i,t}^{\pi^t}(\rho) \right)
    \\ &\leq
\frac{\sqrt{\sup_{t, \pi} \lVert d^\pi_{\rho,t} / \nu \rVert_\infty}}{\eta(1-\gamma)^{3/2}} \cdot \sqrt{T} \cdot \sqrt{2\eta(1-\gamma) \big( \Phi_{T}^{T+1}(v) - \Phi^1_{T}(v) \big) + \frac{8 \eta^3 A^2 n^2}{(1-\gamma)^4} + \mathcal{W}_{r,p} \cdot 2\eta (1-\gamma)}
    \\ &\leq
\sqrt{\frac{\sup_{t, \pi} \lVert d^\pi_{\rho,t} / \nu \rVert_\infty T C_\phi}{\eta(1-\gamma)^2}} + \sqrt{\frac{( \sup_{t, \pi} \lVert d^\pi_{\rho,t} / \nu \rVert_\infty) T C_\Phi}{\eta(1-\gamma)^2}}\sqrt{\frac{8 \eta  (\sup_{t, \pi} \lVert d^\pi_{\rho,t} / \nu \rVert_\infty) A^2 n^2}{(1-\gamma)^4}} + \frac{\sqrt{\widetilde{\kappa}_\rho \cdot T \cdot \mathcal{W}_{r,p}}}{\sqrt{\eta}(1-\gamma)} 
\\ &\leq
c \cdot \sqrt{\frac{ T C_\Phi \cdot \min \left\{ \kappa_\nu, S \right\}^4 n A}{(1-\gamma)^6}} + \frac{\sqrt{8 \min\{\kappa_\rho, S\}^3 n A }}{(1-\gamma)^5} \cdot \sqrt{\mathcal{W}_{r,p}{T}} ,
\end{align*}
where we use in the last inequality the following special cases: if $\nu = \rho$, then we have that $\sup_{t, \pi} \lVert d^\pi_{\rho,t} / \nu \rVert_\infty \leq \kappa_\nu$ and the first two terms are bounded by $O\left( \sqrt{\frac{\kappa_\rho^4 n A T C_\Phi}{(1-\gamma)^6}} \right)$ for the choice $\eta \leq \frac{(1-\gamma)^4}{8 \kappa_\rho^3 n A}$. Next, if we simply choose $\nu$ as the uniform distribution, we have that $\kappa_\nu \leq S$, $\eta = \frac{(1-\gamma)^4}{8S^3 n A} \leq \frac{(1-\gamma)^4}{8 \kappa_\nu^3 n A}$ is a valid choice, finalizing the guarantee depending on $S$.
Finally, we use the upper bound, given by Lemma~\ref{lemma: Bound noise terms}, to show that
\begin{align*}
    \mathcal{W}_{r, p} 
    &=
    \frac{n(n+1)}{1-\gamma} \cdot \Big(\perfc_r +\frac{ \gamma \cdot \perfc_p \sqrt{S}}{1-\gamma} \Big) 
    \sum_{t} \lVert \pi^{t+1} - \pi^t \Vert_2
    \\ &\le
    \frac{n(n+1)}{1-\gamma} \cdot \Big(\perfc_r +\frac{ \gamma \cdot \perfc_p \sqrt{S}}{1-\gamma} \Big) 
    \sum_{t} \lVert \pi^{t+1} - \pi^t \Vert_1
    \\ &\le
    \frac{n(n+1)}{1-\gamma} \cdot \Big(\perfc_r +\frac{ \gamma \cdot \perfc_p \sqrt{S}}{1-\gamma} \Big) 
    \left\lVert \sum_t \pi^{t+1} - \pi^t \right\Vert_1
    \\ &\le
    \frac{n(n+1)}{1-\gamma} \cdot \Big(\perfc_r +\frac{ \gamma \cdot \perfc_p \sqrt{S}}{1-\gamma} \Big) \cdot nS .
\end{align*}


\begin{lemma}[{\citet[Lemma~$1$]{IndepPolicyGrad_Ding}}]\label{Lemma: Performative Difference Lemma}
    Given an underlying $\textit{MPG}(\pi')$ and a policy $\pi = (\pi_i, \pi_{-i})$, for agent $i$ and any state distribution $\mu$, it holds for any two policies $\hat{\pi}_i$ and $\Bar{\pi}_i$
    \begin{align*}
        V_{i, \pi'}^{\hat{\pi}_i, \pi_{-i}}(\mu) - V_{i, \pi'}^{\Bar{\pi}_i, \pi_{-i}}(\mu) 
        =
        \frac{1}{1-\gamma} \sum_{s, a_i} d_{\mu, \pi'}^{\hat{\pi}_i, \pi_{-i}}(s) \cdot \big( \hat{\pi}_i - \Bar{\pi}_i \big)(a_i|s) \cdot \Bar{Q}_{i, \pi'}^{\Bar{\pi}_i, \pi_{-i}}(s, a_i) ,
    \end{align*}
    where we have $\Bar{Q}_{i, \pi'}^{\Bar{\pi}_i, \pi_{-i}}(s, a_i) = \sum_{a_{-i}} \pi_{-i}(a_{-i} | s) \cdot Q_{i, \pi'}^{\hat{\pi}, \pi_{-i}}(s, a_i)$.
\end{lemma}

\begin{proof}
    Follows by e.g., {C.1} by \citet{GlobalConvergence_Leonardos} for a fixed underlying MPG($\pi'$).
\end{proof}

We denote by $i \sim j$ the set of indices $\{ \ell \; | \; i < \ell < j \}$, following the notation by \citet{IndepPolicyGrad_Ding}.

\begin{lemma}[{\citet[Lemma~$2$]{IndepPolicyGrad_Ding}}]\label{Lemma: Mulivariate function difference}
For any function $\Psi^\pi: \Pi \rightarrow \RR$ and any two policies $\pi, \pi' \in \Pi$, 
\begin{align*}
    \Psi^{\pi'} - \Psi^\pi
    &=
    \sum_{i=1}^n \left( \Psi^{\pi'_i, \pi_{-i}} - \Psi^\pi \right)
    \\ &\quad + 
    \sum_{i=1}^n \sum_{j = i+1}^n 
    \Big( \Psi^{\pi_{<i, i \sim j}, \pi'_{>j}, \pi'_i, \pi'_j} - \Psi^{\pi_{<i, i \sim j}, \pi'_{>j}, \pi_i, \pi'_j}
    - \Psi^{\pi_{<i, i \sim j}, \pi'_{>j}, \pi'_i, \pi_j} 
    +
    \Psi^{\pi_{<i, i \sim j}, \pi'_{>j}, \pi_i, \pi_j}
    \Big) .
\end{align*}
\end{lemma}
    
\begin{lemma}[Policy Improvement]\label{Lemma: Policy improvement appendix}
    For an MPG according to Definition \ref{Def: Markov Potential Game}, for any state distribution $\mu$ and two consecutive policies $\pi^{t+1}$ and $\pi^{t}$ generated by the PGA Algorithm~\eqref{eq:PGA ALGO}, we have:
    \begin{align*}
        \Phi^{t+1}_{t+1}(\mu) - \Phi^t_{t}(\mu)
        &\geq
        \frac{1}{2\eta(1-\gamma)} \sum_{i=1}^n \sum_s d_{\mu, t}^{\pi^{t+1}_i, \pi_{-i}}(s) \cdot \left( 1 - \frac{4 \eta \kappa^3_\mu A n}{(1 - \gamma)^4} \right) \cdot \left\lVert \pi_i^{t+1}(\cdot | s) - \pi^t_i(\cdot | s) \right\rVert^2
         \\ &\quad- \frac{n(n+1)}{2} \cdot \delta_{r,p} \cdot \lVert \pi^{t+1} - \pi^t \rVert_2,
    \end{align*}
where we have $\kappa_\mu = \sup_{t}\sup_{\pi \in \Pi} \lVert d^\pi_{\mu,t} / \mu \rVert_\infty$ and $\delta_{r, p} \coloneqq \frac{1}{1-\gamma}  \left(\perfc_r +\frac{ \gamma \cdot \perfc_p \sqrt{S}}{1-\gamma} \right)$.
\end{lemma}

\begin{proof}
By Lemma~\ref{Lemma: Mulivariate function difference} with $\Psi^\pi = \Phi^\pi_{t}(\mu)$ and abbreviating $\pi' = \pi^{(t+1)}$, $\pi = \pi^{(t)}$, we have that
\begin{align*}
    \Phi^{t+1}_{t}(\mu) - \Phi^t_{t}(\mu) = \textbf{Diff}_\alpha + \textbf{Diff}_\beta,
\end{align*}
where 
\begin{align*}
    \textbf{Diff}_\alpha &\coloneqq \sum_{i=1}^n \Phi_{t}^{\pi'_i, \pi_{-i}}(\mu) - \Phi_{t}^\pi(\mu), \\
\textbf{Diff}_\beta &\coloneqq
\sum_{i=1}^n \sum_{j = i+1}^n 
    \Big( \Phi_{t}^{\pi_{<i, i \sim j}, \pi'_{>j}, \pi'_i, \pi'_j}(\mu) - \Phi_{t}^{\pi_{<i, i \sim j}, \pi'_{>j}, \pi_i, \pi'_j}(\mu)
    -\ \Phi_{t}^{\pi_{<i, i \sim j}, \pi'_{>j}, \pi'_i, \pi_j}(\mu) 
    +
    \Phi_{t}^{\pi_{<i, i \sim j}, \pi'_{>j}, \pi_i, \pi_j}(\mu)
    \Big).
\end{align*}
According to the analysis \citet[Lemma $3.(ii)$]{IndepPolicyGrad_Ding}, this implies that,
\begin{equation}\label{eq: Ding analysis}
    \Phi^{t+1}_{t}(\mu) - \Phi^t_{t}(\mu) 
    \geq
    \frac{1}{2\eta(1-\gamma)} \sum_{i=1}^n \sum_s d_{\mu, t}^{\pi^{t+1}_i, \pi_{-i}}(s) \cdot \left( 1 - \frac{4 \eta \kappa^3_\mu A n}{(1 - \gamma)^4} \right) \cdot \left\lVert \pi_i^{t+1}(\cdot | s) - \pi^t_i(\cdot | s) \right\rVert^2
\end{equation}
Further, we derive that,
\begin{align*}
    \Phi_{t+1}^{t+1}(\mu) - \Phi_{t}^t(\mu)
    = \Phi^{t+1}_{t+1} - \Phi^{t+1}_{t}  + \Phi_{t}^{t+1}(\mu) - \Phi^t_t(\mu)
    \geq 
    - | \Phi^{t+1}_{t+1} - \Phi^{t+1}_{t} | + \Phi_{t}^{t+1}(\mu) - \Phi^t_t(\mu).
\end{align*}
By applying \cref{lemma: Bound noise terms} to upper bound the first term and by applying equation \ref{eq: Ding analysis}, we obtain our result.
\end{proof}

\subsection{IPGA -- Without Gradient Oracle}\label{Proof: PGA Convergence Result -- Finite Sample Case}

Similar as in section~\ref{Proof: PGA Convergence Result -- Infinite Sample Case}, we extend the techniques by \citet{IndepPolicyGrad_Ding} to obtain finite sample guarantees for our version of MPGs under performativity.

We bound the maximum occurring deviations over all time steps.
The proof is based on \citet[Theorem 3]{IndepPolicyGrad_Ding} incorporating the additional costs due to performative effects.

\begin{proof}
We bound the performative regret that PGA Algorithm~\eqref{eq:PGA ALGO} achieves in a MPG with performative effects.
Recall that $\overline{Q}_{i,t}$ corresponds to the exact $Q$-value averaged over the policies of agents $-i$ under the performative effect induced by $\pi^t$, while $\widehat{Q}_{i,t}$ corresponds to the computed estimation, and exploration rate $\xi \leq \frac{1}{2}$.
In total, we have
\begin{equation}\label{eq: Main equation Theorem}
\begin{split}
&\sum_{t=1}^T \max_i \left( \max_{\pi_i'} V_{i, t}^{\pi'_i, \pi_{-i}^t}(\rho) - V_{i,t}^{\pi^t}(\rho) \right)    
\\&\overset{(a)}{=}
    \frac{1}{1-\gamma} \sum_{t=1}^T \max_{\pi'_i} \sum_{s,a_i} d_{\rho,t}^{\pi'_i, \pi_{-i}}(s) \left( \pi'_i(a_i|s) - \pi_i^{(t)}(\cdot | s) \right) \Bar{Q}_{i,t}^t(s,a_i)
\\&\overset{(b)}{\leq} 
\frac{1}{\eta(1-\gamma)} \sum_{t=1}^T \sum_s d_{\rho,t}^{\pi'_i, \pi^{t}_{-i}}(s) \left\lVert \pi_i^{t+1}(\cdot | s) - \pi_i^t(\cdot | s) \right\rVert_2 + \frac{\eta T \xi \sqrt{A}}{(1-\gamma)^2} 
\\&\quad+
\frac{1}{1-\gamma}\sum_{t=1}^T \sum_s d_{\rho,t}^{\pi'_i, \pi_{-i}^t}(s) \left\langle \pi'_i(\cdot | s) - \pi_{i}^t(\cdot | s), \Bar{Q}_{i,t}^t(s, \cdot) - \widehat{Q}_{i,t}^t(s, \cdot) \right\rangle_{\mathcal{A}_i}
\\&\overset{(c)}{\leq}
\frac{\sqrt{\kappa_\rho}}{\eta(1-\gamma)^{3/2}} \sum_{t=1}^T \sum_s \sqrt{d_{\rho,t}^{\pi_i^{t+1}, \pi_{-i}^t}(s) \cdot d_{\rho, t}^{\pi'_i, \pi^t_{-i}}(s)} \left\lVert \pi_i^{t+1}(\cdot | s) - \pi_i^t(\cdot | s) \right\rVert_2 + \frac{\xi T \sqrt{A}}{(1-\gamma)^2} 
\\ &\quad+
\frac{\kappa_\rho}{1-\gamma} \left| \sum_{t=1}^T \sum_s d_{\rho,t}^{\pi^t}(s) \cdot \left\langle \pi'_i(\cdot | s) - \pi_i^t(\cdot | s), \Bar{Q}_{i,t}^t(s, \cdot) - \widehat{Q}_{i,t}^t(s, \cdot) \right\rangle_{\mathcal{A}_i} \right| 
\\
&\overset{(d)}{\leq}
\frac{\sqrt{\kappa_\rho}}{\eta(1-\gamma)^{3/2}} \sqrt{\sum_{t=1}^T \sum_s d^{\pi'_i, \pi_{-i}^t}_{\rho,t}(s)} \cdot \sqrt{\sum_{t=1}^T \sum_s d_{\rho,t}^{\pi_i^{t+1}, \pi_{-i}^t}(s) \left\lVert \pi_i^{t+1}(\cdot | s) - \pi_i^t(\cdot | s) \right\rVert_2^2}
\\&\quad+
\frac{\xi T \sqrt{A}}{(1-\gamma)^2} + \frac{\kappa_\rho}{1-\gamma} \sum_{t=1}^T \sqrt{\frac{A L^t_{i}(\widehat{Q}_{i,t}^t)}{\xi}}
\\ 
&\overset{(e)}{\leq}
\frac{\sqrt{\kappa_\rho}}{\eta(1-\gamma)^{3/2}} \sqrt{\sum_{t=1}^T \sum_s d^{\pi'_i, \pi_{-i}^t}_{\rho,t}(s)} \cdot \sqrt{\sum_{t=1}^T \sum_{i=1}^n \sum_s d_{\rho,t}^{\pi_i^{t+1}, \pi_{-i}^t}(s) \left\lVert \pi_i^{t+1}(\cdot | s) - \pi_i^t(\cdot | s) \right\rVert_2^2}
\\&\quad+
\frac{\xi T \sqrt{A}}{(1-\gamma)^2} + \frac{\kappa_\rho}{1-\gamma} \sum_{t=1}^T \sqrt{\frac{A L^t_{i}(\widehat{Q}_{i,t}^t)}{\xi}},
\end{split}
\end{equation}
where we use the multi-agent performance difference Lemma~\ref{Lemma: Performative Difference Lemma} for $(a)$, inequality
$(b)$ follows by abusing notation following \citet{IndepPolicyGrad_Ding}: policy $\pi_i'$ represents the $\argmax_{\pi_i'}$ and $i$ captures the $\argmax_i$ and by using the following property of the algorithm, \citep[Equation~24]{IndepPolicyGrad_Ding} for $\xi \leq \frac{1}{2}$ and $\eta \leq \frac{1-\gamma}{\sqrt{A}}$:
\begin{align*}
    \left\langle \pi'_i(\cdot | s) - \pi_i^{t}(\cdot | s), \eta \widehat{Q}_i^t(s, \cdot) \right\rangle_{\mathcal{A}_i}
    \lesssim
    \frac{1}{\eta} \left\lVert \pi_i^{t+1}(\cdot | s) - \pi_i^t(\cdot | s) \right\rVert_2 + \frac{\xi \sqrt{A}}{1-\gamma} 
    + \left\langle \pi'_i(\cdot | s) - \pi^t(\cdot | s), \overline{Q}_i^t(s, \cdot) - \widehat{Q}_i^t(s, \cdot) \right\rangle_{\mathcal{A}_i} .
\end{align*}
$(c)$ follows by the definition of the distribution mismatch coefficient:
\begin{align*}
    \frac{d_{\rho, t}^{\pi'_i, \pi_{-i}^t}(s)}{d_{\rho,t}^{\pi_i^{t+1}, \pi^t_{-i}}(s)} \leq \frac{d_{\rho, t}^{\pi'_i, \pi_{-i}^t}(s)}{(1-\gamma) \rho(s)} \leq \frac{\kappa_\rho}{1-\gamma}.
\end{align*}
$(d)$ follows by applying the Cauchy-Schwartz inequality, Jensen inequality, and that 
\begin{align*}
\left| \sum_{t=1}^T \sum_s d_{\rho,t}^{\pi^t}(s) \left\langle \pi'_i(\cdot | s) - \pi_i^t(\cdot | s), \Bar{Q}_{i,t}^t(s, \cdot) - \widehat{Q}_{i,t}^t(s, \cdot) \right\rangle_{\mathcal{A}_i} \right| 
\leq
\sqrt{\frac{A L^t_{i}(\widehat{Q}_{i,t}^t) }{\xi}}.
\end{align*}
In $(e)$, we simply replace the $\argmax_i$ by the sum over all players.
The remaining part follows by applying the modified policy improvement Lemma~\ref{Lemma: Finite Sample: Policy improvement} that takes into account performative effects. 

We define $\mathcal{W}_{r,p} = \frac{n(n+1)}{2} \cdot \delta_{r,p} \cdot \lVert \pi^{t+1} - \pi^t \rVert_2$ and that $C_\Phi = \max_{t, \pi, \pi', \mu} |\Phi_t^\pi(\mu) - \Phi_t^{\pi'}(\mu)|$, which describes the performative costs. 
To obtain the guarantee, we apply Lemma~\ref{Lemma: Finite Sample: Policy improvement} to Eq.~\eqref{eq: Main equation Theorem}, which leads to
\begin{align*}
    &\E{
\left[ \sum_{t=1}^T \max_i \left( \max_{\pi_i'} V_{i, t}^{\pi'_i, \pi_{-i}^t}(\rho) - V_{i,t}^{\pi^t}(\rho) \right) \right]}
\\ &\lesssim
\frac{\sqrt{\kappa_\rho}}{\eta(1-\gamma)^{3/2}} \sqrt{T} \sqrt{\eta(1-\gamma)(\Phi^{T+1}_{T+1}(\mu) - \Phi_{1}^1(\mu)) + \frac{\eta^2 \kappa_\rho A}{(1-\gamma)\xi} \sum_{t=1}^T \sum_{i=1}^n \E{\left[ L_i^t(\widehat{Q}^t_{i,t})\right]} + \mathcal{W}_{r, p} \cdot \eta \cdot (1-\gamma)} 
\\ &\quad + \frac{\xi T \sqrt{A}}{(1-\gamma)^2} + \frac{\kappa_\rho}{1-\gamma} \sum_{t=1}^T \sqrt{\frac{A \E{ \left[ L^t_{i}(\widehat{Q}_{i,t}^t)\right]}}{\xi}}
\\
&\lesssim
\sqrt{\frac{\kappa_\rho T C_\Phi}{\eta(1-\gamma)^2}} + \frac{\kappa_\rho T}{(1-\gamma)^2} \sqrt{\frac{A \cdot n \cdot \delta_{stat}}{\xi}} + \frac{\xi T \sqrt{A}}{(1-\gamma)^2} + \frac{\sqrt{\kappa_\rho \cdot T \cdot \mathcal{W}_{r,p}}}{\sqrt{\eta}(1-\gamma)},
\end{align*}
where in the last line, we use that $C_\phi \geq  \Phi^{T+1}_{T+1} - \Phi^1_1$, and that $\mathbb{E}[L_i^t(\widehat{Q}_{i,t}^t)] \leq \delta_{\mathrm{stat}}$.
Finally, the guarantee follows by the choice $\eta = \frac{(1-\gamma)^4}{16 \kappa^3_\rho n A}$, and $\xi \leq
\left( \kappa_\rho^2 \cdot n \cdot \delta_{stat} \right)^{\frac{1}{3}}$.
\end{proof}

We obtain a similar policy improvement lemma as in Lemma~\ref{Lemma: Policy improvement appendix}, incurring additional costs for the estimation error in the expected regression loss $L^{t}_i(\widehat{Q}_{i,t}^t)$. 
\begin{lemma}[Policy improvement]\label{Lemma: Finite Sample: Policy improvement}
    For an MPG with respect to a set $D^*$ according to Definition~\ref{Def: Markov Potential Game}, for any state distribution $\mu$, the potential function $\Phi^\pi(\mu)$ and two consecutive policies $\pi^{t+1}$ and $\pi^{t}$ generated by the PGA Algorithm~\ref{eq:PGA ALGO}, we have:
    \begin{align*}
    \Phi^{t+1}_{t+1}(\rho) - \Phi_{t}^t(\rho)
    &\geq
    \frac{1}{4\eta(1-\gamma)}\sum_{i=1}^n \sum_s d_{\rho, t}^{\pi_i^{t+1}, \pi_{-i}^t}(s) \left(1 - \frac{4\eta\kappa_\rho^3 n A}{(1-\gamma)^4}\right) \left( \left\lVert  \pi^{t+1}(\cdot | s) - \pi^t(\cdot | s) \right\rVert_2^2 \right) 
    \\ &\quad - \frac{\eta \kappa_\rho A}{(1-\gamma)^2 \xi} \sum_{i=1}^n L^t_{i}(\widehat{Q}_{i,t}^t) 
    - \frac{n(n+1)}{2} \cdot \delta_{r,p} \cdot \lVert \pi^{t+1} - \pi^t \rVert_2,
    \end{align*}
where we have $\kappa_\mu = \sup_{t}\sup_{\pi \in \Pi} \lVert d^\pi_{\mu,t} / \mu \rVert_\infty$ and $\delta_{r, p} \coloneqq \frac{1}{1-\gamma}  \left(\perfc_r +\frac{ \gamma \cdot \perfc_p \sqrt{S}}{1-\gamma} \right)$.
\end{lemma}

\begin{proof}
By Lemma~\ref{Lemma: Mulivariate function difference} with $\Psi^\pi = \Phi^\pi_{t}(\mu)$ and abbreviating $\pi' = \pi^{(t+1)}$, $\pi = \pi^{(t)}$, we have that
\begin{align*}
    \Phi^{t+1}_{t}(\mu) - \Phi^t_{t}(\mu) = \textbf{Diff}_\alpha + \textbf{Diff}_\beta,
\end{align*}
where 
\begin{align*}
    \textbf{Diff}_\alpha &\coloneqq \sum_{i=1}^n \Phi_{t}^{\pi'_i, \pi_{-i}}(\mu) - \Phi_{t}^\pi(\mu), \\
\textbf{Diff}_\beta &\coloneqq
\sum_{i=1}^n \sum_{j = i+1}^n 
    \Big( \Phi_{t}^{\pi_{<i, i \sim j}, \pi'_{>j}, \pi'_i, \pi'_j}(\mu) - \Phi_{t}^{\pi_{<i, i \sim j}, \pi'_{>j}, \pi_i, \pi'_j}(\mu)
    -\ \Phi_{t}^{\pi_{<i, i \sim j}, \pi'_{>j}, \pi'_i, \pi_j}(\mu) 
    +
    \Phi_{t}^{\pi_{<i, i \sim j}, \pi'_{>j}, \pi_i, \pi_j}(\mu)
    \Big).
\end{align*}
According to the analysis \citet[Lemma $6.(ii)$]{IndepPolicyGrad_Ding}, this implies that,
\begin{align}\label{eq: Ding analysis without grad}
    \Phi^{t+1}_{t}(\mu) - \Phi^t_{t}(\mu) 
    &\geq
    \frac{1}{4\eta(1-\gamma)}\sum_{i=1}^n \sum_s d_{\rho, t}^{\pi_i^{t+1}, \pi_{-i}^t}(s) \left(1 - \frac{4\eta\kappa_\rho^3 n A}{(1-\gamma)^4}\right) \left( \left\lVert  \pi^{t+1}(\cdot | s) - \pi^t(\cdot | s) \right\rVert_2^2 \right) \nonumber
    \\&- \frac{\eta \kappa_\rho A}{(1-\gamma)^2 \xi} \sum_{i=1}^n L^t_{i}(\widehat{Q}_{i,t}^t).
\end{align}
Further, we derive that,
\begin{align*}
    \Phi_{t+1}^{t+1}(\mu) - \Phi_{t}^t(\mu)
    = \Phi^{t+1}_{t+1} - \Phi^{t+1}_{t}  + \Phi_{t}^{t+1}(\mu) - \Phi^t_t(\mu)
    \geq 
    - | \Phi^{t+1}_{t+1} - \Phi^{t+1}_{t} | + \Phi_{t}^{t+1}(\mu) - \Phi^t_t(\mu).
\end{align*}
By applying \cref{lemma: Bound noise terms} to upper bound the first term and by applying equation \ref{eq: Ding analysis without grad}, we obtain our result.
\end{proof}

\subsection{Independent Natural Policy Gradient 
Ascent -- Analysis}\label{Appendix: Best-iterate Convergence proofs for Natural Policy Gradient Ascent}

In general, this analysis of previous results draw on a performative difference lemma that can be adopted to the PRL setting by incorporating additional costs due to changes of the environment. Recall that we assume a total potential function,

We focus on the result by \citet{DBLP:conf/nips/ZhangMDS022} with $\log$-barrier regularization, because this allows us to have a technical assumption on stationary points under shifting environments.
The natural generalization of the guarantees gives further justification for the robustness of NPG-type algorithms under performative effects in our experiments. 

\subsubsection{Unregularized Version}\label{app:inpg-unreg}
For a fixed MPG at time $t$, we have the following result by \citet[Proposition IX.2]{NPG_3}.
\begin{lemma}[Policy Improvement] \label{lemma:policy_improvement_fixed_MPG}
    For MPG$(\pi^t)$ and any initial distribution $\rho$, the following holds
    \begin{align*}
        \Phi_{t}^{t+1}(\rho) - \Phi_t^{t}(\rho)
        \geq
        &\left(\frac{1}{\eta} - \frac{\sqrt{n}}{(1 - \gamma)^2}\right) \sum_s \mu_{\rho, t}^{\pi^{t+1}}(s) \mathrm{KL}\left( \pi^{t+1}(\cdot | s) || \pi^{t}(\cdot | s) \right) 
        + \frac{1}{\eta} \sum_s \mu_{\rho, t}^{\pi^{t+1}}(s) \sum_i \log Z_i^{t}(s).
    \end{align*}
\end{lemma}
This result allows us to obtain the policy improvement lemma below under performative effects.

\begin{lemma}[Policy Improvement under Performative Effects] \label{NPG Policy Improv}
    For all $i \in \mathcal{N}$, it holds that
    \begin{align*}
        V_{i, t+1}^{t+1}(\rho) - V_{i, t}^{t}(\rho)
        \ge &\left( \frac{1}{M} \left( \frac{1}{\eta} - \frac{\sqrt{n}}{(1- \gamma)^2} \right) - \frac{\sqrt{2}}{1 - \gamma}\left( \perfc_r + \frac{\gamma \perfc_p \sqrt{S}}{1 - \gamma} \right)\right) \mathrm{KL}\left( \pi^{t+1} || \pi^{t} \right) \\
        &+ \frac{1}{\eta} \sum_s \mu_{\rho, t}^{\pi^{t+1}}(s) \sum_i \log Z_i^{t}(s) .
    \end{align*}
    Moreover, for $\eta \le (1 - \gamma) \left(\frac{\sqrt{n}}{1 - \gamma} + \sqrt{2} M \left( \perfc_r + \frac{\gamma \perfc_p \sqrt{S}}{1 - \gamma} \right) \right)^{-1}$,
    \begin{align*}
        V_{i, t+1}^{t+1}(\rho) - V_{i, t}^{t}(\rho)
        \geq 
        \frac{1}{\eta} \sum_s \mu_{\rho, t}^{\pi^{t+1}}\sum_i \log Z_i^{t}(s) .
    \end{align*}
\end{lemma}

\begin{proof} 
    By the definition of the potential function, for all $i \in \mathcal{N}$:
    \begin{align*}
        V_{i, t+1}^{t+1}(\rho) - V_{i, t}^{t}(\rho)
        &= V_{i, t}^{t+1}(\rho) - V_{i, t}^{t}(\rho) - \left( V_{i, t}^{t+1}(\rho) - V_{i, t + 1}^{t + 1}(\rho) \right) \\
        &= \Phi_{t}^{t+1}(\rho) - \Phi_t^{t}(\rho) - \left( V_{i, t}^{t+1}(\rho) - V_{i, t + 1}^{t + 1}(\rho) \right) \\
        &\ge \Phi_{t}^{t+1}(\rho) - \Phi_t^{t}(\rho) - \left\vert V_{i, t+1}^{t+1}(\rho) - V_{i, t}^{t + 1}(\rho) \right\rvert \\
        &\overset{(a)}{\ge} \left( \frac{1}{\eta} - \frac{\sqrt{n}}{(1 - \gamma)^2} \right) \sum_s \mu_{\rho, t}^{\pi^{t+1}}(s) \mathrm{KL}\left( \pi^{t+1}(\cdot | s) || \pi^{t}(\cdot | s) \right) + \frac{1}{\eta} \sum_s \mu_{\rho, t}^{\pi^{t+1}}(s) \sum_i \log Z_i^{t}(s) \\
        &\quad - \frac{1}{1 - \gamma}\left( \perfc_r + \frac{\gamma \perfc_p \sqrt{S}}{1 - \gamma} \right) \left\lVert \pi^{t+1} - \pi^t \right\rVert_2 \\
        &\ge \left( \frac{1}{\eta} - \frac{\sqrt{n}}{(1 - \gamma)^2} \right) \sum_s \mu_{\rho, t}^{\pi^{t+1}}(s) \mathrm{KL}\left( \pi^{t+1}(\cdot | s) || \pi^{t}(\cdot | s) \right) + \frac{1}{\eta} \sum_s \mu_{\rho, t}^{\pi^{t+1}}(s) \sum_i \log Z_i^{t}(s) \\
        &\quad - \frac{1}{1 - \gamma}\left( \perfc_r + \frac{\gamma \perfc_p \sqrt{S}}{1 - \gamma} \right) \left\lVert \pi^{t+1} - \pi^t \right\rVert_1 \\
        &\overset{(b)}{\ge} \left( \frac{1}{\eta} - \frac{\sqrt{n}}{(1- \gamma)^2} \right) \sum_s \mu_{\rho, t}^{\pi^{t+1}}(s) \mathrm{KL}\left( \pi^{t+1}(\cdot | s) || \pi^{t}(\cdot | s) \right) + \frac{1}{\eta} \sum_s \mu_{\rho, t}^{\pi^{t+1}}(s) \sum_i \log Z_i^{t}(s) \\
        &\quad - \frac{\sqrt{2}}{1 - \gamma}\left( \perfc_r + \frac{\gamma \perfc_p \sqrt{S}}{1 - \gamma} \right) \mathrm{KL}(\pi^{t + 1} || \pi^t) \\
        &\ge \sum_s \left( \left( \frac{1}{\eta} - \frac{\sqrt{n}}{(1- \gamma)^2} \right) \mu_{\rho, t}^{\pi^{t+1}}(s) - \frac{\sqrt{2}}{1 - \gamma}\left( \perfc_r + \frac{\gamma \perfc_p \sqrt{S}}{1 - \gamma} \right)\right) \mathrm{KL}\left( \pi^{t+1}(\cdot | s) || \pi^{t}(\cdot | s) \right) \\
        &\quad + \frac{1}{\eta} \sum_s \mu_{\rho, t}^{\pi^{t+1}}(s) \sum_i \log Z_i^{t}(s) \\
        &\overset{(c)}{\ge} \sum_s \left( \frac{1}{M} \left( \frac{1}{\eta} - \frac{\sqrt{n}}{(1- \gamma)^2} \right) - \frac{\sqrt{2}}{1 - \gamma}\left( \perfc_r + \frac{\gamma \perfc_p \sqrt{S}}{1 - \gamma} \right)\right) \mathrm{KL}\left( \pi^{t+1}(\cdot | s) || \pi^{t}(\cdot | s) \right) \\
        &\quad + \frac{1}{\eta} \sum_s \mu_{\rho, t}^{\pi^{t+1}}(s) \sum_i \log Z_i^{t}(s) \\
        &\ge \left( \frac{1}{M} \left( \frac{1}{\eta} - \frac{\sqrt{n}}{(1- \gamma)^2} \right) - \frac{\sqrt{2}}{1 - \gamma}\left( \perfc_r + \frac{\gamma \perfc_p \sqrt{S}}{1 - \gamma} \right)\right) \mathrm{KL}\left( \pi^{t+1} || \pi^{t} \right) \\
        &\quad + \frac{1}{\eta} \sum_s \mu_{\rho, t}^{\pi^{t+1}}(s) \sum_i \log Z_i^{t}(s),
    \end{align*}
    where $(a)$ follows from Lemma~\ref{lemma: Bound noise terms} and Lemma~\ref{lemma:policy_improvement_fixed_MPG}, $(b)$ follows from Pinsker's inequality, and $(c)$ follows from Assumption~\ref{assumption:positive_visit}.
    
    
    Thus, for $\eta \le (1 - \gamma) \left(\frac{\sqrt{n}}{1 - \gamma} + \sqrt{2} M \left( \perfc_r + \frac{\gamma \perfc_p \sqrt{S}}{1 - \gamma} \right) \right)^{-1}$,
    \begin{equation*}
        V_{i, t+1}^{t+1}(\rho) - V_{i, t}^{t}(\rho)
        \geq 
        \frac{1}{\eta} \sum_s \mu_{\rho, t}^{\pi^{t+1}}\sum_i \log Z_i^{t}(s) . \qedhere
    \end{equation*}
\end{proof}
Moreover, we borrow the following lemma that bounds the performative gap for the fixed MPG at time $t$ from \citet[Lemma~IX.3]{NPG_3}.
\begin{lemma} \label{lemma:perform_gap_npg}
    For $\eta \le (1 - \gamma)^2$ and any initial distribution $\rho$, it holds that
    \begin{align*}
        \left( \max_i \max_{\pi_i'} V_{i, t}^{\pi_i', \pi_{-i}^t}(\rho) - V_{i, t}^{\pi^t}(\rho) \right)^2 \le \frac{3 \Tilde{\kappa}_\rho}{c \eta^2 (1 - \gamma)} \sum_i \sum_s \mu_{\rho, t}^{\pi^{t+1}}(s) \log Z_i^{t}(s),
    \end{align*}
    where $c \coloneqq \min_i \min_t \min_s \sum_{a_i^* \in \arg\max_{a_i \in A_i} \Bar{Q}_{i, t}^{\pi^t}(s, a_i)} \pi_i^t(a_i^* | s) > 0$.
\end{lemma}
We now have all the ingredients for the convergence proofs.

\begin{proof}[Proof of Theorem~\ref{thm:inpg_conv}]
    By Jensen's inequality, Lemma~\ref{NPG Policy Improv} and Lemma~\ref{lemma:perform_gap_npg}:
    \begin{align*}
        \mathrm{Perform\textit{-}Regret}(T)
        &\le
        \sqrt{\frac{1}{T} \sum_{t = 1}^T \left(\max_i \max_{\pi_i'} V_{i, t}^{\pi_i', \pi_{-i}^t} - V_{i, t}^{\pi^t} \right)^2} \\
        &\le
        \sqrt{\frac{1}{T} \sum_{t = 1}^T \frac{3 \Tilde{\kappa}_\rho}{c \eta^2 (1 - \gamma)} \sum_i \sum_s \mu_{\rho, t}^{\pi^{t+1}}(s) \log Z_i^{t}(s)} \\
        &\le
        \sqrt{\frac{1}{T} \sum_{t = 1}^T \frac{3 \Tilde{\kappa}_\rho}{c \eta (1 - \gamma)} \left(V_{i, t+1}^{t+1}(\rho) - V_{i, t}^{t}(\rho) \right)} \\
        &\le
        \sqrt{\frac{1}{T} \frac{3 \Tilde{\kappa}_\rho \left(\frac{\sqrt{n}}{1 - \gamma} + \sqrt{2} M \left( \perfc_r + \frac{\gamma \perfc_p \sqrt{S}}{1 - \gamma} \right) \right)}{c (1 - \gamma)^3}} ,
    \end{align*}
    where we set $\eta \le (1 - \gamma) \left(\frac{\sqrt{n}}{1 - \gamma} + \sqrt{2} M \left( \perfc_r + \frac{\gamma \perfc_p \sqrt{S}}{1 - \gamma} \right) \right)^{-1}$ for the last inequality.
\end{proof}



\subsubsection{Regularized Version}\label{app:inpg-reg}

We follow the proof of \citet{DBLP:conf/nips/ZhangMDS022} and define the following for notational simplicity.
\begin{align*}
    \Delta_i^t(s, a_i) 
    &\coloneqq
    \frac{\pi_i^{t+1}(a_i | s)}{\pi_i^t(a_i | s)} - 1.
\end{align*}




We borrow the following four lemmas from \citet{DBLP:conf/nips/ZhangMDS022}. Their proofs follow from the same lines as in the original lemmas for the fixed MPG at time step $t$.

\begin{lemma}[Lemma~24 by \cite{DBLP:conf/nips/ZhangMDS022}] \label{lemma:inpg_min_pol}
    For $\eta \le \frac{1}{15\left(\frac{1}{(1 - \gamma)^2} + \lambda A_i M \right)}$ and $\theta_i^0 = 0$, the following is satisfied by the regularized INPG for all $t \ge 1$:
    \begin{equation*}
        \pi_{\theta_i^t}(a_i | s) \ge \frac{\lambda}{4 \left( \lambda A_i M + (1 - \gamma)^{-2} \right)} .
    \end{equation*}
\end{lemma}

\begin{lemma}[Lemma~26 by \cite{DBLP:conf/nips/ZhangMDS022}] \label{lemma:policy_diff_inpg_reg}
    For MPG$(\pi^t)$, any initial distribution $\rho$, $\eta \le \frac{1}{15\left(\frac{1}{(1 - \gamma)^2} + \lambda A_i M \right)}$, and $\theta_i^0 = 0$, the following is satisfied by the regularized INPG dynamics:
    \begin{equation*}
        \Tilde{\Phi}_{t}^{t+1}(\rho) - \Tilde{\Phi}_t^{t}(\rho)
        \ge
        \left( \frac{1}{2 \eta} - 4 \lambda A_{\max} M^2 \frac{4 M}{(1 - \gamma)^2} - \frac{3nM}{(1 - \gamma)^3} \right)
        \sum_i \sum_{s, a_i} \mu_{\rho, t}^{\pi^t}(s) \pi_i^t(a_i | s) \Delta_i^t(s, a_i)^2.
    \end{equation*}
\end{lemma}


\begin{lemma}[Lemma~27 by \cite{DBLP:conf/nips/ZhangMDS022}] \label{lemma:inpg_reg_delta_f}
    For MPG$(\pi^t)$, any initial distribution $\rho$, $\eta \le \frac{1}{15\left(\frac{1}{(1 - \gamma)^2} + \lambda A_i M \right)}$, and $\theta_i^0 = 0$, the following is satisfied by the regularized INPG dynamics:
    \begin{align*}
        \sum_i \sum_{s, a_i} \mu_{\rho, t}^{\pi^t}(s) \pi_i^t(a_i | s) \Delta_i^t(s, a_i)^2
        \ge \frac{\eta^2}{9}\sum_i \sum_{s, a_i} \mu_{\rho, t}^{\pi^t}(s) \pi_i^t(a_i | s) f_i^t(s, a_i)^2 .
    \end{align*}
\end{lemma}

\begin{lemma}[Lemma~28 by \cite{DBLP:conf/nips/ZhangMDS022}] \label{lemma:inpg_reg_policy_diff}
    For MPG$(\pi^t)$, any initial distribution $\rho$, $\eta \le \frac{1}{15\left(\frac{1}{(1 - \gamma)^2} + \lambda A_i M \right)}$, and $\theta_i^0 = 0$, the following inequality is satisfied by the regularized INPG dynamics:
    \begin{align*}
        \max_i \max_{\pi_i'} V_{i, t}^{\pi_i', \pi_{-i}^t} - V_{i, t}^{\pi^t}
        \le 
        \frac{\sum_i \sum_{s, a_i} \mu_{\rho, t}^{\pi^t}(s) \pi_i^t(a_i | s) f_i^t(s, a_i)^2}{4 \lambda} + \lambda M A_{\max} .
    \end{align*}
\end{lemma}
The following lemma provides an upper bound on the KL divergence of the consecutive policies generated by the regularized INPG dynamics in terms of the regularization parameter. By Pinsker's inequality, this also allows us to bound the performative effects.
\begin{lemma} \label{lemma:inpg_reg_bounds}
    For MPG$(\pi^t)$, any initial distribution $\rho$, $\eta \le \frac{1}{15\left(\frac{1}{(1 - \gamma)^2} + \lambda A_i M \right)}$, and $\theta_i^0 = 0$, the following inquality is satisfied by the regularized INPG dynamics:
    \begin{equation*}
        \mathrm{KL}(\pi^{t + 1} || \pi^t)
        \le \eta n S \left( \frac{1}{(1 - \gamma)^2} + 4 \lambda M \left( \lambda A_{\max} M + \frac{1}{(1 - \gamma)^2} \right) \right).
    \end{equation*}
\end{lemma}

\begin{proof}
    By the regularized INPG dynamics and Lemma~\ref{lemma:inpg_min_pol},
    \begin{align*}
        \log \left(\frac{\pi_i^{t + 1}(a_i | s)}{\pi_i^{t}(a_i | s)} \right)
        &= \frac{\eta}{1 - \gamma} \Bar{A}_{i, t}^{\pi^t}(s, a_i) + \frac{\eta \lambda}{\mu_{\rho, t}^{\pi^t}(s) \pi_i^{t} (a_i | s)} - \frac{\eta \lambda A_i}{\mu_{\rho, t}^{\pi^t}} - \log Z_i^{t}(s) \\
        &\le \frac{\eta}{1 - \gamma} \Bar{A}_{i, t}^{\pi^t}(s, a_i) + \frac{\eta \lambda}{\mu_{\rho, t}^{\pi^t}(s) \pi_i^{t} (a_i | s)} \\
        &\le \frac{\eta}{(1 - \gamma)^2} + 4 \eta \lambda M \left( \lambda A_{\max} M + \frac{1}{(1 - \gamma)^2} \right).
    \end{align*}
    Thus,
    \begin{align*}
        \mathrm{KL}(\pi^{t + 1} || \pi^t)
        &= \sum_i \sum_s \sum_{a_i} \pi_i^{t + 1}(a_i | s) \log \left(\frac{\pi_i^{t + 1}(a_i | s)}{\pi_i^{t}(a_i | s)} \right) \\
        &\le \sum_i \sum_s \sum_{a_i} \pi_i^{t + 1}(a_i | s) \left( \frac{\eta}{(1 - \gamma)^2} + 4 \eta \lambda M \left( \lambda A_{\max} M + \frac{1}{(1 - \gamma)^2} \right)\right) \\
        &= \eta n S \left( \frac{1}{(1 - \gamma)^2} + 4 \lambda M \left( \lambda A_{\max} M + \frac{1}{(1 - \gamma)^2} \right) \right). \qedhere
    \end{align*}
\end{proof}

We can now provide the proof of Theorem~\ref{thm:inpg-reg-conv}. We follow the proof by \citet[Theorem 7]{DBLP:conf/nips/ZhangMDS022} while taking the performative effects into account.
\begin{proof}[Proof of Theorem~\ref{thm:inpg-reg-conv}]
    By the definition of $\Tilde{\Phi}$, Lemma~\ref{lemma: Bound noise terms}, Lemma~\ref{lemma:policy_diff_inpg_reg} and Lemma~\ref{lemma:inpg_reg_delta_f},
    \begin{align*}
        \Tilde{V}_{i, t + 1}^{t+1}(\rho) - \Tilde{V}_{i, t}^{t}(\rho)
        &\ge
        \Tilde{\Phi}_{t}^{t+1}(\rho) - \Tilde{\Phi}_t^{t}(\rho) - \left\lvert \Tilde{V}_{i, t+1}^{t+1}(\rho) - \Tilde{V}_{i, t}^{t+1}(\rho) \right\rvert \\
        &\ge \frac{1}{4 \eta} \sum_{i} \sum_{s, a_i} \mu_{\rho, t}^{\pi^t}(s) \pi_i^t(a_i | s) \Delta_i^t(s, a_i)^2 - \frac{1}{1 - \gamma}\left( \perfc_r + \frac{\gamma \perfc_p \sqrt{S}}{1 - \gamma} \right) \left\lVert \pi^{t+1} - \pi^t \right\rVert_2 \\
        &\ge \frac{\eta}{36} \sum_{i} \sum_{s, a_i} \mu_{\rho, t}^{\pi^t}(s) \pi_i^t(a_i | s) f_i^t(s, a_i)^2 - \frac{1}{1 - \gamma}\left( \perfc_r + \frac{\gamma \perfc_p \sqrt{S}}{1 - \gamma} \right) \left\lVert \pi^{t+1} - \pi^t \right\rVert_2.
    \end{align*}
    Then,
    \begin{align*}
        \frac{1}{T} \sum_{t = 0}^{T - 1} \mu_{\rho, t}^{\pi^t}(s) \pi_i^t(a_i | s) f_i^t(s, a_i)^2 
        &\le \frac{1}{T} \frac{36 \left(\Tilde{V}_{T}^{T}(\rho) - \Tilde{V}_0^{0}(\rho) \right)}{\eta} + \frac{1}{1 - \gamma}\left( \perfc_r + \frac{\gamma \perfc_p \sqrt{S}}{1 - \gamma} \right) \left\lVert \pi^{t+1} - \pi^t \right\rVert_2 \\
        &\le \frac{1}{T} \frac{36 \left(\Tilde{V}_{T}^{T}(\rho) - \Tilde{V}_0^{0}(\rho) \right)}{\eta} + \frac{1}{1 - \gamma}\left( \perfc_r + \frac{\gamma \perfc_p \sqrt{S}}{1 - \gamma} \right) \left\lVert \pi^{t+1} - \pi^t \right\rVert_1 \\
        &\le \frac{1}{T} \frac{36 \left(\Tilde{V}_{T}^{T}(\rho) - \Tilde{V}_0^{0}(\rho) \right)}{\eta} + \frac{1}{1 - \gamma}\left( \perfc_r + \frac{\gamma \perfc_p \sqrt{S}}{1 - \gamma} \right) \mathrm{KL}\left(\pi^{t+1} || \pi^t \right) \\
        &\le \frac{1}{T} \frac{36\sqrt{2} \left(\Tilde{V}_{T}^{T}(\rho) - \Tilde{V}_0^{0}(\rho) \right)}{\eta} \\
        &\quad + \frac{\eta n S}{1 - \gamma}\left( \perfc_r + \frac{\gamma \perfc_p \sqrt{S}}{1 - \gamma} \right) \left( \frac{1}{(1 - \gamma)^2} + 4 \lambda M \left( \lambda A_i M + \frac{1}{(1 - \gamma)^2} \right) \right) ,
    \end{align*}
    where we use Pinsker's inequality for the third inequality, and the last step follows from Lemma~\ref{lemma:inpg_reg_bounds}.
    Thus, by Lemma~\ref{lemma:inpg_reg_policy_diff}:
    \begin{align*}
        \mathrm{Perform\textit{-}Regret}(T) 
        &\le \frac{1}{T} \sum_{t = 0}^{T - 1} \mu_{\rho, t}^{\pi^t}(s) \pi_i^t(a_i | s) f_i^t(s, a_i)^2 + \lambda M A_{\max} \\
        &\le \frac{9\sqrt{2}\left(\Tilde{V}_{T}^{T}(\rho) - \Tilde{V}_0^{0}(\rho) \right)}{\eta \lambda T} + \lambda M A_{\max}
        \\
        &\quad + \frac{\eta n S}{1 - \gamma}\left( \perfc_r + \frac{\gamma \perfc_p \sqrt{S}}{1 - \gamma} \right) \left( \frac{1}{(1 - \gamma)^2} + 4 \lambda M \left( \lambda A_i M + \frac{1}{1 - \gamma} \right) \right) \\
        &\le \frac{9\sqrt{2}}{\eta \lambda (1 - \gamma) T} + \lambda M A_{\max}
        \\
        &\quad + \frac{\eta n S}{1 - \gamma}\left( \perfc_r + \frac{\gamma \perfc_p \sqrt{S}}{1 - \gamma} \right) \left( \frac{1}{(1 - \gamma)^2} + 4 \lambda M \left( \lambda A_{\max} M + \frac{1}{1 - \gamma} \right) \right),
    \end{align*}
    where $A_i \coloneqq |\mathcal{A}_i|$ and $A_{\max} \coloneqq \max_{i \in \mathcal{I}} A_i$. Moreover, for any $\epsilon > 0$, by setting $\lambda = \frac{\epsilon}{3 M A_{\max}}$,
    \begin{align*}
        \eta = \min \Biggl\{ &\frac{1}{15\left(\frac{1}{(1 - \gamma)^2} + \lambda A_{\max} M \right)}, \frac{1}{4 \left( 4 \lambda A_{\max} M^2 + \frac{4M}{(1 - \gamma)^2} + \frac{3nM}{(1 - \gamma)^3}\right)},  \\
        &\frac{1 - \gamma}{3 n S}\left( \perfc_r + \frac{\gamma \perfc_p \sqrt{S}}{1 - \gamma} \right)^{-1} \left( \frac{1}{(1 - \gamma)^2} + 4 \lambda M \left( \lambda A_{\max} M + \frac{1}{(1 - \gamma)^2} \right) \right)^{-1} \Biggr\} \\
        = \min \Biggl\{ &\left(\frac{15}{(1 - \gamma)^2} + 5\epsilon \right)^{-1}, \left( \frac{16\epsilon M}{3} + \frac{16M}{(1 - \gamma)^2} + \frac{12nM}{(1 - \gamma)^3}\right)^{-1},  \\
        &\frac{1 - \gamma}{3 n S}\left( \perfc_r + \frac{\gamma \perfc_p \sqrt{S}}{1 - \gamma} \right)^{-1} \left( \frac{1}{(1 - \gamma)^2} + \frac{4 \epsilon^2}{9 A_{\max}} + \frac{4 \epsilon}{3 A_{\max} (1 - \gamma)^2} \right)^{-1} \Biggr\},
    \end{align*}
    and
    \begin{align*}
        T 
        \ge \mathcal{O} \left( \frac{\Tilde{V}_{T}^{T}(\rho) - \Tilde{V}_0^{0}(\rho)}{\eta \lambda \epsilon} \right)
        \ge \mathcal{O} \left(\frac{n A_{\max} M^2}{\epsilon^2 (1 - \gamma)^4} \max\left\{1, S \left( \perfc_r + \frac{\gamma \perfc_p \sqrt{S}}{1 - \gamma} \right) \right\} \right) ,
    \end{align*}
    we obtain
    \begin{equation*}
        \text{Perform-Regret}(T) \le \frac{\epsilon}{3} + \frac{\epsilon}{3} + \frac{\epsilon}{3} \le \epsilon . \qedhere
    \end{equation*}
\end{proof}

\subsection{Last-Iterate Convergence in MPGs with Performative Effects and Agent Independent Transitions}


Recall the sensitivity assumption.
\begin{assumption}[Sensitivity]\label{Assump: Sensitivity (appendix)}
For any two policies $\pi$ and $\pi'$, we have that for all $i \in \mathcal{N}$,
    \begin{align*}
        \lVert r_{i,\pi}(\cdot, \cdot) - r_{i,\pi'}(\cdot, \cdot) \rVert_2 \leq \zeta_r \cdot \lVert \mu-\mu' \rVert_2,
        \\ 
        \lVert P_\pi(\cdot \mid \cdot, \cdot) - P_{\pi'}(\cdot \mid \cdot, \cdot ) \rVert_2 \leq \zeta_p \cdot \lVert \mu-\mu' \rVert_2.  
    \end{align*}
\end{assumption}

We provide a stronger variant of Theorem~\ref{Thm: Subgame Finite-LIC}, by given a tighter choice for the regularization parameter $\lambda$ which only provides a sublinear dependence on the number of agents.

\begin{theorem}
Let $\alpha_{\min} = \min_{s, \pi} \alpha_{\pi}(s)$, let 
\begin{align*}
\lambda > \frac{\sqrt{A_{\max}}}{A_{\min}}
\cdot O \left( \zeta_p \cdot \frac{S^{2} \sqrt{n} \gamma A^{9/4}_{\max}}{(1-\gamma)^6} +
 \zeta_r \cdot \frac{S^{3/2} \gamma \sqrt{n} A_{\max}^{9/4}}{(1-\gamma)^4} + \frac{S^{3/2} \gamma A^{5/4} \beta}{(1-\gamma)^3 \alpha_{\min}} \right) 
 \end{align*}
 be the fixed point of the objective in Eq.~\eqref{Eq: Repreated Optimization}. It holds that, $\lVert \mu^T - \mu^\lambda \rVert_2 \leq \delta$ for $T \geq 2(1-\mu)^{-1} \ln(2/\delta(1-\gamma))$ and the performative gap is bounded:
 \begin{align*}
     \max_{i \in \mathcal{N}} \max_{\pi'_i} \left( V_{i,\pi^T}^{\pi'_i, \pi^{(T)}_{-i}}(\rho) - V_{i,\pi^T}^{\pi^{(T)}}(\rho) \right) 
         \leq
    \frac{\kappa_\rho}{\min_s \alpha_\lambda(s)(1-\gamma)} \cdot \left( \sqrt{A_{\max}} \cdot \delta + \frac{\lambda}{2(1-\gamma)}\right) .
    \end{align*}
\end{theorem}

\begin{proof}[Proof of Theorem \ref{Thm: Subgame Finite-LIC}]
Given policy $\pi^\lambda$ induced by the computed $\mu^\lambda$, given Lemma~\ref{Lemma: Contraction mapping}, suppose that we bound the performative-gap 
\begin{align*}
    \max_i \max_{\pi'_i} V_{i, \pi}^{\pi'_i, \pi_{-i}}(\rho) - V_{i, \pi}^\pi(\rho)
    &=
    \max_i \max_{\pi'_i} \Phi_\pi^{\pi'_i, \pi_{-i}}(\rho)
    - \Phi_\pi^\pi(\rho)
    \\
    &\leq
    \frac{\kappa_\rho}{1-\gamma} \max_i \max_{\pi'_i} \left\langle \pi'_i - \pi_i, \nabla_{\pi_i} \Phi_\pi^\pi(\rho) \right\rangle
    \\ 
    &\leq
    \frac{\kappa_\rho}{1-\gamma} \cdot \left\langle \pi' - \pi, \nabla_{\pi} \Phi_\pi^\pi(\rho) \right\rangle
    \\ 
    &\leq
    \kappa_\rho \cdot \left\langle \frac{\mu'}{d'_\lambda} - \frac{\mu}{d_\lambda}, \nabla_\pi \Phi_\pi^\pi(\rho) \right\rangle,
\end{align*}
where we apply the gradient domination property in the first inequality.
Further, we exploit the correspondence between policies and occupancy measures due to agent independent transitions in the third inequality.
Further, we have
\begin{align*}
    \kappa_\rho \cdot \left\langle \frac{\mu'}{d'_\lambda} - \frac{\mu}{d_\lambda}, \nabla_\pi \Phi_\pi^\pi(\rho) \right\rangle
    &\leq
    \frac{\kappa_\rho}{\min_s \alpha_\lambda(s)} \langle \mu' - \mu, \nabla_\pi \Phi_\pi^\pi(\rho) \rangle 
    \\ &\leq
    \frac{\kappa_\rho}{\min_s \alpha_\lambda(s)} \cdot \left[ \langle \mu' - \mu^\lambda, \nabla_\pi \Phi_\pi^\pi(\rho) \rangle + \langle \mu^\lambda - \mu, \nabla_\pi \Phi_\pi^\pi(\rho) \rangle \right]
    \\ &\leq
    \frac{\kappa_\rho}{\min_s \alpha_\lambda(s)} \cdot
    \left[ \frac{\sqrt{A_{\max}} \cdot \delta }{1-\gamma} + \frac{\lambda}{2} \left( \lVert \mu' \rVert_2^2 -  \lVert \mu \rVert_2^2 \right) \right]
    \\ &\leq
    \frac{\kappa_\rho}{\min_s \alpha_\lambda(s)(1-\gamma)} \cdot \left( \sqrt{A_{\max}} \cdot \delta + \frac{\lambda}{2(1-\gamma)}\right). \qedhere
\end{align*}
\end{proof}

First, we show the convergence.
\begin{lemma}\label{Lemma: Contraction mapping}
    Repeatedly optimizing Eq.~\eqref{Eq: Repreated Optimization} converges to a fixed point $\mu^\lambda$, more precisely, given that \begin{align*}
\lambda > \frac{\sqrt{A_{\max}}}{A_{\min}}
\cdot O \left( \zeta_p \cdot \frac{S^{2} \sqrt{n} \gamma A^{9/4}_{\max}}{(1-\gamma)^6} +
 \zeta_r \cdot \frac{S^{3/2} \gamma \sqrt{n} A_{\max}^{9/4}}{(1-\gamma)^4} + \frac{S^{3/2} \gamma A^{5/4} \beta}{(1-\gamma)^3 \alpha_{\min}} \right) \end{align*} chosen up to a sufficiently large constant, it holds that
 $\lVert \mu^T - \mu^\lambda \rVert_2 \leq \delta$ for $T \geq 2(1-\mu)^{-1} \ln(2/\delta(1-\gamma))$.
\end{lemma}

Let us recall the primal objective in Eq.~\eqref{Eq: Repreated Optimization} and explicitly formulate the constraint that $\mu = (\mu_1, \dots, \mu_n)$ where $\mu_i$ is a state-action occupancy measure over $\mathcal{S} \times \mathcal{A}_i$ (add flow constraints and $\mu \geq 0$). Hence, we consider the following objective:
\begin{equation}\label{App: Primal optimization} \begin{aligned}
    & \max_{\mu \geq 0} \quad \left\langle \nabla \Phi_t^t(\rho), \mu \right\rangle - \frac{\lambda}{2} \cdot \lVert \mu \rVert_2^2 \\
    & \ \text{s.t.} \quad \sum_{a_i} \mu_i(s,a_i) = \rho(s) + \gamma \cdot \sum_{s'} P_{t}(s \mid s') \sum_{a_i} \mu_i(s',a_i)  \quad \text{for all} \quad s \in \mathcal{S}, i \in \mathcal{N}.
\end{aligned}
\end{equation}
The corresponding Lagrangian is formulated as: 
\begin{equation} \label{eq: lagrangian}
\begin{aligned}
    \mathcal{L}(\mu, h) 
    = &\langle \mu, \nabla \Phi_t^t(\rho) \rangle - \frac{\lambda}{2} \lVert\mu \rVert_2^2 \\
    &+  \sum_{i \in \mathcal{N}} \sum_s h_i(s) \left( - \sum_{a_i} \mu_i(s, a_i) + \rho(s) + \gamma \cdot \sum_{s'_i, a_i} \mu_i(s', a_i) P_{t}(s \mid s')
    \right) .
\end{aligned}
\end{equation}
To find an optimal $\mu$, we take the gradient $\nabla_\mu \mathcal{L}$ and set it to zero:
\begin{align*}
     \partial_{x_{i,s,a_i}} \Phi_t^t(\rho) - \lambda \cdot \mu_i(s, a_i) - h_i(s)  + \gamma \cdot  \sum_{\widetilde{s}} h_i(\widetilde{s}) P_{t}(\widetilde{s} \mid s)  = 0.
\end{align*}
Recall, that with $\partial_{x_{i, s, a_i}} \Phi_t^t(\rho)$, we refer to the partial derivative with respect to the played occupancy measure.
Solving for $\mu_i(s, a_i)$, we obtain that
\begin{align*}
    \mu_i(s, a_i) 
    = \frac{1}{\lambda}
    \cdot \left( \partial_{x_{i,s,a_i}} \Phi_t^t(\rho) - h_i(s) + \gamma \cdot \sum_{\widetilde{s}} h_i(\widetilde{s}) P_{t}(\widetilde{s} \mid s) \right).
\end{align*}
We substitute this value back to obtain the Lagrangian dual formulation:
\begin{equation}\label{eq: Dual objective}
\begin{split}
    \min_{h \in \RR^{n \times S}} & - \frac{1}{\lambda} \sum_{i} \sum_{s, a_i} h_i(s) \cdot \partial_{x_{i,s,a_i}} \Phi_t^t(\rho) + \frac{\gamma}{\lambda} \sum_{i} \sum_{s} \sum_{s'. a_i} 
    \partial_{\mu_i, s', a_i} \Phi(s', a_i) \cdot h_i(s) \cdot P_{t}(s \mid s') 
    \\ &+  \sum_{i} \sum_{s} h_i(s) \rho(s) +
    \frac{1}{2\lambda} \sum_{i} A_i \sum_{s} h_i(s)^2 
    - \frac{\gamma}{\lambda} \sum_{i} \sum_{s, a_i} h_i(s) \sum_{s'_i} h_i(s') P_{t}(s \mid s')
    \\ &+ \frac{\gamma^2}{2\lambda} \sum_{i} \sum_{s, a_i} \sum_{\widetilde{s}, \widehat{s}} h_i(\widetilde{s})h_i(\widehat{s}) P_{t}(\widehat{s} \mid s) P_{t}(\widetilde{s} \mid s). 
\end{split}
\end{equation}

The dual is objective is parameterized with $\nabla \Phi_t^t$ and probability transition function $P_{t}$, which illustrates the performative effect given the underlying game $\mathcal{G}_t$ (shorthand notation for $\mathcal{G}(\pi^t)$) induced by the state occupancy measure $\mu^t$.
Further, we have that the gradient is dependent on the played policy $\pi^t$.
These parameters capture indeed the influence of $\mu^t$ and we denote the dual objective as $\mathcal{L}(\cdot, \mathcal{G}_t, \pi^t)$ to express that.

Let $\mathrm{GD}(\mu^t)$ be the optimal solution to the primal problem, given that $\mathcal{G}_t$ is the underlying game. We show that $\mathrm{GD}(\cdot)$ corresponds to a contraction mapping.
%
Given two occupancy measures $\mu, \widehat{\mu}$, we denote $\nabla \Phi = \nabla \Phi_\pi^\pi$ (resp.\ $\nabla \widehat{\Phi} = \nabla \Phi_{\widehat{\pi}}^{\widehat{\pi}}$) and $P$ (resp.\ $\widehat{P}$) and $\pi$ (respectively $\widehat{\pi}$) as the implemented policy.
Further, let $h$ respectively $\widehat{h}$ be the associated optimal solutions for the dual objective in Eq.~\eqref{eq: Dual objective}.
By the strong-convexity property, see Lemma~\ref{Lemma: Strong convex}, the following both inequalities hold:
\begin{align*}
    \mathcal{L}(h, M, \pi) - \mathcal{L}(\widehat{h}, M, \pi) &\geq \left\langle h - \widehat{h}, \nabla \mathcal{L}(\widehat{h}, M) \right\rangle + \frac{A_{\min} (1-\gamma)^2}{2\lambda} \left\lVert h - \widehat{h} \right\rVert_2^2 \, \\
    \mathcal{L}(\widehat{h}, M, \pi) - \mathcal{L}(h, M, \pi) &\geq \frac{A_{\min} (1-\gamma)^2}{2\lambda}\left\lVert h - \widehat{h} \right\rVert_2^2,
\end{align*}
which implies that
\begin{equation}\label{Eq: StrongConvex}
    - \frac{A_{\min}(1-\gamma)^2}{\lambda} \left\lVert h - \widehat{h} \right\rVert_2^2 \geq \left\langle h - \widehat{h}, \nabla \mathcal{L}(\widehat{h}, M) \right\rangle = \left\langle h - \widehat{h}, \nabla \mathcal{L}(\widehat{h}, M) - \nabla \mathcal{L}(\widehat{h}, \widehat{M}) \right\rangle,
\end{equation}
where the last inequality uses that $\widehat{h}$ is optimal for $\mathcal{L}(\cdot, \widehat{M})$.
Further, we can apply Lemma~\ref{Lemma: L smooth} to show that:
\begin{align*}
    \left\lVert \nabla \mathcal{L}(\widehat{h}, M) - \nabla\mathcal{L}(\widehat{h}, \widehat{M}) \right\rVert_2
    &\leq
    \frac{\gamma S \sqrt{10 A_{\max} \left(1 + \lVert \nabla\Phi(\rho) \rVert_\infty \right)}}{\lambda} \left\lVert \nabla \Phi(\rho) - \nabla \widehat{\Phi}(\rho) \right\rVert_2 \\
    & + \frac{5 \gamma S \sqrt{A_{\max} \left(1 + \lVert \nabla \Phi(\rho) \rVert_\infty \right)}}{\lambda} \left\lVert h \right\rVert_2 \left\lVert P - \widehat{P} \right\rVert_2.
\end{align*}

Further, observe that using Lemma~\ref{Lemma: gradient bound} and Lemma~\ref{Eq: Smoothness inequality}, we obtain that, (recall that $\Phi = \Phi_\pi^\pi$ and $\widehat{\Phi} = \Phi_{\widehat{\pi}}^{\widehat{\pi}}$)
\begin{align*}
    \lVert \nabla \Phi - \nabla \widehat{\Phi}(\rho) \rVert_2 &\leq
    \lVert \nabla \Phi - \nabla \Phi^{\widehat{\pi}}_\pi \rVert_2 + \lVert \nabla \Phi_\pi^{\widehat{\pi}} - \nabla \widehat{\Phi} \rVert_2
    \\ &\leq
    \left( \frac{\beta}{\alpha_{\min}} + \frac{2 \gamma \sqrt{n S A_{\max}}}{(1-\gamma)^{3}} \zeta_p + \frac{\sqrt{n A_{\max}}}{(1-\gamma)^2}\zeta_r \right)\lVert \mu - \widehat{\mu} \rVert_2,
\end{align*}
where we denote $\alpha_{\min} = \min_{\pi, s} \alpha_{\pi}(s)$.
By combining Lemma~\ref{Lemma: L opt sol bounded} with the observation that $\lVert \nabla\Phi(\rho) \rVert_\infty \leq \frac{\sqrt{A_{\max}}}{(1-\gamma)^2}$, which holds independent of the choice of the underlying environment $\mathcal{M}$ and the played policy $\pi$, we get that $\lVert h \rVert_2 \leq \frac{\sqrt{9n S}}{(1 - \gamma)^2} \left\lVert \nabla\Phi_t^t(\rho) \right\rVert_\infty \leq \frac{3 \sqrt{n S A_{\max}}}{(1-\gamma)^4}$.
Moreover, we apply the Sensitivity Assumption~\ref{Assump: Sensitivity (appendix)} to bound $\lVert P - \widehat{P} \rVert_2 \leq \perfc_p \cdot \lVert \mu - \widehat{\mu} \rVert_2$.
This leads to
\begin{align*}
    \left\lVert \nabla \mathcal{L}(\widehat{h}, M) - \nabla\mathcal{L}(\widehat{h}, \widehat{M}) \right\rVert_2
    &\leq
    \left( \frac{S \gamma A_{\max}^{3/4}\beta}{\lambda (1-\gamma) \alpha_{\min}}\right) \cdot \lVert \mu - \widehat{\mu} \rVert_2 \\
        & + \frac{38 \gamma \sqrt{n} S^{3/2} A_{\max}^{5/4}}{\lambda (1-\gamma)^4} \cdot \zeta_p \cdot \left\lVert \mu - \widehat{\mu} \right\rVert_2 + \frac{\zeta_r \gamma S \sqrt{n} A^{5/4}_{\max}}{(1-\gamma)^2} \cdot \lVert \mu - \widehat{\mu} \rVert_2.
\end{align*}
We substitute the latter bound into Eq.~\eqref{Eq: StrongConvex} to derive
\begin{align*}
    - \frac{A_{\min}(1-\gamma)^2}{\lambda} \lVert h - \widehat{h} \rVert_2^2 
    &\geq
    - \lVert h - \widehat{h} \rVert_2 \cdot \lVert \nabla \mathcal{L}(\widehat{h}, M) - \nabla \mathcal{L}(\widehat{h}, \widehat{M}) \rVert_2 
    \\ &\geq 
    - \lVert h - \widehat{h} \rVert_2 \left( \frac{38 \gamma S^{3/2} \cdot \sqrt{n} A_{\max}^{5/4} }{\lambda (1-\gamma)^4} \cdot \zeta_p + \frac{\zeta_r \gamma S \sqrt{n} A_{\max}^{5/4}}{(1-\gamma)^2} + \frac{S \gamma A^{3/4} \beta}{\lambda (1-\gamma) \alpha_{\min}}\right) \cdot \lVert \mu - \widehat{\mu} \rVert_2,
\end{align*}
so that
\begin{align*}
    \lVert h - \widehat{h} \rVert_2 
    \leq
    \frac{\lambda}{A_{\min}(1-\gamma)^2}\left( \frac{38 \gamma S^{3/2} \cdot \sqrt{n} A_{\max}^{5/4} }{\lambda (1-\gamma)^4} \cdot \zeta_p + \frac{\zeta_r \gamma S \sqrt{n} A_{\max}^{5/4}}{(1-\gamma)^2} + \frac{S \gamma A^{3/4} \beta}{\lambda (1-\gamma) \alpha_{\min}}\right) \left\lVert \mu - \widehat{\mu} \right\rVert_2.
\end{align*}
Finally, we have the ingredients to bound the difference between the optimal primal solution ($\mathrm{GD}(\mu)$) when the deployed occupancy measure is $\mu$ with $\mathrm{GD}(\widehat{\mu})$.
Let us define
\begin{align*}
     &4\left( \frac{\sqrt{n A_{\max}}}{(1-\gamma)^2}\zeta_r + 2\frac{\sqrt{S n  A_{\max}} \zeta_p \gamma}{(1-\gamma)^{3}} \right) + 6 \gamma \zeta_p \left\lVert h \right\rVert_2
    \\ &\leq
    4\left( \frac{\sqrt{n A_{\max}}}{(1-\gamma)^2}\zeta_r + 2\frac{\sqrt{S n  A_{\max}} \zeta_p \gamma}{(1-\gamma)^{3}} \right) + 6 \gamma \zeta_p \sqrt{n S A_{\max}} / (1-\gamma)^4 \eqqcolon K.
\end{align*}
By Lemma~\ref{Lemma: occupancy measures}, we have that,
\begin{align*}
    &\lVert \mathrm{GD}(\mu) - \mathrm{GD}(\widehat{\mu})\rVert_2
    \leq
      \left(1 + \frac{K}{\lambda}\right) \frac{3\sqrt{S A_{\max}}}{\lambda A_{\min} (1-\gamma)^2} \left\lVert h - \widehat{h} \right\rVert_2
      \\ &\leq 
      \underbrace{\left(1 + \frac{K}{\lambda}\right) \frac{\sqrt{S A_{\max}}}{(1-\gamma)^2 A_{\min}}\left( \frac{38 \gamma S^{3/2} \cdot \sqrt{n} A_{\max}^{5/4} }{\lambda (1-\gamma)^4} \cdot \zeta_p + \frac{\zeta_r \gamma S \sqrt{n} A_{\max}^{5/4}}{(1-\gamma)^2} + \frac{S \gamma A^{3/4} \beta}{\lambda (1-\gamma) \alpha_{\min}}\right)}_{\ell} \left\lVert \mu - \widehat{\mu} \right\rVert_2 .
\end{align*}
So, we obtain a bound 
\begin{align*}
    \lVert \mathrm{GD}(\mu) - \mathrm{GD}(\widehat{\mu})\rVert_2
    \leq \xi \cdot \lVert \mu - \widehat{\mu} \rVert_2.
\end{align*}
For the choice 
\begin{align*}
\lambda > \frac{\sqrt{A_{\max}}}{A_{\min}} 
\cdot O \left( \zeta_p \cdot \frac{S^{2} \sqrt{n} \gamma A^{9/4}_{\max}}{(1-\gamma)^6} +
 \zeta_r \cdot \frac{S^{3/2} \gamma \sqrt{n} A_{\max}^{9/4}}{(1-\gamma)^4} + \frac{S^{3/2} \gamma A^{5/4} \beta}{(1-\gamma)^3 \alpha_{\min}} \right),
\end{align*}
we have that $\xi < 1$ for selecting $\lambda$ up to sufficiently large constants. 
So that the operator $\mathrm{GD}$ is indeed a contraction map and converges against a fixed point $\mu^\lambda$. This implies, that for $t \geq \ln{(\lVert d_0 - d_\lambda \rVert_2) / \delta)} \ln{1/\xi}$, it is guaranteed that $\lVert \mu^t - \mu^\lambda \rVert_2 \leq \delta$.
\begin{lemma}\label{Lemma: Strong convex}
    The dual objective $\mathcal{L}_\mu(\cdot, M)$ (defined in Eq.~\eqref{eq: Dual objective}) is $\left[ A_{\min} \frac{(1 - \gamma)^2}{\lambda} \right]$-strongly convex with respect to $h$ given a fixed $M$.
\end{lemma}

\begin{proof}
    The partial derivative of the objective in Eq.~\eqref{eq: Dual objective} is given by
    \begin{align}\label{eq:dualderivative}
        \begin{split}
            \frac{\partial \mathcal{L}_\mu(h)}{\partial h_i(s)}
            &= - \frac{1}{\lambda} \sum_{a_i} \partial_{x_{i,s,a_i}} \Phi_t^t(\rho) + \rho(s)
             + \frac{\gamma}{\lambda} \sum_{s', a_i} \partial_{x_{i,s,a_i}} \Phi_t^t(\rho) P_{t}( s \mid s') \\
            &\quad + \frac{A_i}{\lambda} h_i(s) - \frac{\gamma}{\lambda} \sum_{s'} h_i(s') \sum_{a_i} \left( P_{t}(s' \mid s) + P_{t} (s \mid s') \right) \\
            &\quad + \frac{\gamma^2}{\lambda} \sum_{\widetilde{s}} h_i(\widetilde{s}) \sum_{s', a_i} P_{t}(\widetilde{s} \mid s')P_{t}(s \mid s').
        \end{split}
    \end{align}
    Let $A_{\min} \coloneqq \min_{i} A_i$, and $I$ denote the $S \times S$ identity matrix. Moreover, define $M \in \mathbb{R}^{S \times S}$ such that $M(s, s') = P_{t} (s' \mid s)$ for all $s, \widetilde{s} \in \mathcal{S}$.
    We obtain the following inequality:
    \begin{align*}
        &\left\langle \nabla\mathcal{L}_\mu(h, M, \pi) - \nabla\mathcal{L}_\mu(\widehat{h}, M, \pi), h - \widehat{h} \right\rangle \\
        &= \frac{1}{\lambda} \sum_{i, s} A_i \left(h_i(s) - \widehat{h}_i(s)\right)^2 \\
        &\quad - \frac{\gamma}{\lambda} \sum_{i, s} \left(h_i(s) - \widehat{h}_i(s) \right) \sum_{s'} \left(h_i(s') - \widehat{h}_i(s')\right) \sum_{a_i} \left(P_{t}(s' \mid s) + P_{t}(s \mid s') \right) \\
        &\quad + \frac{\gamma^2}{\lambda} \sum_{i, s} \left(h_i(s) - \widehat{h}_i(s)\right) \sum_{\widetilde{s}} \left(h_i(\widetilde{s}) - \widehat{h}_i(\widetilde{s})\right) \sum_{s', a_i} P_{t}(\widetilde{s} \mid s') P_{t}(s \mid s') \\
        &= \frac{1}{\lambda} \sum_{i} \sum_{a_i} \left(h_i - \widehat{h}_i\right)^T \left(I - \gamma M - \gamma^2 M^T M \right) \left(h_i - \widehat{h}_i\right) \\
        &\ge \frac{(1 - \gamma)^2}{\lambda} \sum_{i} A_i \left\lVert h_i - \widehat{h}_i \right\rVert_2^2 \\
        &\ge A_{\min} \frac{(1 - \gamma)^2}{\lambda} \left\lVert h - \widehat{h} \right\rVert_2^2,
    \end{align*}
    where the first inequality follows from \citet[Lemma~5]{PerformativeReinforcementLearning}.
\end{proof}

\begin{lemma}\label{Lemma: L smooth}
    The dual objective $\mathcal{L}$ (defined in Eq.~\eqref{eq: Dual objective}) satisfies the following bound for any $h$ and MPGs $M, \widehat{M}$:
    \begin{align*}
        \left\lVert \nabla \mathcal{L}(h, M, \pi) - \nabla \mathcal{L}(h, \widehat{M}, \widehat{\pi}) \right\Vert_2 \leq 
        &\frac{\gamma S \sqrt{10 A_{\max} \left(1 + \lVert \nabla\Phi(\rho) \rVert_\infty \right)}}{\lambda} \left\lVert \nabla \Phi(\rho) - \nabla \widehat{\Phi}(\rho) \right\rVert_2 \\
        & + \frac{5 \gamma S \sqrt{A_{\max} \left(1 + \lVert \nabla \Phi(\rho) \rVert_\infty \right)}}{\lambda} \left\lVert h \right\rVert_2 \left\lVert P - \widehat{P} \right\rVert_2 .
    \end{align*}
\end{lemma}

\begin{proof}
    Let $A_{\max} \coloneqq \max_{i} A_i$.
    By the partial derivative of the dual objective in Eq.~\eqref{eq:dualderivative}, we obtain
    \begin{align*}
        &\left\lVert \nabla \mathcal{L}_\mu(h, M, \pi) - \nabla \mathcal{L}_\mu(h, \widehat{M}, \widehat{\pi}) \right\Vert_2^2 \\
        &\leq \frac{5}{\lambda^2} \sum_{i} \sum_{s} \sum_{a_i} \left(\partial_{x_{i,s,a_i}}\ \Phi(\rho) - \partial_{x_{i,s,a_i}}\widehat{\Phi}(\rho)\right)^2 \\
        &\quad + \frac{5 \gamma^2}{\lambda^2} \sum_{i} \sum_{s} \sum_{s', a_i} \left(\partial_{x_{i,s,a_i}}\ \Phi(\rho) P(s' \mid s) - \partial_{x_{i,s,a_i}}\widehat{\Phi}(\rho) \widehat{P}(s' \mid s) \right)^2 \\
        &\quad + \frac{5 \gamma^2}{\lambda^2} \sum_{i} \sum_{s} \sum_{s', a_i} h_i(s') \left(P(s' \mid s) - \widehat{P}(s' \mid s) \right)^2 \\
        &\quad + \frac{5 \gamma^2}{\lambda^2} \sum_{i} \sum_{s} \sum_{s', a_i} h_i(s') \left(P(s \mid s') - \widehat{P}_{i}(s \mid s') \right)^2 \\
        &\quad + \frac{5 \gamma^4}{\lambda^2} \sum_{i} \sum_{s} \sum_{\widetilde{s}} h_i(\widetilde{s}) \sum_{s', a_i} \left(P(\widetilde{s} \mid s') P(s \mid s') - \widehat{P}(\widehat{s} \mid s') \widehat{P}(s \mid s') \right)^2 \\
        &\leq \frac{5}{\lambda^2} \left\lVert \nabla \Phi(\rho) - \nabla \widehat{\Phi}(\rho) \right\rVert_2^2 + \frac{5 \gamma^2}{\lambda^2} \left(1 + \left\lVert \nabla \Phi(\rho) \right\rVert_\infty \right) S^2 A_{\max} \left\lVert \nabla \Phi(\rho) - \nabla \widehat{\Phi}(\rho) \right\rVert_2^2 \\
        &\quad + \frac{5 \gamma^2}{\lambda^2} \left(1 + \left\lVert \nabla \Phi(\rho) \right\rVert_\infty \right) \left\lVert \nabla \Phi(\rho) \right\rVert_\infty SA_{\max} \left\lVert P - \widehat{P} \right\rVert \\
        &\quad + \frac{10 \gamma^2}{\lambda^2} A_{\max} \left\lVert h \right\rVert_2^2 \left\lVert P - \widehat{P} \right\rVert_2^2 + \frac{20 \gamma^4}{\lambda^2} S^2 A_{\max} \left\lVert h \right\rVert_2^2 \left\lVert P - \widehat{P} \right\rVert_2^2 \\
        &= \left( \frac{5}{\lambda^2} + \frac{5 \gamma^2}{\lambda^2} \left(1 + \left\lVert \nabla \Phi(\rho) \right\rVert_\infty \right) S^2 A_{\max} \right) \left\lVert \nabla \Phi(\rho) - \nabla \widehat{\Phi}(\rho) \right\rVert_2^2 \\
        &\quad + \left(\frac{5 \gamma^2}{\lambda^2} \left(1 + \left\lVert \nabla \Phi(\rho) \right\rVert_\infty \right) S^2 A_{\max} + \frac{10 \gamma^2}{\lambda^2} A_{\max} \left\lVert h \right\rVert_2^2 + \frac{20 \gamma^4}{\lambda^2} S^2 A_{\max} \left\lVert h \right\rVert_2^2 \right) \left\lVert P - \widehat{P} \right\rVert_2^2 ,
    \end{align*}
    where the first inequality follows from Jensen's inequality and
    the second inequality uses the following inequalities:
    \begin{align*}
        &\sum_{s', a_i} \left(\partial_{x_{i,s,a_i}}\ \Phi(\rho) P(s' \mid s) - \partial_{x_{i,s,a_i}}\widehat{\Phi}(\rho) \widehat{P}(s' \mid s) \right)^2 \\
        &\le \sum_{s', a_i} \left( \left\lvert \partial_{x_{i,s,a_i}}\ \Phi(\rho) - \partial_{x_{i,s,a_i}}\widehat{\Phi}(\rho) \right\rvert + \left\lVert \nabla \Phi(\rho) \right\rVert_\infty \left\lvert P(s \mid s') - \widehat{P}(s \mid s') \right\rvert \right)^2 \\
        &\le \left(1 + \left\lVert \nabla \Phi(\rho) \right\rVert_\infty \right) \left( \sum_{s', a_i} \left\lvert \partial_{x_{i,s,a_i}}\ \Phi(\rho) - \partial_{x_{i,s,a_i}}\widehat{\Phi}(\rho) \right\rvert \right)^2 \\
        &\quad + \left(1 + \left\lVert \nabla \Phi(\rho) \right\rVert_\infty \right) \left\lVert \nabla \Phi(\rho) \right\rVert_\infty \left(\sum_{s', a_i} \left\lvert P(s \mid s') - \widehat{P}(s \mid s') \right\rvert \right)^2 \\
        &\le \left(1 + \left\lVert \nabla \Phi(\rho) \right\rVert_\infty \right) S^2 A_i \sum_{a_i} \left( \partial_{x_{i,s,a_i}}\ \Phi(\rho) - \partial_{x_{i,s,a_i}}\widehat{\Phi}(\rho) \right)^2 \\
        &\quad + \left(1 + \left\lVert \nabla \Phi(\rho) \right\rVert_\infty \right) \left\lVert \nabla \Phi(\rho) \right\rVert_\infty S A_i \sum_{s', a_i} \left( P(s, s') - \widehat{P}(s \mid s') \right)^2 ;
    \end{align*}

    \begin{align*}
        &\sum_{i} \sum_{s} \sum_{s', a_i} h_i(s') \left(P(s' \mid s) - \widehat{P}(s' \mid s) \right)^2 \\
        &\le \sum_{i} \sum_{s} \left(\sum_{s', a_i} h_i^2(s') \right) \left(\sum_{s', a_i}  \left(P(s' \mid s) - \widehat{P}(s' \mid s)\right)^2 \right) \\
        &\le \sum_{i} A_i \left\lVert h_i \right\rVert_2^2 \left\lVert P - \widehat{P} \right\rVert_2^2 \\
        &\le n A_{\max} \left\lVert h \right\rVert_2^2 \left\lVert P - \widehat{P} \right\rVert_2^2 ;
    \end{align*}

    \begin{align*}
        &\sum_{i} \sum_{s} \sum_{s', a_i} h_i(s') \left(P(s \mid s') - \widehat{P}(s \mid s') \right)^2 \\
        &\le \sum_{i} \sum_{s} \left(\sum_{s', a_i} h_i^2(s') \right) \left(\sum_{s', a_i}  \left(P(s \mid s') - \widehat{P}(s \mid s')\right)^2 \right) \\
        &\le n A_{\max} \left\lVert h_i \right\rVert_2^2 \left\lVert P - \widehat{P} \right\rVert_2^2 \\
        &\le A_{\max} \left\lVert h \right\rVert_2^2 \left\lVert P - \widehat{P} \right\rVert_2^2 ;
    \end{align*}

    \begin{align*}
        &\sum_{i} \sum_{s} \sum_{\widetilde{s}, a_i, s'} h_i(\widetilde{s}) \left(P(\widetilde{s} \mid s') P(s \mid s') - \widehat{P}(\widehat{s} \mid s') \widehat{P}(s \mid s') \right)^2 \\
        &\le \sum_{i} \sum_{s} \left(\sum_{s', \widetilde{s}} h_i^2(\widetilde{s}) \right) \left( \sum_{s', \widetilde{s}} \left(P(\widetilde{s}, s') P(s \mid s') - \widehat{P}(\widehat{s} \mid s') \widehat{P}(s \mid s') \right)^2 \right) \\
        &\le \sum_i S A_i \left\lVert h_i \right\rVert_2^2 \left( \sum_{s, \widetilde{s}} \sum_{s', a_i} \left(P(\widetilde{s} \mid s') P(s \mid s') - \widehat{P}(\widehat{s} \mid s') \widehat{P}(s \mid s') \right)^2 \right) \\
        &\le \sum_i S A_i \left\lVert h_i \right\rVert_2^2 \sum_{s, \widetilde{s}} \sum_{s', a_i} \left(\left\lvert P(\widetilde{s} \mid s') - P(s \mid s')\right\rvert + \left\lvert \widehat{P}(\widehat{s} \mid s') - \widehat{P}(s \mid s') \right\rvert \right)^2 \\
        &\le 4 \sum_i S A_i \left\lVert h_i \right\rVert_2^2 \sum_{s, \widetilde{s}} \sum_{s', a_i} \left( \left( P(\widetilde{s} \mid s') - P(s \mid s')\right)^2 + \left( \widehat{P}(\widehat{s} \mid s') - \widehat{P}(s \mid s') \right)^2 \right) \\
        &\le 4 S^2 A_{\max} \sum_{i} \left\lVert h_i \right\rVert_2^2 \left\lVert P - \widehat{P} \right\rVert_2^2 \\
        &\le 4 S^2 A_{\max} \left\lVert h \right\rVert_2^2 \left\lVert P - \widehat{P} \right\rVert_2^2 . \qedhere
    \end{align*}
\end{proof}

\begin{lemma}\label{Lemma: occupancy measures}
    Consider two state-action occupancy measures $\mu$ and $\widehat{\mu}$. For $\lambda \geq 4n \zeta_r \sqrt{S} + 6 \gamma \perfc_p \left\lVert h \right\rVert_2$, the following bound holds:
    \begin{align*}
        \lVert \mu - \widehat{\mu} \rVert_2 
        \leq \left(1 + \frac{4\left( \frac{\sqrt{n A_{\max}}}{(1-\gamma)^2}\zeta_r + 2\frac{\sqrt{S n  A_{\max}} \zeta_p \gamma}{(1-\gamma)^{3}} \right) + 6 \gamma \zeta_p \left\lVert h \right\rVert_2}{\lambda}\right) \frac{3\sqrt{S A_{\max}}}{\lambda} \left\lVert h - \widehat{h} \right\rVert_2^2 .
    \end{align*}
\end{lemma}

\begin{proof}
    By the partial derivative of the dual objective in Eq.~\eqref{eq:dualderivative}, we obtain
    \begin{align*}
        \left(\mu_i(s, a_i) - \widehat{\mu}_i(s, a_i) \right)^2
        &= \frac{1}{\lambda^2} \biggl( \left( \partial_{x_{i,s,a_i}} \Phi(\rho) - \partial_{x_{i,s,a_i}} \widehat{\Phi}(\rho) \right) + \left(-h_i(s) + \widehat{h}_i(s) \right) \\
        &\quad + \gamma \biggl( \sum_{\widetilde{s}} h_i(\widetilde{s}) P(\widetilde{s} \mid s) - \sum_{\widetilde{s}} \widehat{h}_i(\widetilde{s}) \widehat{P}(\widetilde{s} \mid s) \biggr) \biggr)^2 \\
        &\leq \frac{3}{\lambda^2} \biggl( \left( \partial_{x_{i,s,a_i}} \Phi(\rho) - \partial_{x_{i,s,a_i}} \widehat{\Phi}(\rho) \right)^2 + \left(-h_i(s) + \widehat{h}_i(s) \right)^2 \\
        &\quad + \gamma^2 \biggl( \sum_{\widetilde{s}} h_i(\widetilde{s}_i) P(\widetilde{s} \mid s) - \sum_{\widetilde{s}} \widehat{h}_i(\widetilde{s}_i) \widehat{P}(\widetilde{s} \mid s) \biggr)^2 \biggr) \\
        &\leq \frac{3}{\lambda^2} \biggl( \left( \partial_{x_{i,s,a_i}} \Phi(\rho) - \partial_{x_{i,s,a_i}} \widehat{\Phi}(\rho) \right)^2 + \left(-h_i(s) + \widehat{h}_i(s) \right)^2 \\
        &\quad + 2\gamma \biggl( \sum_{\widetilde{s}} \left( h_i(\widetilde{s}) - \widehat{h}_i(\widetilde{s})\right) P(\widetilde{s} \mid s) \biggr)^2 \\
        &\quad + 2\gamma \biggl( \sum_{\widetilde{s}} \widehat{h}_i(\widetilde{s}_i) \left( P(\widetilde{s} \mid s) - \widehat{P}(\widetilde{s} \mid s) \right) \biggr)^2 \\
        &\leq \frac{3}{\lambda^2} \biggl( \left( \partial_{x_{i,s,a_i}} \Phi(\rho) - \partial_{x_{i,s,a_i}} \widehat{\Phi}(\rho) \right)^2 + \left(-h_i(s) + \widehat{h}_i(s) \right)^2 \\
        &\quad + 2\gamma \biggl( \sum_{\widetilde{s}} \left( h_i(\widetilde{s}) - \widehat{h}_i(\widetilde{s})\right) P(\widetilde{s} \mid s) \biggr)^2 \\
        &\quad + 2\gamma \biggl( \sum_{\widetilde{s}} \widehat{h}_i(\widetilde{s}) \left( P(\widetilde{s} \mid s) - \widehat{P}(\widetilde{s} \mid s) \right) \biggr)^2 \\
        &\leq \frac{3}{\lambda^2} \biggl( \left( \partial_{x_{i,s,a_i}} \Phi(\rho) - \partial_{x_{i,s,a_i}} \widehat{\Phi}(\rho) \right)^2 + \left(-h_i(s) + \widehat{h}_i(s) \right)^2 \\
        &\quad + 2\gamma \biggl( \sum_{\widetilde{s}} \left( h_i(\widetilde{s}) - \widehat{h}_i(\widetilde{s})\right) \biggr)^2 \\
        &\quad + 2\gamma \biggl(\sum_{\widetilde{s}} \widehat{h}_i^2(\widetilde{s}) \biggr) \biggl( \sum_{\widetilde{s}} \left( P(\widetilde{s} \mid s) - \widehat{P}(\widetilde{s} \mid s) \right)^2 \biggr).
    \end{align*}
    Therefore,
    \begin{align*}
        \lVert \mu - \widehat{\mu} \rVert_2^2
        &\leq \frac{3}{\lambda^2} \left\lVert \nabla\Phi(\rho) - \nabla\widehat{\Phi}(\rho) \right\rVert_2^2 + \frac{3}{\lambda^2} \left\lVert h - \widehat{h} \right\rVert_2^2 + \frac{6 \gamma^2}{\lambda^2} \sum_{i} \sum_{s, a_i} \left\lVert h_i - \widehat{h}_i \right\rVert_2^2 \\
        &\quad + \frac{6 \gamma^2}{\lambda^2} \sum_{i} \sum_{s, a_i} \left\lVert h_i  \right\rVert_2^2 \left\lVert P(\cdot \mid s) - \widehat{P}(\cdot \mid s) \right\rVert_2^2 \\
        &\leq \frac{3}{\lambda^2} \left\lVert \nabla\Phi(\rho) - \nabla\widehat{\Phi}(\rho) \right\rVert_2^2 + \frac{9}{\lambda^2} S \cdot \sum_i A_i \cdot \left\lVert h - \widehat{h} \right\rVert_2^2 + \frac{6 \gamma^2}{\lambda^2} \left\lVert h \right\rVert_2^2 \left\lVert P - \widehat{P} \right\rVert_2^2 .
    \end{align*}
    By Lemma~\ref{Lemma: gradient bound} and  Assumption~\ref{Assump: Sensitivity (appendix)}:
    \begin{align*}
        \lVert \mu - \widehat{\mu} \rVert_2 
        \le \frac{2}{\lambda} \left( \frac{\sqrt{n A_{\max}}}{(1-\gamma)^2}\zeta_r + 2\frac{\sqrt{S n  A_{\max}} \zeta_p \gamma}{(1-\gamma)^{3}} \right) \left\lVert \mu - \widehat{\mu} \right\rVert_2 + \frac{3}{\lambda} \sqrt{S A_{\max}} \left\lVert h - \widehat{h} \right\rVert_2 + \frac{3 \gamma}{\lambda} \zeta_p \left\lVert h \right\rVert_2 \left\lVert \mu - \widehat{\mu} \right\rVert_2 .
    \end{align*}
    Rearranging the terms yields the following bound:
    \begin{align*}
        \lVert \mu - \widehat{\mu} \rVert_2
        &\leq \left(1 - \frac{2 \left( \frac{\sqrt{n A_{\max}}}{(1-\gamma)^2}\zeta_r + 2\frac{\sqrt{S n  A_{\max}} \zeta_p \gamma}{(1-\gamma)^{3}} \right) + 3 \gamma \zeta_p \left\lVert h \right\rVert_2}{\lambda}\right)^{-1} \frac{3\sqrt{S A_{\max n}}}{\lambda} \left\lVert h - \widehat{h} \right\rVert_2^2 \\
        &\leq \left(1 + \frac{4\left( \frac{\sqrt{n A_{\max}}}{(1-\gamma)^2}\zeta_r + 2\frac{\sqrt{S n  A_{\max}} \zeta_p \gamma}{(1-\gamma)^{3}} \right) + 6 \gamma \zeta_p \left\lVert h \right\rVert_2}{\lambda}\right) \frac{3\sqrt{S A_{\max}}}{\lambda} \left\lVert h - \widehat{h} \right\rVert_2^2
    \end{align*}
    for $\lambda \geq 4\left( \frac{\sqrt{n A_{\max}}}{(1-\gamma)^2}\zeta_r + 2\frac{\sqrt{S n  A_{\max}} \zeta_p \gamma}{(1-\gamma)^{3}} \right)$.
\end{proof}

\begin{lemma}\label{Lemma: L opt sol bounded}
    The norm of the optimal solution to the dual objective (defined in Eq.~\eqref{eq: Dual objective}) is bounded by $\frac{\sqrt{9n S}}{(1 - \gamma)^2} \left\lVert \nabla\Phi_t^t(\rho) \right\rVert_\infty$ for any MPG $\mathcal{M}$.
\end{lemma}

\begin{proof}
    Rearranging the terms of the partial derivative of the dual objective in Eq.~\eqref{eq:dualderivative}, we have
    \begin{align*}
        &h_i(s) \left[\frac{A_i}{\lambda} - \frac{2 \gamma}{\lambda} \sum_{a_i} P_{t}(s' \mid s) + \frac{\gamma^2}{\lambda} \sum_{s', a_i} P_{t}^2(s \mid s') \right] \\
        &+\sum_{\widetilde{s} \neq s} h_i(\widetilde{s}) \left[ -\frac{\gamma}{\lambda} \sum_{a_i} P_{t}(\widetilde{s} \mid s) - \frac{\gamma}{\lambda} \sum_{a_i} P_{t}(s \mid \widetilde{s}) + \frac{\gamma^2}{\lambda} \sum_{s', a_i} P_{t}(\widetilde{s} \mid s') P_{t}(s \mid s') \right] \\
        &=\frac{1}{\lambda} \sum_{a_i} \partial_{x_{i,s,a_i}} \Phi_t^t(\rho) - \rho(s) - \frac{\gamma}{\lambda} \sum_{s', a_i} \partial_{x_{i,s,a_i}} \Phi_t^t(\rho)P_{t}(s \mid s') .
    \end{align*}
    For each $i \in \mathcal{N}$, define $B_i \in \mathbb{R}^{S \times S}$ and $b_i \in \mathbb{R}^{S}$, where
    \begin{equation*}
        B_i(s, \widetilde{s}_i) =
        \begin{cases}
            \frac{A_i}{\lambda} - \frac{2 \gamma}{\lambda} \sum_{a_i} P_{t}(s' \mid s, a_i) + \frac{\gamma^2}{\lambda} \sum_{s', a_i} P_{t}^2(s \mid s', a_i) & s = \widetilde{s} \\
            -\frac{\gamma}{\lambda} \sum_{a_i} P_{t}(\widetilde{s} \mid s, a_i) - \frac{\gamma}{\lambda} \sum_{a_i} P_{t}(\widetilde{s}, a_i, s) + \frac{\gamma^2}{\lambda} \sum_{s', a_i} P_{t}(\widetilde{s} \mid s', a_i) P_{t}(s \mid s', a_i) & \mathrm{s \neq \widetilde{s}}
        \end{cases}
    \end{equation*}
    and
    \begin{equation*}
        b_i(s) = \frac{1}{\lambda} \sum_{a_i} \partial_{\mu_i, s, a_i} \Phi_t^t(\rho) - \rho(s) - \frac{\gamma}{\lambda} \sum_{s', a_i} \partial_{\mu_i, s, a_i} \Phi_t^t(\rho)P_{t}(s \mid s', a_i).
    \end{equation*}
    Then, the optimal solution of the system is the solution of $B_i h_i = b_i$ for all agents $i \in \mathcal{N}$.
    For each $i \in \mathcal{N}$ and $a_i \in \mathcal{A}_i$, define $M_{i, a_i} \in \mathbb{R}^{S \times S}$ such that $M_{t, i, a_i}(s, s') = P_{t} (s' \mid s, a_i)$ for all $s, \widetilde{s} \in \mathcal{S}_i$.
    Let $I_i$ denote an $S \times S$ identity matrix.
    We have
    \begin{equation*}
        B_i = \frac{A_i}{\lambda} I_i - \frac{\gamma}{\lambda} \sum_{a_i} \left(M_{t, i, a_i} + M_{t, i, a_i}^T\right) + \frac{\gamma^2}{\lambda} M_{t, i, a_i}^TM_{t, i, a_i} \leq \frac{A_i (1-\gamma)^2}{\lambda} I_i,
    \end{equation*}
    where the inequality follows from \citet[Lemma~5]{PerformativeReinforcementLearning}.
    Moreover,
    \begin{align*}
        &\left(\frac{1}{\lambda} \sum_{a_i} \partial_{\mu_i, s, a_i} \Phi_t^t(\rho) - \rho(s) - \frac{\gamma}{\lambda} \sum_{s', a_i} \partial_{\mu_i, s, a_i} \Phi_t^t(\rho)P_t(s \mid s', a_i) \right)^2 \\
        &\le \frac{3}{\lambda^2} \left( \sum_{a_i} \partial_{\mu_i, s, a_i} \Phi_t^t(\rho) \right)^2 + 3\rho^2(s) + \frac{3\gamma^2}{\lambda} \left(\sum_{s', a_i} \partial_{\mu_i, s, a_i} \Phi_t^t(\rho) P_t(s \mid s', a_i) \right)^2\\
        &\le \frac{3}{\lambda^2} A_i^2 \left\lVert \nabla\Phi_t^t(\rho) \right\rVert_\infty^2 + 3 \rho^2(s) + \frac{3\gamma^2}{\lambda^2} S \left(\sum_{a_i} \left( \partial_{\mu_i, s, a_i} \Phi_t^t(\rho) \right)^2 \right)\left(\sum_{s', a_i} \left(P_t(s \mid s', a_i)\right)^2 \right) \\
        &\le \frac{3}{\lambda^2} A^2 \left\lVert \nabla\Phi_t^t(\rho) \right\rVert_\infty^2 + 3 \rho^2(s) + \frac{3\gamma^2}{\lambda^2} S \left(\sum_{a_i} \left( \partial_{\mu_i, s, a_i} \Phi_t^t(\rho) \right)^2 \right)\left(\sum_{s', a_i} \left(P_t(s \mid s', a_i)\right)^2 \right) .
    \end{align*}
    Thus, we obtain
    \begin{align*}
        &\left\lVert h \right\rVert_2^2 \le \sum_{i} \frac{\left\lVert b_i \right\rVert_2^2}{\left(\lambda_{\min}(B_i) \right)^2} \\
        &\le \frac{\lambda^2}{(1 - \gamma)^4 A_{\max}^2} \sum_{i, s} \left( \frac{3}{\lambda^2} A_i^2 \left\lVert \nabla\Phi_t^t(\rho) \right\rVert_\infty^2 + 3 \rho^2(s) \right) \\
        &\quad+ \frac{\lambda^2}{(1 - \gamma)^4 A_{\max}^2} \sum_{i, s} \frac{3\gamma^2}{\lambda^2} S \left(\sum_{a_i} \left( \partial_{\mu_i, s, a_i} \Phi_t^t(\rho) \right)^2 \right)\left(\sum_{s', a_i} \left(P_t(s \mid s', a_i)\right)^2 \right) \\
        &\le \frac{\lambda^2}{(1 - \gamma)^4 A_{\max}^2} \left(\frac{3}{\lambda^2} n S A_{\max}^2 \left\lVert \nabla\Phi_t^t(\rho) \right\rVert_\infty^2 + 3 + \frac{3\gamma^2}{\lambda^2} SA_{\max} \left\lVert \nabla\Phi_t^t(\rho) \right\rVert_\infty^2 \right) .
    \end{align*}
    For $\lambda < \sqrt{n S} A_{\max} \left\lVert \nabla\Phi_t^t(\rho) \right\rVert_\infty$, we can further simplify the bound to obtain
    \begin{equation*}
        \left\lVert h \right\rVert_2^2 \le \frac{\lambda^2}{(1 - \gamma)^4 A_{\max}^2} \frac{9n S A_{\max}^2}{\lambda^2} \left\lVert \nabla\Phi_t^t(\rho) \right\rVert_\infty^2 \le \frac{9n S}{(1 - \gamma)^4} \left\lVert \nabla\Phi_t^t(\rho) \right\rVert_\infty^2
        \qedhere
    \end{equation*}
\end{proof}

\begin{lemma}\label{Lemma: gradient bound}
    Let $M$ and $\widehat{M}$ be two different underlying MPGs associated with state-action occupancy measures, $\mu$ and $\widehat{\mu}$, given policy $\widehat{\pi}$, it holds that
    \begin{equation*}
    \left\lVert \nabla \Phi^{\Tilde{\pi}}_M - \nabla \Phi^{\Tilde{\pi}}_{\widehat{M}} \right\rVert_2 \leq \lVert \mu - \widehat{\mu} \rVert_2 \cdot \left( \frac{\sqrt{n A_{\max}}}{(1-\gamma)^2}\zeta_r + 2\frac{\sqrt{S n  A_{\max}} \zeta_p \gamma}{(1-\gamma)^{3}} \right) .
    \end{equation*}
\end{lemma}

\begin{proof}
    By the policy gradient theorem \citep{sutton_pg99}, for an agent $i \in \mathcal{N}$ policy $\pi \in \Pi$:
    \begin{align*}
        \frac{\partial \Phi^{\Tilde{\pi}}_{\pi}}{\partial \pi_i(a_i | s)} = \frac{1}{1 - \gamma} d_{\pi}^{\Tilde{\pi}}(s) \Bar{Q}_{i, \pi}^{\Tilde{\pi}}(s, a_i).
    \end{align*}
    Then,
    \begin{align*}
    &\left\lVert \nabla \Phi_M^{\Tilde{\pi}} - \nabla \Phi_{\widehat{M}}^{\Tilde{\pi}} \right\rVert_2^2 \\
    &=
    \frac{1}{(1 - \gamma)^2} \sum_i \sum_{s, a_i} \left( d_{\rho, \pi}^{\Tilde{\pi}}(s) \Bar{Q}_{i, \pi}^{\Tilde{\pi}}(s, a_i) -  d_{\rho, \widehat{\pi}}^{\Tilde{\pi}}(s) \Bar{Q}_{i, \widehat{\pi}}^{\Tilde{\pi}}(s, a_i) \right)^2 \\
    &\leq
    \frac{1}{(1 - \gamma)^2} \sum_i \sum_{s, a_i} 
    \left( \left\lvert d_{\rho, \pi}^{\Tilde{\pi}}(s) - d_{\rho, \widehat{\pi}}^{\Tilde{\pi}}(s) \right\rvert \Bar{Q}_{i, \pi}^{\Tilde{\pi}}(s, a_i) 
    + d_{\rho, \widehat{\pi}}^{\Tilde{\pi}}(s) \left\lvert \Bar{Q}_{i, \pi}^{\Tilde{\pi}}(s, a_i) - \Bar{Q}_{i, \widehat{\pi}}^{\Tilde{\pi}}(s, a_i) \right\rvert \right)^2 \\
    &\leq
    \frac{1}{(1 - \gamma)^2} \sum_i \sum_{s, a_i} 
    \left( \left\lvert d_{\rho, \pi}^{\Tilde{\pi}}(s) - d_{\rho, \widehat{\pi}}^{\Tilde{\pi}}(s) \right\rvert \Bar{Q}_{i, \pi}^{\Tilde{\pi}}(s, a_i) 
    + d_{\rho, \widehat{\pi}}^{\Tilde{\pi}}(s) \left\lvert \Bar{Q}_{i, \pi}^{\Tilde{\pi}}(s, a_i) - \Bar{Q}_{i, \widehat{\pi}}^{\Tilde{\pi}}(s, a_i) \right\rvert \right)^2 \\
    &\leq
    \frac{2}{(1 - \gamma)^2} \sum_i \sum_{s, a_i} 
    \left[ \left\lvert d_{\rho, \pi}^{\Tilde{\pi}}(s) - d_{\rho, \widehat{\pi}}^{\Tilde{\pi}}(s) \right\rvert \Bar{Q}_{i, \pi}^{\Tilde{\pi}}(s, a_i) \right]^2
    + \left[ d_{\rho, \widehat{\pi}}^{\Tilde{\pi}}(s) \left\lvert \Bar{Q}_{i, \pi}^{\Tilde{\pi}}(s, a_i) - \Bar{Q}_{i, \widehat{\pi}}^{\Tilde{\pi}}(s, a_i) \right\rvert \right]^2 \\
    &\leq
    \frac{2}{(1 - \gamma)^4} \sum_i \sum_{s, a_i} 
    \left[ d_{\rho, \pi}^{\Tilde{\pi}}(s) - d_{\rho, \widehat{\pi}}^{\Tilde{\pi}}(s)  \right]^2
    + 
    \frac{2}{(1 - \gamma)^2} \sum_i \sum_{s, a_i} \left[ d_{\rho, \widehat{\pi}}^{\Tilde{\pi}}(s) \left\lvert \Bar{Q}_{i, \pi}^{\Tilde{\pi}}(s, a_i) - \Bar{Q}_{i, \widehat{\pi}}^{\Tilde{\pi}}(s, a_i) \right\rvert \right]^2 \\
    &\lesssim
    \lVert \mu - \widehat{\mu} \rVert_2 \cdot \left( \frac{ S n A_{\max} \cdot \zeta_p^2}{(1 - \gamma)^4} + \left[ \frac{\sqrt{n A_{\max}}}{(1 - \gamma)^2}\left( \zeta_r + \frac{\gamma \cdot \zeta_p \sqrt{S}}{1 - \gamma} \right) \right]^2 \right).
\end{align*}
    For the last inequality to hold, we exploit the following bounds: 
    first, we use that
    \begin{align*}
        \sum_{i, s, a_i} ( d_{\rho, \pi}^{\Tilde{\pi}}(s) - d_{\rho, \widehat{\pi}}^{\Tilde{\pi}}(s) )^2
        &\leq
        (1-\gamma)^2 \cdot n A_{\max} \cdot \lVert \Tilde{\mu}_\pi - \Tilde{\mu}_{\widehat{\pi}}
        \Vert_2^2
        \\ &\leq
        3 S n A_{\max} \cdot \zeta^2_p \cdot \lVert \mu - \widehat{\mu} \rVert_2^2,
    \end{align*}
    where the last inequality uses the same computation as in Lemma~\ref{Eq: Smoothness inequality}.  Further, exploiting Lemma~\ref{lemma: Bound noise terms}, we obtain
    \begin{align*}
        \max_{s, a} \left\lvert \Bar{Q}_{i, \pi}^{\Tilde{\pi}}(s, a_i) - \Bar{Q}_{i, \widehat{\pi}}^{\Tilde{\pi}}(s, a_i) \right\rvert \le \frac{1}{1-\gamma} \cdot \Big(\zeta_r +\frac{ \gamma \cdot \zeta_p \sqrt{S}}{1-\gamma} \Big) \cdot \left\lVert \mu - \widehat{\mu} \right\rVert_2.
    \end{align*}
    By combining, we conclude the proof.
\end{proof}

\begin{lemma}\label{Eq: Smoothness inequality}
    It holds that $\lVert \nabla \Phi_\pi^{\pi'} - \Phi_\pi^{\pi''} \rVert_2 
    \lesssim
    \frac{\beta}{\min_{s,\pi} \alpha_\pi(s)}
    \left( 1 + \frac{\zeta_p \sqrt{S}}{1 - \gamma} \right) \lVert \mu'' - \mu' \rVert_2 $.
\end{lemma}

\begin{proof}
First, we use the assumption that $\Phi_\pi^x$ is $\beta$-smooth in $x$, we have
    \begin{align*}
        \lVert \nabla \Phi_\pi^{\pi'} - \nabla \Phi_\pi^{\pi''} \rVert_2 \leq
        \beta \cdot \lVert \pi' - \pi'' \rVert_2.
    \end{align*}
Due to agent-independent transitions, it holds that
\begin{align*}
    \lVert \pi' - \pi'' \rVert_2 
    \leq
    \frac{ \lVert \mu_{\pi'}^{\pi'} -\mu_{\pi'}^{\pi''} \rVert_2}{\min_s \alpha_\pi(s)}
    \leq
    \frac{ \lVert \mu_{\pi'}^{\pi'} -\mu_{\pi''}^{\pi''} \rVert_2 + \lVert \mu_{\pi''}^{\pi''} -\mu_{\pi'}^{\pi''} \rVert_2}{\min_s \alpha_\pi(s)}. 
\end{align*}
It remains to show that $\lVert \mu_{\pi''}^{\pi''} -\mu_{\pi'}^{\pi''} \rVert_2$ is bounded by a fraction of $\lVert \mu_{\pi''}^{\pi''} -\mu_{\pi'}^{\pi''} \rVert_2$. We rewrite the expression explicitly and abuse notation to express the probability that an action/state pair occurs, i.e.,
\begin{align*}
    \lVert \mu_{\pi''}^{\pi''} - \mu_{\pi'}^{\pi''} \rVert_2
    &=
    \sum_{i \in \mathcal{N}} \sum_{s \in \mathcal{S}} \sum_{a_i} \left[ \sum_{t=0}^\infty \left(P_{\pi''}(s^t = s, a^t_i = a_i \mid \pi'') - P_{\pi'}(s_t = s, a^t_i = a_i \mid \pi'')\right) \right]^2 \\
    &\leq
    \frac{1}{(1-\gamma)^2} \sum_{i,s,a_i} \left( \max_t \left(P_{\pi''}(s^t = s, a^t_i = a_i \mid \pi'') - P_{\pi'}(s_t = s, a^t_i = a_i \mid \pi'')\right) \right)^2
    \\ &= 
    \frac{1}{(1-\gamma)^2} \sum_{i,s,a_i} \left( \sum_{s'} \left( \left| P_{\pi''}(s^t = s, a^t_i = a_i \mid \pi'', s') - P_{\pi'}(s_t = s, a^t_i = a_i \mid \pi'', s') \right| \cdot \max_{x \in \{ \pi', \pi'' \}} P_{x}(s_{t-1} = s ) \right) \right)^2
    \\ &\leq
    \frac{3 S}{(1-\gamma)^2} \sum_{i, s, a_i} 
    \sum_i \sum_{s, s', a} | P_{\pi''}(s,a \mid s') - P_{\pi'}(s,a \mid s') |^2
    \\
    &\leq 
    \frac{3S}{(1-\gamma)^2}\lVert P_{\pi''}(\cdot \mid \cdot, \cdot) 
 - P_{\pi'}(\cdot \mid \cdot, \cdot )\rVert_2^2
    \\
    &\leq
    \frac{3S \zeta_p^2}{(1-\gamma)^2} \lVert \mu'' - \mu' \rVert_2^2 ,
\end{align*}
where we especially make use of the Cauchy-Schwarz inequality in the second inequality and the sensitivity assumption in the last inequality.
Putting the inequalities together, it holds that
\begin{equation*}
    \lVert \pi' - \pi'' \rVert_2 
    \lesssim
    \frac{1}{\min_{s,\pi} \alpha_\pi(s)}
    \left( 1 + \frac{\zeta_p \sqrt{S}}{1 - \gamma} \right) \lVert \mu'' - \mu' \rVert_2 \qedhere
\end{equation*}
\end{proof}

The next statement can be proved as Lemma~\ref{lemma: Bound noise terms}, by exchanging the definition of sensitivity. 
\begin{lemma}\label{lemma: Bound noise terms, occupancy measures}
We have for all $\pi, \widehat{\pi}$, given a fixed $\Tilde{\pi}$, it holds that
    \begin{align*}
        \max_{s, a} \left\lvert \Bar{Q}_{i, \pi}^{\Tilde{\pi}}(s, a_i) - \Bar{Q}_{i, \widehat{\pi}}^{\Tilde{\pi}}(s, a_i) \right\rvert \le \frac{1}{1-\gamma} \cdot \Big(\zeta_r +\frac{ \gamma \cdot \zeta_p \sqrt{S}}{1-\gamma} \Big) \cdot \left\lVert \mu - \widehat{\mu} \right\rVert_2.
    \end{align*}
\end{lemma}

\section{Experimental Setup}

This section describes details about the experiments presented in Section~\ref{sec:experiments}.
The code can be found in \href{https://github.com/PauliusSasnauskas/performative-mpgs}{https://github.com/PauliusSasnauskas/performative-mpgs}.

\subsection{Algorithms}\label{app:algorithms}

The policy distance presented in Figures \ref{fig:comp_leo_ding_app}, \ref{fig:comp_cong_leo_ding_app}, \ref{fig:comp_inpg_inpg_reg_app}, \ref{fig:comp_cong_inpg_inpg_reg_app} is the distance from the current policy to the average of the last 10 in that run:
\begin{equation*}
    \text{Policy distance }(t) = \frac{1}{N} \sum_i^N \left\| \pi_i^t - \pi_i^\text{last} \right\| ,
\end{equation*}
where $\pi_i^\text{last}$ is the average of the last 10 policies in that run.

\subsubsection{Independent Projected Gradient Ascent (IPGA)}
We improve upon the code presented by \citet{GlobalConvergence_Leonardos} for the IPGA algorithm, which is shown in Algorithm \ref{alg:pga-practical}.
We set the hyperparameters as shown in Table~\ref{tab:pga-hyp}.
We name the regularized IPGA version in the name of \citet{GlobalConvergence_Leonardos} -- \textit{IPGA-L}, and the unregularized version in the name of \citet{IndepPolicyGrad_Ding} -- \textit{IPGA-D}.
The comparison of performance for different performativity strengths and different learning rates can be seen in Figure~\ref{fig:comp_leo_ding_app} for the safe-distancing game, and in Figure~\ref{fig:comp_cong_leo_ding_app} for the stochastic congestion game.

\begin{table}[ht]
    \centering
    \caption{Hyperparameters used for the IPGA algorithm (unless noted otherwise).}
    \begin{tabular}{l c}
        \textbf{Parameter} & \textbf{Value} \\
        \hline
        Learning rate $\eta$ & 0.0001 \\
        Discount factor $\gamma$ & 0.99 \\
        Number of episodes per round & 20 \\
    \end{tabular}
    \label{tab:pga-hyp}
\end{table}

\begin{algorithm}
\caption{Independent Projected Gradient Ascent Practical Implementation}
\label{alg:pga-practical}
\begin{algorithmic}[1]
    \State{\textbf{Input}:} $\eta$ -- step size, $T$ -- number of rounds, $K$ -- \text{episodes per round} 
    \State{\textbf{Init.}:} for all $i \in \mathcal{N}, a_i \in \mathcal{A}_i , s \in S$, let $\pi^t_i(a_i | s ) = 1/|\mathcal{A}_i| $.
    \For{$t=1$ to $T$}
        \State Roll out policies $\pi^t$ for $K$ episodes to get trajectories $\tau$
        \State Compute state visitation distribution $\mu(s)$ from trajectories $\tau$
        \State Compute state-action value function $Q_i(s, a)$ from trajectories $\tau$ for each agent $i \in \mathcal{N}$
        \For{$i=1$ to $n$ (simultaneously)}
            \State Compute $g_i(s, a) = \mu(s) \, Q_i(s, a) \quad \forall s \in S, a \in \mathcal{A}_i$ \color{green!50!black} \Comment{Set $\mu(s) = 1 \; \forall s \in S$ for \textit{IPGA-D} version} \color{black}
            \State Update $\pi_i^{t+1}(a | s) = \text{Proj}_{\Delta_{\mathcal{A}_i}} (\pi_i^t(a | s) + \eta g(s, a)) \quad \forall s \in S, a \in \mathcal{A}_i$
        \EndFor
    \EndFor
\end{algorithmic}
\end{algorithm}

\begin{figure*}[ht]
    \centering
    \vspace{.3in}
    \includegraphics[width=0.4526\textwidth]{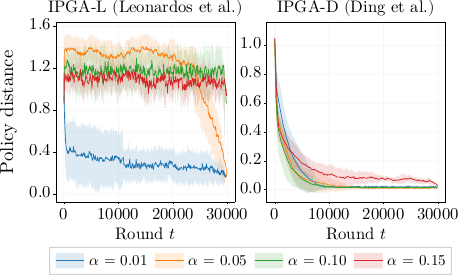}
    \hspace{1.5em}
    \includegraphics[width=0.4526\textwidth]{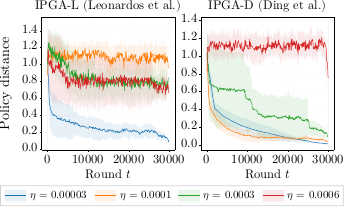}
    \vspace{.3in}
    \caption{
    Comparison of IPGA-L and IPGA-D in the safe-distancing game varying the performativity strength $\alpha$ (left two plots, $\eta = 0.0001$), and learning rate $\eta$ (right two plots, $\alpha = 0.15$). Mean and standard deviation over $10$ experiment replications.
    }
    \label{fig:comp_leo_ding_app}
\end{figure*}

\begin{figure*}[ht]
    \centering
    \vspace{.3in}
    \includegraphics[width=0.4526\textwidth]{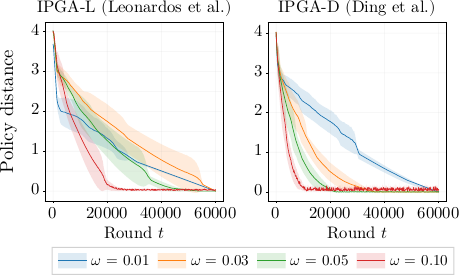}
    \hspace{1.5em}
    \includegraphics[width=0.4526\textwidth]{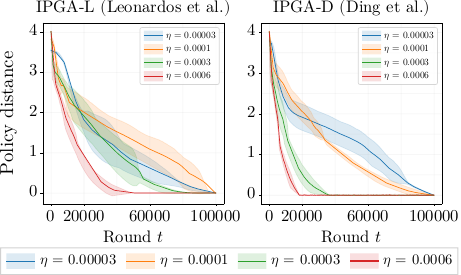}
    \vspace{.3in}
    \caption{Comparison of IPGA-L and IPGA-D in the stochastic congestion game varying the performativity strength $\perfc = \perfc_r = \perfc_p$ (left two plots, $\eta = 0.0003$), and learning rate $\eta$ (right two plots, $\perfc = 0.03$). Mean and standard deviation over $5$ experiment replications.}
    \label{fig:comp_cong_leo_ding_app}
\end{figure*}


\subsubsection{Independent Natural Policy Gradient (INPG)}

\begin{figure*}[ht]
    \centering
    \vspace{.3in}
    \includegraphics[width=0.4526\textwidth]{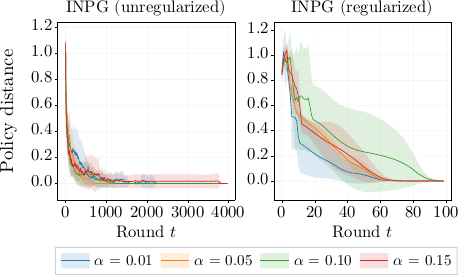}
    \hspace{0.5em}
    \includegraphics[width=0.4526\textwidth]{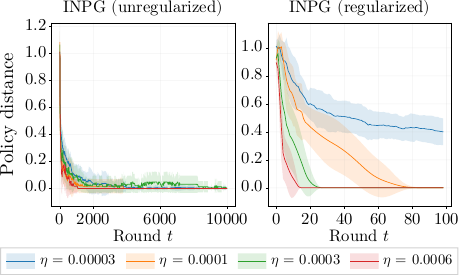}
    \vspace{.3in}
    \caption{Comparison of INPG regularized vs.\ unregularized version in the safe-distancing game varying the performativity strength $\alpha$ (left two plots, $\eta = 0.0001$) and learning rate $\eta$ (right two plots, $\alpha = 0.15$). Mean and standard deviation over $10$ experiment replications. In the INPG (regularized) version with $\eta = 0.00003$ (blue line) in the rightmost plot converges after approx.\ 3000 rounds (not seen in the plot). (Note the stark difference in the number of rounds.)}
    \label{fig:comp_inpg_inpg_reg_app}
\end{figure*}

\begin{figure*}[htb]
    \centering
    \vspace{.3in}
    \includegraphics[width=0.4526\textwidth]{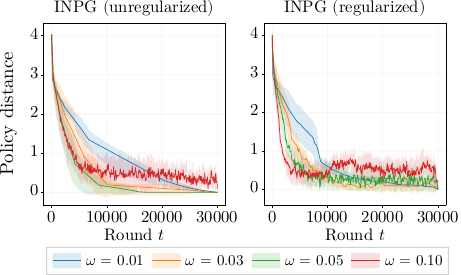}
    \hspace{0.5em}
    \includegraphics[width=0.4526\textwidth]{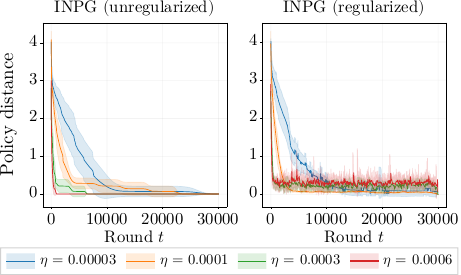}
    \vspace{.3in}
    \caption{Comparison of INPG regularized vs.\ unregularized version in the stochastic congestion game varying the performativity strength $\perfc = \perfc_r = \perfc_p$ (left two plots, $\eta = 0.00003$) and learning rate $\eta$ (right two plots, $\perfc = 0.03$). Mean and standard deviation over $5$ experiment replications.}
    \label{fig:comp_cong_inpg_inpg_reg_app}
\end{figure*}

We implement our own version of the INPG algorithm shown in Algorithm \ref{alg:inpg-practical}, based on the variant by \citet{NPG_converges}.
We name this algorithm \textit{INPG (unregularized)} in the plots.
We name the log-barrier regularized version \textit{INPG (regularized)}.
We set the hyperparameters as shown in Table~\ref{tab:inpg-hyp}.

The comparison of performance for different performativity strengths and different learning rates in the safe-distancing game can be seen in Figure~\ref{fig:comp_inpg_inpg_reg_app}, in the stochastic congestion game in Figure~\ref{fig:comp_cong_inpg_inpg_reg_app}.

\begin{table}[ht]
    \centering
    \caption{Hyperparameters used for the INPG algorithm (unless noted otherwise).}
    \begin{tabular}{l c}
        \textbf{Parameter} & \textbf{Value} \\
        \hline
        Learning rate $\eta$ & 0.0001 \\
        Discount factor $\gamma$ & 0.99 \\
        Number of episodes per round & 20 \\
        Regularizer strength $\lambda$ (only in regularized version) & 0.003 \\
    \end{tabular}
    \label{tab:inpg-hyp}
\end{table}

\begin{algorithm}
\caption{Independent Natural Policy Gradient Practical Implementation}
\label{alg:inpg-practical}
\begin{algorithmic}[1]
    \State{\textbf{Input}:} $\eta$ -- step size, $T$ -- number of rounds, $K$ -- \text{episodes per round} 
    \State{\textbf{Init.}:} for all $i \in \mathcal{N}, a_i \in \mathcal{A}_i , s \in S$, let $\pi^t_i(a_i | s ) = 1/|\mathcal{A}_i| $.
    \For{$t=1$ to $T$}
        \State Roll out policies $\pi^t$ for $K$ episodes to get trajectories $\tau$
        \State Compute state visitation distribution $\mu(s)$ from trajectories $\tau$
        \State Compute value functions $V_i(s)$ and $Q_i(s, a)$ from trajectories $\tau$ for each agent $i \in \mathcal{N}$
        \For{$i=1$ to $n$ (simultaneously)}
            \State Compute $A_i(s, a) = Q_i(s, a) - V_i(s) \quad \forall s \in S, a \in \mathcal{A}_i$
            \State Update $\pi_i^{t+1}(a | s) = \pi_i^t(a | s) \, \text{exp}\left( \frac{\eta}{1-\gamma}A_i(s, a) + \frac{\lambda}{\mu(s) \pi^t_i(a | s)} - \frac{\lambda |\mathcal{A}_i|}{\mu(s)} \right) \frac{1}{Z}
            \quad \forall s \in S, a \in \mathcal{A}_i$ \item[] \Comment{$Z$ is the renormalization term. For \textit{INPG (unregularized)} we set $\lambda = 0$.}
        \EndFor
    \EndFor
\end{algorithmic}
\end{algorithm}

\begin{figure}[htb]
    \centering
    \vspace{.3in}
    \includegraphics[width=0.45\linewidth]{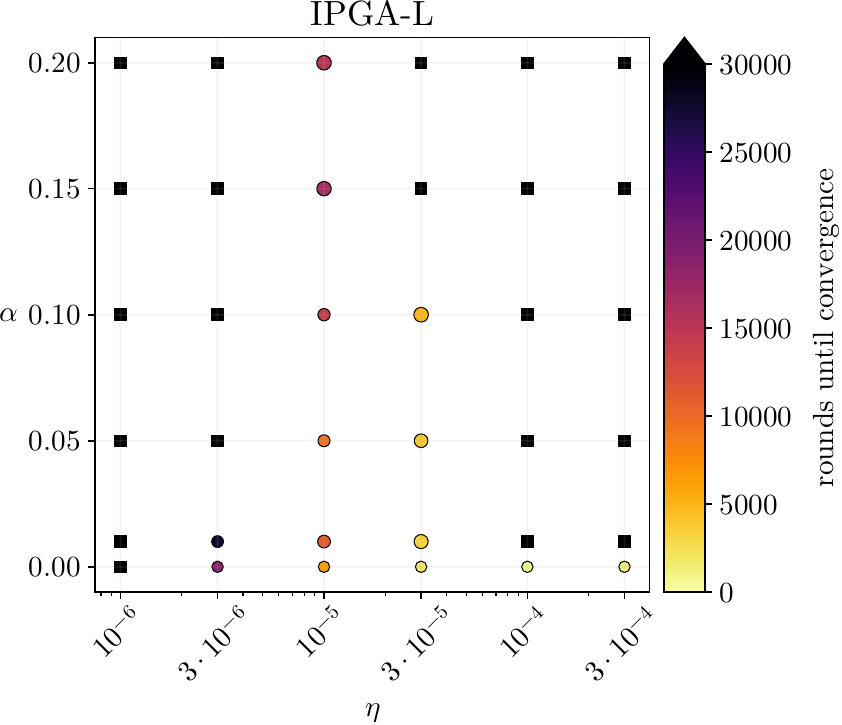}
    \hspace{0.5em}
    \includegraphics[width=0.45\linewidth]{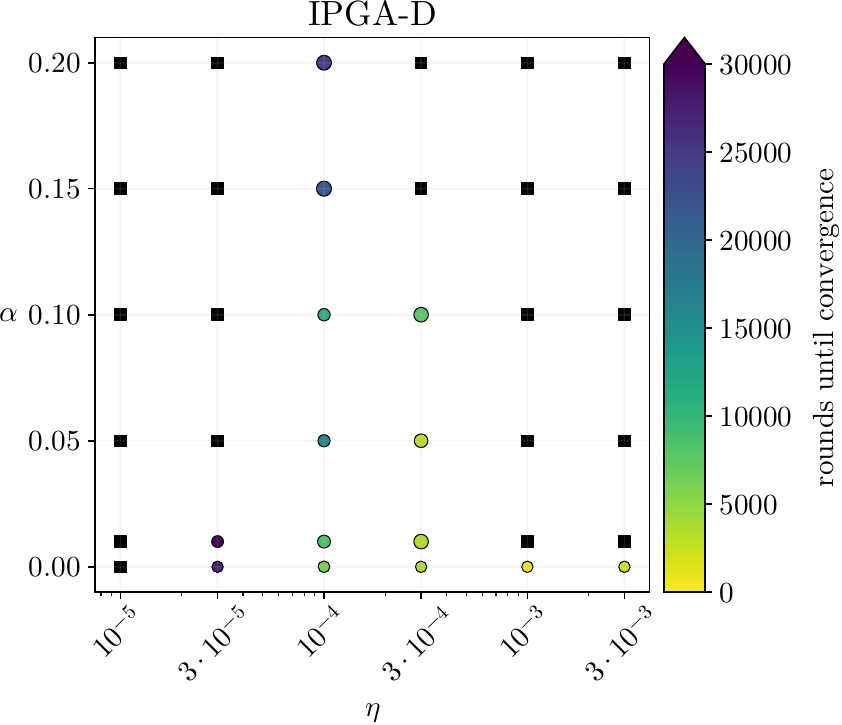}   
    \vspace{.3in}
    \caption{
    Comparison of different values of $\alpha$ (linear scale) and $\eta$ (log scale) for IPGA-L and IPGA-D in the safe-distancing environment.
    Circles indicate the runs converged in under 30000 rounds with the number of rounds indicated by the color scale shown on the right.
    Black squares indicate the runs did not converge in 30000 rounds.
    Values shown are from the mean over 10 runs.
    }
    \label{fig:omegas_lrs}
\end{figure}

Figure~\ref{fig:omegas_lrs} shows the importance of selecting an appropriate learning rate for the PGA algorithms.
Some learning rates are more stable for a larger set of performativity strengths $\alpha$, as seen in the plot ($\eta = 0.00001$ for IPGA-L and $\eta = 0.0001$ for IPGA-D).

\clearpage
\subsection{Environments}\label{app:environments}

\subsubsection{Safe-Distancing Game}\label{app:safe-dist-env}

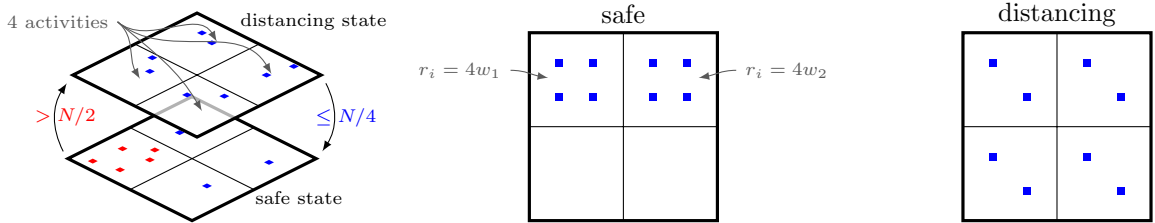
\begin{figure}[hbt]
    \centering
    \vspace{.3in}
    \begin{subfigure}[b]{0.33\textwidth}
        \centering
        \begin{tikzpicture}[scale=0.33,every node/.style={minimum size=0cm}, >={}]
             \begin{scope}[yshift=-63,every node/.append style={yslant=0.5,xslant=-1},yslant=0.5,xslant=-1]
            
                   \draw[step=25mm, black] (0,0) grid (5,5);
                   \draw[black,very thick] (0,0) rectangle (5,5);
                   \fill[red] (2,3.5) rectangle ++(0.2,0.2); 
                   \fill[blue] (1.6,1) rectangle ++(0.2,0.2); 
                   \fill[red] (1.5,3.2) rectangle ++(0.2,0.2); 
                   \fill[red] (0.3,4.3) rectangle ++(0.2,0.2); 
                   \fill[red] (0.5,3.4) rectangle ++(0.2,0.2); 
                   \fill[blue] (3.8,0.7) rectangle ++(0.2,0.2); 
                   \fill[red] (1.4,4.2) rectangle ++(0.2,0.2); 
                   \fill[blue] (3.2,3.7) rectangle ++(0.2,0.2); 
            \end{scope}
            \begin{scope}[yshift=30,every node/.append style={yslant=0.5,xslant=-1},yslant=0.5,xslant=-1]
                   \fill[white,fill opacity=0.7] (0.2,0) rectangle (5.2,5);
                   \draw[black,very thick] (0.2,0) rectangle (5.2,5);
                   \draw[step=25mm, black, xshift=.2cm] (0,0) grid (5,5);
                   \fill[blue] (1.6,1.8) rectangle ++(0.2,0.2); 
                   \fill[blue] (2.3,1) rectangle ++(0.2,0.2); 
                   \fill[blue] (4,1) rectangle ++(0.2,0.2);
                   \fill[blue] (4.2,3.4) rectangle ++(0.2,0.2); 
                   \fill[blue] (4.9,0.8) rectangle ++(0.2,0.2); 
                   \fill[blue] (4.4,4) rectangle ++(0.2,0.2);
                   \fill[blue] (2.4,4) rectangle ++(0.2,0.2);
                   \fill[blue] (1.8,3.5) rectangle ++(0.2,0.2);
               \end{scope}
            \node at (5, 0.2)  (a) {};
            \node at (-5, 0.2)  (b) {};
            \node at (-4.8, 3.7) (c) {};
            \node at (5.2, 3.7)  (d) {};
            
            \draw[-latex] (d) edge[bend left] node [midway,left,blue,very thick, xshift=0.7cm] {\scriptsize $\le N/4$} (a);
            \draw[-latex] (b) edge[bend left] node [midway,right,red,thick,xshift=-0.4cm] {\scriptsize $>N/2$} (c);
            
            \draw[-latex,gray!70!black](-3,5.8)node[left]{\scriptsize 4 activities}
                   to[out=-40,in=120] (0.4,2.2);
            \draw[-latex,gray!70!black](-3,5.8)node[left]{}
                   to[out=-40,in=120] (1,5);
            \draw[-latex,gray!70!black](-3,5.8)node[left]{}
                   to[out=-40,in=120] (3.2,3.7);
            \draw[-latex,gray!70!black](-3,5.8)node[left]{ }
                   to[out=-40,in=120] (-2.1,3.5);
            \draw[thick](4.9,5.8) node {\scriptsize distancing state};
            \draw[thick](4.3,-1.3) node {\scriptsize safe state};
        \end{tikzpicture}
    \end{subfigure}%
    \begin{subfigure}[b]{0.33\textwidth}
        \centering
        \begin{tikzpicture}[scale=0.5,every node/.style={minimum size=0cm}, >={}]
            \node at (2.5, 5.5)  (title) {safe};
            
            \draw[step=25mm, black] (0,0) grid (5,5);
            \draw[black,very thick] (0,0) rectangle (5,5);
            \fill[blue] (0.7,4.1) rectangle ++(0.2,0.2);
            \fill[blue] (0.7,3.2) rectangle ++(0.2,0.2);
            \fill[blue] (1.6,4.1) rectangle ++(0.2,0.2);
            \fill[blue] (1.6,3.2) rectangle ++(0.2,0.2);
            \fill[blue] (4.1,4.1) rectangle ++(0.2,0.2);
            \fill[blue] (3.2,4.1) rectangle ++(0.2,0.2);
            \fill[blue] (4.1,3.2) rectangle ++(0.2,0.2);
            \fill[blue] (3.2,3.2) rectangle ++(0.2,0.2);

            \draw[-latex,gray!70!black](-0.5,4) node[left] {\scriptsize $r_i = 4 w_1$} to[out=0,in=170] (0.5,3.8);
            \draw[-latex,gray!70!black](5.5,4) node[right] {\scriptsize $r_i = 4 w_2$} to[out=-180,in=10] (4.5,3.8);


        \end{tikzpicture}
    \end{subfigure}
    \begin{subfigure}[b]{0.33\textwidth}
        \centering
        \begin{tikzpicture}[scale=0.5,every node/.style={minimum size=0cm}, >={}]
            \node at (2.5, 5.5)  (title) {distancing};
            
            \draw[step=25mm, black] (0,0) grid (5,5);
            \draw[black,very thick] (0,0) rectangle (5,5);
            \fill[blue] (0.7,4.1) rectangle ++(0.2,0.2);
            \fill[blue] (1.6,3.2) rectangle ++(0.2,0.2);
            \fill[blue] (4.1,3.2) rectangle ++(0.2,0.2);
            \fill[blue] (3.2,4.1) rectangle ++(0.2,0.2);

            \fill[blue] (0.7,1.6) rectangle ++(0.2,0.2);
            \fill[blue] (1.6,0.7) rectangle ++(0.2,0.2);
            \fill[blue] (3.2,1.6) rectangle ++(0.2,0.2);
            \fill[blue] (4.1,0.7) rectangle ++(0.2,0.2);

        \end{tikzpicture}
    \end{subfigure}
    \vspace{.3in}
    \caption{An illustration of the Safe-Distancing environment that is used in the experiments.
    \textbf{Left:} illustration of the two states. When more than $\frac{N}{2} = 4$ agents in the \textit{safe} state are performing the same activity (5 red squares) all agents are transitioned to the \textit{distancing} state, where they have to spread out, and no more than $\frac{N}{4} = 2$ agents may perform an activity to transition back to the \textit{safe} state.
    \textbf{Center:} illustration of the optimal joint policy in the \textit{safe} state ($w_1$ and $w_2$ being the two highest weighted rewards).
    \textbf{Right:} illustration of the optimal joint policy in the \textit{distancing} state.
    Figure adapted from \citet{GlobalConvergence_Leonardos}.}
    \label{fig:safe-dist-env}
\end{figure}

We use one environment setup defined by \citet{GlobalConvergence_Leonardos} -- the safe-distancing game.
We consider an MDP with two states, one state is called \textit{safe}, the other -- \textit{distancing}.
There are $N=8$ agents, $|\mathcal{A}_i| = 4$ activities the agents can perform.
In both states the reward each agent receives for performing activity $a_i = k$ is equal to a weight $w_k$ multiplied with the number of agents performing that activity.
The weights satisfy $w_1 > w_2 > w_3 > w_4$, i.e., activity $1$ is the most preferable.
If more than $\frac{N}{2} = 4$ agents are performing the same activity, all agents are transitioned to the \textit{distancing} state.
At the \textit{distancing} state the reward weights are the same, except the reward is reduced by a (considerably large) constant $c=100$.
To transition back to the \textit{safe} state the agents have to distribute themselves evenly among the activities, i.e., no more than $\frac{N}{4} = 2$ agents may perform the same activity.
A visualization of the game and example policies can be seen in Figure~\ref{fig:safe-dist-env}.
In our experiments we set $(w_1, w_2, w_3, w_4) = (4, 3, 2, 1)$.

\paragraph{Performative Effect.}
To model the performative response we modify the environment by taking inspiration from \citet{PerformativeReinforcementLearning}, and do as follows.
Each agent is controlled by a principal agent (the learning algorithm), and an influencer agent.
The influencer agent may override some of the actions taken by the principal agent.
Therefore, the principal agent's effective environment is performative.

For example, the principal agent selects one of $|\mathcal{A}_i| = 4$ actions.
The influencer agent may choose to keep this action the same, or intervene, by overriding the action and choosing a different activity.
The influencer agent maintains $Q$-values of $|\Acal_i| + 1 = 5$ actions -- four for intervening by changing into one of $|\mathcal{A}_i| = 4$ actions, and an additional one for no intervention.
Parameter $\alpha$ controls the performative strength.
With a probability of $1-\alpha$ the original principal agent action is selected, and with probability $\alpha$ (e.g., $\alpha = 0.15$) the influencer agent action gets activated.
Its $Q$-values are computed on a perturbed environment, as described by \citet{PerformativeReinforcementLearning}.
The action selected by the influencer agent is sampled from:
\begin{equation*}
    \pi_2(a_i|s) = \frac{\text{exp}(Q^{*|\pi_1}(s|a_i)}{\sum_j \text{exp}(Q^{*|\pi_1}(s, a_j))} .
\end{equation*}

In our experiments, we use the default 
$\alpha = 0.15$, unless noted otherwise.

\subsubsection{Stochastic Congestion Game}\label{app:cong-game}

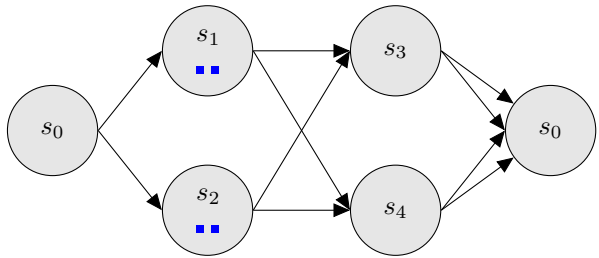
\begin{figure}[hbt]
    \centering
    \vspace{.3in}
    \begin{tikzpicture}[node distance=1.5cm, >=triangle 45]
        \tikzstyle{every node}=[circle, draw, fill=gray!20, minimum size=12mm]
        
        \node (s) at (0,0) {$s_0$};
        \node (a) [above right of=s, xshift=1cm, text width=0.3cm, align=center] {$s_1$\\\ };
        \node (b) [below right of=s, xshift=1cm, text width=0.3cm, align=center] {$s_2$\\\ };
        \node (c) [right of=a, xshift=1cm] {$s_3$};
        \node (d) [right of=b, xshift=1cm] {$s_4$};
        \node (t) [below right of=c, xshift=1cm] {$s_0$};
        
        \fill[blue] (a) ++(-0.15,-0.3) rectangle ++(0.1,0.1);
        \fill[blue] (a) ++(0.05,-0.3) rectangle ++(0.1,0.1);
        
        \fill[blue] (b) ++(-0.15,-0.3) rectangle ++(0.1,0.1);
        \fill[blue] (b) ++(0.05,-0.3) rectangle ++(0.1,0.1);
        
        \draw[->] (s.east) -- (a.west);
        \draw[->] (s.east) -- (b.west);
        \draw[->] (a.east) -- (c.west);
        \draw[->] (a.east) -- (d.west);
        \draw[->] (b.east) -- (c.west);
        \draw[->] (b.east) -- (d.west);
        \draw[->] (c.east) -- (t.west);
        \draw[->] (c.east) -- (t);
        \draw[->] (d.east) -- (t.west);
        \draw[->] (d.east) -- (t);

        
    \end{tikzpicture}
    \vspace{.3in}
    \caption{An illustration of the stochastic congestion game. From states $s_3$ and $s_4$ the game wraps around and state $s_0$ is reached. The agents (blue squares) in state $s_1$ can choose the same action (e.g., the edge to $s_3$), yielding a reward of $15$ each, or can choose different actions (edges to $s_3$ and $s_4$), yielding a reward of $50$ each.
    Figure adapted from \citet{NPG_converges}.}
    \label{fig:cong-game}
\end{figure}

We use the experiment setup defined by \citet{NPG_converges} -- the stochastic congestion game.
We consider an MDP with 5 states, $N=4$ agents.
The states and actions transition as shown in Figure~\ref{fig:cong-game}.
In every state each agent can perform one of $|\mathcal{A}_i| = 2$ actions.
The reward received by each agent is based on the number of agents choosing the same action at that same state.
In our experiments, if only one agent is choosing that action, the reward is $r_i = 50$, if two agents are choosing the same action $r_i = 15$, three $r_i = 5$, four $r_i = 1$.
The agents start at state $s_0$.

\paragraph{Performative Effect.}
To model the performative effect in this environment, we change the rewards and transition probabilities as follows:
\begin{equation}
    r_{i,\pi'} = r_{i,\pi_0} + \frac{\perfc_r}{(1-\gamma) \sqrt{|\mathcal{S}|\,|\mathcal{A}_i|}} (\pi' - \pi_0)\,,
\end{equation}
\begin{equation}
    P_{i,\pi'} = P_{i,\pi_0} + \frac{\perfc_p}{(1-\gamma) \sqrt{|\mathcal{S}|\,|\mathcal{A}_i|}} (\pi' - \pi_0) \frac{1}{|\mathcal{S}|}\,,
\end{equation}
varying the strength via $\perfc_r$ and $\perfc_p$.
We set $\pi_0$ to the initial uniform-random policy.

We restrict the changes the transition kernel in $P_{i, \pi'}$ to be valid for the game, (for example, following the naming in Figure \ref{fig:cong-game}, in state $s_2$ the kernel is restricted to only transition to $s_3$ or to $s_4$, and not to, for example, $s_0$).

In our experiments, we use the defaults $\perfc_r = \perfc_p = 0.03$, unless noted otherwise.


\subsection{Computing Infrastructure}

We ran the experiments on an internal computing cluster with NVIDIA A100 80 GB GPUs.
Running 10 experiment replications in parallel, a single round in the safe-distancing environment takes approx.~0.7~s, 10000 rounds -- approx.~2~h.
Running 5 experiment replications in parallel, a single round in the stochastic congestion game takes approx.~1.6~s, 10000 rounds -- approx.~4.5~h.

\section{Additional Related Work}\label{Additional Related work}

\paragraph{Performative Prediction.}
Since the seminal work of Performative Prediction \citep{DBLP:conf/icml/PerdomoZMH20}, various adaptations has been studied, we refer to the survey by \citet{PP_PF}.
There are variations considering stochastic optimization \citep{NEURIPS2020_PPSO} for finding performatively stable points.
Performative power -- a notion that measures the potential of a firm to influence the population distribution of participants on an online platform \citep{DBLP:conf/nips/HardtJM22}, performative prediction with neural networks \citep{DBLP:conf/aistats/MofakhamiMG23} considers a setting, which allows weaker assumptions on the loss function, regret minimization with performative feedback \citep{DBLP:conf/icml/JagadeesanZM22} and the connection between performativity and causality has been studied by \citet{NEURIPS2022_PPCausal, kulynych2022causal}.

\paragraph{Multi-Agent Performative Prediction}

A recent line of work studies performative prediction in a multi-agent setting with slightly different frameworks, e.g., consensus seeking agents where an agent $i$ has a local distribution that is not effected by the other agents' decisions \citep{li2022multi}, smooth games where the local distributions are affected by the joint decision \citep{narang2023multiplayer}, global distributions \citep{piliouras2023multi}.
Among these, our setting is conceptually closest to \citet{narang2023multiplayer}, which provide a game theoretic notion of multi-agent performative prediction. In this work, they provide methods to converge to a PSE in (stateless) strongly-monotone games.

\paragraph{Variations of MPGs.} 
MPGs have also been considered in different variations, e.g., $\alpha$-approximate MPGs \citep{Guo2023MarkovG}, Networked MPGs \citep{NetworkedMPGs}, fully decentralized settings (agents may not require to know if other agents exist) \citep{MPG_DecentralizedLearning}. 
Another recent line of work considers constrained MPGs to study MARL under safety constraints \citep{DBLP:conf/atal/AlaturRH024, DBLP:conf/aistats/JordanBH24}.

\paragraph{Non-stationary Multi-agent Reinforcement Learning.}
Our work is also related non-stationary MARL. There is work on full-information settings (e.g. gradients are known), see e.g., \cite{AdversarialCorruptions_zerosum, OntheConvergenceofNo-RegretLearningDynamicsinTime-VaryingGames, MultiagentOnlineLearninginTime-VaryingGames},
bandit feedback (gradient estimations are required) \citep{DBLP:conf/iclr/JiangCXFD24}.
\newpage
\section{Errata}

In the original version of this paper, there are minor typos in the first and second complexity result in the table \ref{tab:main-results}, stated correctly in this version. Additionally, in Theorem \ref{Theorem: Convergence Result - Infinite Sample Case}, $\mathcal{W}_{r,p}$ does not dependent on $T$. 

}
}
{}

\end{document}